%% file: graphem_Q_R1.tex
\documentclass[twoside,11pt]{article}

\usepackage[abbrvbib, preprint]{jmlr2e} 

\newtheorem{mytheorem}{Theorem}
\newtheorem{mylemma}{Lemma}  

\usepackage{amsmath}
\usepackage{amsfonts}
\usepackage{amssymb}
\usepackage{bm}
\usepackage[usenames]{color}
\usepackage[dvipsnames]{xcolor}

\usepackage{etoolbox}
\newcommand{\ubold}{\fontseries{b}\selectfont}
\robustify\ubold

\usepackage{multirow}
\usepackage{subcaption}
\usepackage{graphicx}

\usepackage{tikzit}
\input{arrowleft.tikzstyles}

\usepackage{siunitx}
\newcommand\acro{DGLASSO}

\newcommand{\argmin}[2]{\ensuremath{\underset{\substack{{#1}}}{\operatorname{argmin}}\;\;#2 }}
\def\x{{\mathbf x}}
\def\y{{\mathbf y}}
\def\v{{\mathbf v}}
\def\q{{\mathbf q}}
\def\r{{\mathbf r}}
\def\z{{\mathbf z}}

\def\A{{\mathbf A}}

\def\H{{\mathbf H}}
\def\R{{\mathbf R}}
\def\Q{{\mathbf Q}}
 
\def\V{{\mathbf V}}
\def\U{{\mathbf U}}
\def\G{{\mathbf G}}

\def\M{{\mathbf M}}

\def\Z{{\mathbf Z}}
\def\W{{\mathbf W}}

\def\z{{\mathbf z}}
\def\P{{\mathbf P}}
\def\S{{\mathbf S}}

\def\ctAP{{\mathrm{ct}_{/\mathbf{A},\mathbf{P}}}}
\def\Id{{\textbf{Id}}}
\def\L{{\cal L}}
\newcommand{\Sigmab}{{\boldsymbol \Sigma}}
\newcommand{\Deltab}{{\boldsymbol \Delta}}
\newcommand{\Psib}{{\boldsymbol \Psi}}
\newcommand{\Phib}{{\boldsymbol \Phi}}
\newcommand{\Pib}{{\boldsymbol \Pi}}

\newtheorem{Algoritmo}{\em Algorithm}
\newcommand{\mub}{\boldsymbol\mu}

\newcommand{\nub}{{\boldsymbol\nu}}

\newcommand{\meanPrior}{\mub}
\newcommand{\covPrior}{\Sigmab}
\newcommand{\meanFilt}{\mub} 
\newcommand{\covFilt}{\Sigmab}
\newcommand{\meanPred}{\mub}
\newcommand{\covPred}{\Sigmab}

\newcommand{\meanSmooth}{\mub^{\text{s}}}
\newcommand{\covSmooth}{\Sigmab^{\text{s}}}

\newcommand{\loss}{\mathcal{L}}

\newcommand{\maj}{\mathcal{Q}}

\newcommand{\K}{{\bf K}}

\def\gauss{{\mathcal N}}

\def\Ab{{\mathbf A}}

 \newcommand{\Real}{\mathbb{R}}

\newcommand{\cblue}{\textcolor{black}}

\jmlrheading{1}{2023}{1-48}{4/00}{10/00}{chouzenouxjmlr}{Emilie Chouzenoux and V\'ictor Elvira}

\newcommand{\titre}{Sparse Graphical {Linear}  Dynamical Systems}

\ShortHeadings{\titre}{Chouzenoux and Elvira}
\firstpageno{1}

\begin{document}

\title{\titre}

\author{\name Emilie Chouzenoux \email emilie.chouzenoux@inria.fr \\
       \addr Center for Visual Computing\\
       Inria, University Paris Saclay\\
       91190 Gif-sur-Yvette, France
       \AND
       \name V\'ictor Elvira \email victor.elvira@ed.ac.uk \\
       \addr School of Mathematics\\
       University of Edinburgh\\
       EH9 3FD Edinburgh, UK}

\editor{Kevin Murphy and Bernhard Sch{\"o}lkopf}

\maketitle

\begin{abstract}
{Time-series datasets are central in \cblue{machine learning with applications in} numerous fields of science and engineering, such as biomedicine, Earth observation, and network analysis. Extensive research exists on state-space models (SSMs), which are powerful mathematical tools that allow for probabilistic and interpretable learning on time series. Learning the model parameters in SSMs is arguably one of the most complicated tasks, and the inclusion of prior knowledge is known to both ease the interpretation but also to complicate the inferential tasks. Very recent works have attempted to incorporate a graphical perspective on some of those model parameters, but they present notable limitations that this work addresses. More generally, existing graphical modeling tools are designed to incorporate either static information, focusing on statistical dependencies among independent random variables (e.g., graphical Lasso approach), or dynamic information, emphasizing causal relationships among time series samples (e.g., graphical Granger approaches). However, there are no joint approaches combining static and dynamic graphical modeling within the context of SSMs. This work proposes a novel approach to fill this gap by introducing a joint graphical modeling framework that bridges the graphical Lasso model and a causal-based graphical approach for the linear-Gaussian SSM. We present \emph{DGLASSO (Dynamic Graphical Lasso)}, a new inference method within this framework that implements an efficient block alternating majorization-minimization algorithm. The algorithm's convergence is established by departing from modern tools from nonlinear analysis. Experimental validation on \cblue{various synthetic data} showcases the effectiveness of the proposed model and inference algorithm. This work will significantly contribute to the understanding and utilization of time-series data in diverse scientific and engineering applications where incorporating a graphical approach is essential to perform the inference.}
\end{abstract}


\section{Introduction}


\cblue{Time series appear in applications of most fields of science and engineering, ranging from biomedicine, Earth observation, or network analysis, to name a few. 
The literature of time-series analysis is vast across different fields. For instance, in statistics and signal processing, it is mostly dominated by auto-regressive moving-averaging (ARMA) models and their extensions \citep{hamilton2020time}. The machine learning literature has recently seen a plethora of neural-network-based models, including long short-term memory (LSTM) models \citep{HochreiterLSTM}, recurrent neural networks (RNNs) \citep{HUSKEN2003223}, gated recurrent unit networks (GRNs) \citep{KyunghyunGRU}, and attention mechanisms \citep{ShihGRU}.}  

\paragraph{\cblue{Graphical inference.}} \cblue{Graphical modeling is an important family of approaches for time series analysis in statistical machine learning. The literature in multivariate times series is abundant in models for representing correlation without exploiting the temporal structure \citep{Lauritzen}. We refer to these approaches as static graphical models.} Let us mention the famous Graphical Lasso algorithm \citep{Friedman07}, and its multiple variants \citep{Chandrasekaran2012,Benfenati18,Belilovsky,Bach04,Ying2020b,JMLR:v13:mazumder12a,JMLRFattahi,JMLR:Pircalabelu}, for inference of sparse Gaussian graphical models. Modeling (non-independant) dynamical behaviors through (directed) graph representations, and in particular, causality, has also been explored~\citep{Eichler2012,Ioannidis2019,Giannakis2018,JMLR:v9:zhang08a,JMLR:Witte}. \cblue{Graph inference can also be tackled through a supervised machine learning machinery, leading to graph representation learning \citep{JMLR:Chami}, and graphical neural networks \citep{SeoGCRN,JMLR:Cini}, with successful applications in the context of time series \citep{JinGNN} and beyond \citep{Krzywda2022gnn}. Note also that ARMA processes have also been studied from the graphical perspective within the machine learning literature (e.g., in \cite{songsiri2010topology}).} 

\noindent\textbf{\cblue{State-space models (SSMs).}}  State-space models (SSMs) have became very popular in the last decades \citep{hamilton1994state,kim1999state,Sarkka,newman2023state} for time series modeling. SSMs characterize complex systems through discrete-time models composed of a hidden (or latent) state that evolves in a Markovian manner. SSMs are composed of a state model,  which can mimic realistically the complicated dynamics of the system, and the observation model, which links the temporal observations to the hidden state at the same time step.
In SSMs, the Bayesian filtering task consists on computing the posterior probability density function (pdf) of the  hidden state at a given time step, given all observations from the time series available up to that time. However, in most SSMs the posterior pdf is intractable and must be approximated, generally through particle filters {\citep{djuric2003particle,doucet2009tutorial,naesseth2019elements}}.
One relevant exception is the linear-Gaussian state-space model (LG-SSM), which allows for exact inference, when the model parameters are known, through the Kalman filter and the Rauch-Tung-Striebel (RTS) smoother \citep[Chapter 8]{Sarkka}. The LG-SSM is arguably the most popular model and is still subject of intense research also in the machine learning community \cblue{(see for instance a Kalman filter for high-dimensional spaces with a low-rank covariance approximation in \cite{schmidt2023rank}).}

\noindent\textbf{\cblue{Parameter learning.}} One of the main challenges in SSMs lies in learning the model parameters, which considers both probabilistic and point-wise approaches. The probabilistic (or fully Bayesian) methods can be applied for a wide class of SSMs (i.e., beyond LG-SSMs) and include for instance  {particle Markov chain Monte Carlo}  \citep{andrieu2010particle}, particle Gibbs \citep{lindsten2014particle}, sequential Monte Carlo squared  {\citep{chopin2013smc2}}, and nested particle filters {\citep{crisan2018nested,perez2021nested,perez2023adaptive}} (see \citep{kantas2015particle} for a review).
%
\cblue{Point-wise approaches are generally based on maximum likelihood estimation of the model parameters in LG-SSMs~\citep[Chap.~12]{Sarkka}. A first family of methods implements optimizers such as quasi-Newton \citep{Olsson} or Newton-Raphson \citep{Gupta1974}, requiring recursive likelihood derivatives expressions~\citep{Segal88,Segal89,Nagakura2021,Gupta1974} (see the discussion in \citep[Sec.~10.2.4]{Cappe2005}). The second family of point-wise methods is based on expectation-minimization (EM) algorithm \citep{shumway1982approach}\citep[Sec.~10.4]{Cappe2005}\citep[Sec.~12.2.3]{Sarkka}. EM exhibits a simplicity of implementation in the LG-SSM context, which might explain its wide use on practical applied fields, such  as finance, electrical engineering, and radar~\citep{Sharma,SharmaRDL,Frenkel}. The benefits and drawbacks of these algorithms are discussed in \citep[Sec.~1]{shumway1982approach}.} 



\noindent\textbf{\cblue{Dynamical graphical modeling, and other machine learning approaches.}} There is a clear link between state-space modeling and dynamical graphical modeling, as emphasized in \citep{Barber10}. For instance, in \cite{Sagi2023}, the extended Kalman filter (EKF) is enhanced by taking a graphical perspective. \cite{Ioannidis2018} take an approach of jointly estimating latent processes and the topology of a related graph. \cite{Alippi2023} also adopt a graphical perspective in linear-Gaussian models (or linearized versions of them through the EKF) in order to learn model parameters through deep neural networks. \cblue{This approach bears links with differentiable particle filters (DPFs) \citep{corenflos2021differentiable,chen2021differentiable,chen2023overview}, where the proposal of the particle filters but also the dynamics of the latent state are generally learned through gradient-based methods, including modern neural-network architectures.} 
 \cblue{It is also worth mentioning other tight links between SSMs and neural network architectures, such as LSTM and RNNs, for time series processing, as emphasized for instance in \cite{pmlr-v80-doerr18a,NEURIPS2018_Rangapuram,JMLR:Lim,coskun2017long}.}


\cblue{The approaches for graphical model inference and model parameter estimation are often very similar algorithmically speaking, with the difference being mostly related to the interpretation of the parameters and their representation.} The graphical modeling brings new insights such as the choice of specific priors (typically, sparsity) and algorithms (e.g., binary edge selection tools). Sparsity usually plays a key role~\citep{Brendan2010}, since a graph with few edges can be enforced by imposing a sparsity prior at inference stage. A typical choice is the $\ell_1$ (i.e., Lasso) penalty~\citep{Friedman07,Meinshausen2006,Chouzenoux2020}, which has the advantage of being convex.\footnote{See also \cite{Gao2015} for state-space model inference under $\ell_1$ prior, although not related to graphical modeling} Other types of priors have also been explored in \cite{Ying2020b,Benfenati18,Chandrasekaran2012,Kumar2020,Hippertgraphical}, with the aim of imposing specific graph structures (e.g., low rank, bounded spectrum, block sparsity). Proximal algorithms \citep{Combettes2011}, including augmented Lagrangian methods~\citep{komodakis2015playing}, are typically the answer of choice for accounting for the various class of priors of interest in graphical inference. Discrete optimization techniques for edge selection have also been explored (see \cite{BenfenatiReview} and references therein). Dynamic graph inference has also been performed using proximal tools \citep{Ioannidis2019}. In the particular case of state-space models, as we explained earlier, several methods, for instance based on Newton or EM, are available for parameter inference. But most available methods did not introduce any graphical aware prior.
In our recent works \citep{Chouzenoux2023,elvira2022graphem,chouzenoux2024graphical}, we introduced more sophisticated EM-based schemes, with proximal updates, to cope with generic convex regularizations. The goal was to estimate the transition matrix within an LG-SSM, adopting a Granger-based graphical interpretation of such matrix. \cblue{In other words, it became possible to identify which components of the latent space contain information for a one-step-ahead prediction of a given component of the same latent vector.} A fully Bayesian approach to estimate the transition matrix was taken in \cite{cox2022parameter,cox2023sparse} where the space of sparse matrices was explored via reversible jump Markov chain Monte Carlo \citep{green2009reversible}.  
 
\paragraph{\cblue{The need for better graphical modeling in SSMs.}} Despite this prolific literature, there exists a gap that we propose to fill in this work. Graphical modeling tools are either dedicated to represent static information related to statistical dependence between independent realizations of random variables (e.g., graphical Lasso approach), \cblue{or to represent dynamic information across different time series}. For the latter case, time series are either processed directly \citep{Ioannidis2019,songsiri2010topology,Mei2017}, or through a state-space modeling including hidden state \citep{elvira2022graphem}. \cblue{Furthermore, deep learning approaches for SSM parameter estimation \citep{coskun2017long,Revach2022,Buchnik2023}  implements supervised training mechanisms which require consequent number of data batches.}

 However, we are not aware of any joint approach which includes both static and dynamic graphical modeling (hence, two graphs with two distinct purposes), \cblue{applicable in the context of state-space models, and without the need for supervision}. Our contribution lies within this goal, as we describe in the next section.
 
\subsection{Summary of our main contributions}

\cblue{
Our contributions are as follows:
\begin{itemize}
\item We introduce a joint graphical modeling for representing static and dynamical behaviors in the hidden state of a linear Gaussian state-space model, bridging the gap between the graphical Lasso model from~\cite{Friedman07} and the dynamic model from~\cite{elvira2022graphem}.
\item We present a novel Bayesian inference method, called \emph{DGLASSO (Dynamic Graphical Lasso)}, \cblue{that performs the graph representation learning task, given a single observed time series. DGLASSO estimates both graphs, under a sparsity prior.}
\item We propose an original optimization algorithm for the estimation task, implementing a block alternating proximal algorithm with efficient majorization-minimization inner steps. We establish the convergence of our algorithm using recent tools from nonlinear analysis.
\item We then perform an extensive experimental validation of the proposed model and inference algorithm by means of experiments on \cblue{various synthetic datasets}. For reproducibility purpose, the code for DGLASSO algorithm, is made publicly available.\footnote{https://pages.saclay.inria.fr/emilie.chouzenoux/Logiciel.html}
\end{itemize}
}
%
%
 %
\section{\cblue{Preliminaries}}

\subsection{Notation}

Bold symbols are used for matrix and vectors. We denote by $\|\x\|_2={\sqrt{\x^\top \x}}$ the Euclidean norm of $\x \in \Real^N$, where $\top$ states from the transpose operation and $\Real^N$ is the $N$-dimensional Euclidean space. We also introduce $\| \mathbf{X} \|_F$, $\| \mathbf{X} \|_2$ and $\text{tr}(\mathbf{X})$, the Frobenius norm, the spectral norm (i.e., largest singular value), and the trace, respectively, of elements $\mathbf{X} = (X(n,\ell))_{1 \leq n \leq N, 1 \leq \ell \leq M} \in \Real^{N \times M}$. $\Id_{N}$ is the identity matrix of $\Real^N$. $\mathcal{S}_{N}$ denotes symmetric matrices of $\Real^{N \times N}$. Both $\Real^{N \times N}$ and $\mathcal{S}_{N}$ are Hilbert spaces, endowed with the Frobenius norm $\| \cdot \|_F$ and the trace scalar product $ \left\langle \mathbf{X},\mathbf{Y} \right\rangle = \text{tr}(\mathbf{X} \mathbf{Y})$. $\mathcal{S}_{N}^+$ (resp. $\mathcal{S}_{N}^{++}$) is the set of $N \times N$ symmetric positive semidefinite (resp. definite) matrices of $\Real^N$. {Given a sequence of elements $\{\x_k\}_{k=1}^K$ of length $K \geq 1$ and size $N$, we denote each element as $\x_k = (x_k(n))_{1 \leq n \leq N}$   and we use the notation $\x_{k_1:k_2}$ to refer to the subsequence $\{\x_k\}_{k=k_1}^{k_2}$, for $1 \leq k_1 < k_2 \leq K$.} 
 For convex analysis concepts, we rely on the notation in the reference textbook by \cite{bauschke2017convex}. We denote $\Gamma_0(\mathcal{H})$ the set of proper, lower semi continuous convex functions from a Hilbert space $\mathcal{H}$ to $(-\infty,+\infty]$ \citep[Chap. 9]{bauschke2017convex}, and $\partial f$ the subdifferential of $f \in \Gamma_0(\mathcal{H})$ \citep[Chap. 16]{bauschke2017convex}. With a slight abuse of notation, we make use of the extended form of the minus logarithm determinant function, defined as
\begin{equation}
(\forall \P \in \mathcal{S}_N) \quad
- {\log \det}(\P) = 
\begin{cases}
- \log |\P|, & \text{if }\P \in \mathcal{S}_N^{++},\\
+ \infty, & \text{otherwise},
\end{cases}
\label{eq:logdet}
\end{equation}
with $|\P|$ the product of the eigenvalues of $\P$. According to \cite{bauschke2017convex}, the function in Eq. \eqref{eq:logdet} belongs to $\Gamma_0(\mathcal{S}_N)$.

\subsubsection{Graphs}     

We introduce here our notation for describing graphs. Most of our notation is inherited from \cite{buhlmann2011statistics}. 

Let us define a graph $\mathcal{G}$ made of set of $N$ vertices $\mathcal{V} = \{v^{(n)}\,\text{s.t.}\,n \in \{1,\ldots,N\}\}$ and of a set of edges $\mathcal{E} = \{e^{(n,\ell)}\,\text{s.t.}\,(n,\ell) \in \mathbb{E}\}$. The latter gathers ordered pairs of distinct vertices, and as such, $\mathbb{E} \subset \{1,\ldots,N\}^2$. Undirected graphs are made of undirected edges, that is such that $(n,\cblue{m}) \in \mathbb{E}$ and $(m,n) \in \mathbb{E}$, for every $(m,n) \in \{1,\ldots,N\}^2$. In contrast, directed graphs consist of directed edges, where we say that some $e^{(n,\ell)} \in \mathcal{E}$ is directed (from $n$ to $\ell$) if $(\ell,n) \notin \cblue{\mathbb{E}}$. We can also distinguish reflexive graphs if self-loops are allowed (i.e., one can have $(n,n) \in \mathcal{E}$), and nonreflexive graphs otherwise. Given these definitions, one can simply bound the cardinality of $\mathbb{E}$ for each category. For instance, for \cblue{an undirected nonreflexive graph, $\text{Card}(\mathbb{E}) \leq N(N-1)/2$, while a directed nonreflexive graph has $\text{Card}(\mathbb{E}) \leq N(N-1)$, and a directed reflexive graph has $\text{Card}(\mathbb{E}) \leq N^2$}. Such graph definitions are binary, as it only described presence/absence of edges between vertex pairs. In this work, we require the notion of a weighted graph, where the edges $(e^{(n,\ell)})_{(n,\ell) \in \mathbb{E}}$ are associated to real valued weights $(\omega^{(n,\ell)})_{(n,\ell) \in \mathbb{E}}$. The edge positions and weight values of a graph are summarized in a matrix $\M \in \mathbb{R}^{N \times N}$, where, for every $(n,\ell) \in \mathbb{E}$, $M(n,\ell) = \omega^{(n,\ell)}$ and, for every $(n,\ell) \in \{1,\ldots,N\}^2 \notin \mathbb{E}$, $M(n,\ell) = 0$. An undirected (resp. directed) graph is thus associated to a symmetric (resp. non symmetric) matrix $\M$. A reflexive (resp. non reflexive) graph has non zero (resp. zero) entries in the diagonal of $\M$. The number of edges is simply obtained as the number of non-zero entries in $\M$. We finally define the so-called \emph{binary support} of the (possibly sparse) matrix $\M$. For every $\M  \in \mathbb{R}^{N \times N}$, $\text{supp}(\M) = \S \in \{0,1\}^{N \times N}$, with,  for every $(n,\ell) \in \{1,\ldots,N\}^2$, $S(n,\ell) = 0$ if and only if $M(n,\ell) = 0$ (i.e., $(n,\ell) \in \mathbb{E}$).

{\color{blue}
\subsection{Dynamical modeling}

\cblue{Let us consider an observed multivariate time series $\{\y_{k} \}_{k=1}^K\in \mathbb{R}^{N_y}$, with $K \geq 1$, and dimension $N_y \geq 1$. A model describes the evolution of time-series across time. This modeling allows for various machine learning tasks, such as forecasting (i.e., prediction of the future time series, for $k > K$), parameter estimation (e.g., \cite{andrieu2010particle}), classification (e.g.., \cite{HUSKEN2003223}), change point detection (e.g., \cite{Aminikhanghahi2017}), among many other tasks. A vast amount of models are available in the literature of time series analysis. Here we focus on the powerful state-space modeling (SSM) representation, a central block in some recent high-performance generative models such as Mamba \citep{gu2023mamba}}.
                                                                                                                                   


}

\subsection{Considered model}

\cblue{We consider the linear-Gaussian state-space model (LG-SSM), widely studied in the statistics, control theory, and signal processing literature, and recently popular again in the machine learning literature (see for instance its integration with RNNs and CNNs in \cite{gu2021combining}).} The LG-SSM is described, for $k=1,\ldots,K$, as
 \begin{align}
       \x_{k} &= \A \x_{k-1} + \q_{k},       \label{eq_model_state}
\\
     \y_k &= \H_k \x_k + \r_{k}, \label{eq_model_obs}
    \end{align}
where, 
\begin{itemize}
    \item $\{\x_{k} \}_{k=1}^K\in \mathbb{R}^{N_x}$ and  $\{\y_{k} \}_{k=1}^K\in \mathbb{R}^{N_y}$, are the hidden state and the observations at each time $k$ respectively,
    \item $\A \in \mathbb{R}^{N_x \times N_x}$ is the transition matrix that we aim at estimating,
    \item $\{\H_k\}_{k=1}^K \in \mathbb{R}^{N_y \times N_x}$ maps the observation model matrices, possibly varying with $k$, that are assumed to be known,
    \item $\{ \q_k \}_{k=1}^K \sim \mathcal{N}(0,{\Q})$ is the i.i.d. state noise process, assumed to follow a zero-mean Gaussian model with covariance {matrix $\Q \in \mathcal{S}_{N_x}^{++}$} that we also aim at estimating,
    \item $\{ \r_k \}_{k=1}^K \sim \mathcal{N}(0,\R_k)$ is the i.i.d. observation noise process, again zero-mean Gaussian with known covariance matrices $\R_k \in \mathcal{S}_{N_y}^{++}$.
\end{itemize}
Throughout the paper, we denote $\P = \Q^{-1}$ the precision matrix of the state noise. We furthermore assume an initial state distributed such that $\x_0 \sim \mathcal{N}(\x_0 ; \meanPrior_0, \covPrior_0)$ with known $\meanPrior_0 \in \mathbb{R}^{N_x}$ and $\covPrior_0\in \mathcal{S}^{++}_{N_x}$. The state noises and the observation noises are mutually independent and also independent of the initial state~$\x_0$. 

In the next subsection, we provide some reminders about filtering/smoothing procedures for time series described by an LG-SSM. 


\subsection{Filtering and smoothing algorithms in linear dynamical systems}



Both filtering and smoothing algorithms consist on the computation of a posterior probability density function (pdf) of the hidden state $\{\x_k\}_{k=1}^K$. For every $k \in \{1,\ldots,K\}$, the filtering distribution is $p(\x_k|\y_{1:k})$, where we denote as $\y_{1:k} = \{ \y_j \}_{j=1}^k$ the set of observations available up to the time step $k$, i.e., no future observations can be used to estimate $\x_k$. The filtering problem is suitable for online processing of the observations. The smoothing distribution is $p(\x_k|\y_{1:K})$, where $K$ is the final time-step for which there is an available observation, i.e., for $k \in \{1,\ldots,K-1\}$ (note that it is also possible to condition on a subset of future observations, e.g., $p(\x_k|\y_{1:k+\tau})$ with $\tau\in\mathcal{N}$). 

The filtering and smoothing distributions are in general intractable for most SSMs of interest. The LG-SSM is one of the exceptions that admit closed-form solutions.  

Estimating the filtering and smoothing distributions is in general a challenging problem, since obtaining these distributions of interest is possible only in few models of interest. For instance, for the LG-SSM described in \eqref{eq_model_state}-\eqref{eq_model_obs}, it is possible to obtain the filtering and smoothing distributions, for $k=1,\ldots,K$, in the case where the model parameters $\A$, {$\Q$}, $\{\H_k \}_{k=1}^K$, and $\{\R_k \}_{k=1}^K$ are known. Interestingly, these distributions can be obtained in an efficient sequential manner. In particular, the Kalman filter (KF) \citep{Kalman60} allows to obtain recursively (in a forward manner) the sequence of filtering distributions. Its smoothing counterpart, the Rauch-Tung-Striebel (RTS) smoother \citep{Briers05}, runs backwards to obtain the sequence of smoothing distributions. 
We note that both algorithms require the model parameters to be known, which in the case of the LG-SSM presented in the previous section are $\A$, {$\Q = \P^{-1}$}, $\{\H_k \}_{k=1}^K$, and $\{\R_k \}_{k=1}^K$. Algorithm \ref{alg_kf} describes KF, which at each time step $k \in \{1,\ldots,K\}$ performs the (a) prediction/propagation step, where the mean $\meanPred_{k|k-1}$ 
 and covariance $\covPred_{k|k-1}$ of the (state) predictive distribution are obtained; and the (b) update step, where the mean $\meanFilt_{k}$ 
 and covariance $\covFilt_k$ of the filtering distribution are obtained.
 Algorithm \ref{alg_rts} describes the RTS smoother. In this case, the iteration starts at {$k=K$} and runs backwards. It can be interpreted as a refinement from the mean and covariance matrices of the filtering distribution, given by Kalman, updating them with information present in future observations. However, note that the observations are not re-used in the RTS algorithm, i.e., all the required information in the observations is absorbed by the filtering distributions, which are used to produce the smoothing distributions. 


\begin{table}[!t]
\vspace{4mm}
    \centering
    \begin{tabular}{|p{0.95\columnwidth}|}
    \hline
\begin{Algoritmo}
\label{alg_kf}
Kalman Filter
\begin{enumerate}
  \item[] \textbf{Input.} Prior parameters $\meanPrior_0$\;and  $\covPrior_0$; model parameters $\A$, ${\P}$, $\{\H_k \}_{k=1}^K$, and $\{\R_k \}_{k=1}^K$; set of observations $\{\y_k \}_{k=1}^K$.

 
\item[] \textbf{Recursive step.} For $k=1,\ldots,K$  
\begin{enumerate}
  \item {{\sf Prediction/propagation step.}} 
\begin{eqnarray}
\meanPred_{k|k-1} &=& \A\meanFilt_{k-1} \\
\covPred_{k|k-1} &=& \A\covFilt_{k-1}\A^\top + {\P^{-1}}
\label{eq_propagation_kf}
\end{eqnarray}
\item {{\sf Update step.}} 
\begin{eqnarray}
\nub_{k} &=& \H_k \meanPred_{k|k-1}\\
\v_{k} &=& \y_k - \nub_k \\
\S_{k} &=&  \H_k \covPred_{k|k-1} \H_k^\top + {\R_k} \\
\K_{k} &=&  \covPred_{k|k-1} \H_k^\top \S_k^{-1} \\
\meanFilt_{k} &=& \meanPred_{k|k-1} + \K_k \v_k\\
\covFilt_k &=& \covPred_{k|k-1} - \K_{k} \S_k \K_{k}^\top 
\label{eq_update_kf}
\end{eqnarray}
\end{enumerate}
\item[] \textbf{{Output.}} $\{\meanFilt_k, \covFilt_k \}_{k=1}^K$. Then, for each $k=1,...,K$: 
\begin{itemize}
  \item state filtering pdf: $p(\x_k|\y_{1:k}) = \gauss(\x_k ; \meanFilt_k, \covFilt_k)$
  \item observation predictive pdf: $p(\y_k|\y_{1:k-1}) = \gauss(\y_k;\nub_k,\S_k)$
  
\end{itemize}
\end{enumerate}
\end{Algoritmo}\\
        \hline
\end{tabular}
\end{table}

\begin{table}[!t]
\vspace{4mm}
    \centering
    \begin{tabular}{|p{0.95\columnwidth}|}
    \hline
\begin{Algoritmo}
\label{alg_rts}
RTS Smoother
\begin{enumerate}
  \item[] \textbf{Input.} Filtering parameters $\{\meanFilt_k, \covFilt_k \}_{k={0}}^K$ from the Kalman filter; model parameters $\A$ and {$\P$}.
\item[] \textbf{Initialization.} Set $\meanSmooth_{K} = \meanFilt_{K}$ and $\covSmooth_{K} = \covFilt_{K}$. 
 
\item[] \textbf{Recursive step.} For {$k=K,K-1,...,{0}$}  

{\begin{align}
\meanFilt_{k+1}^{-} &= \A\meanFilt_{k} \\
\covFilt_{k+1}^{-} &= \A\covFilt_{k}\A^\top + {\P^{-1}}\\
\G_k &= \covFilt_{k}\A^\top \Big(\covFilt_{k+1}^{-} \Big)^{-1}\\
\meanSmooth_{k} &= \meanPred_{k|k-1} + \G_k \left(\meanSmooth_{k+1}  - \meanFilt_{k+1}^{-} \right)\\
\covSmooth_k &= \covPred_{k|k-1} - \G_k \left(\covSmooth_{k+1}  - \covFilt_{k+1}^{-} \right)\G_k^\top
\end{align}} 

\item[] \textbf{{Output.}} $\{\meanSmooth_k, \covSmooth_k \}_{k=1}^K$. Then, for each $k=1,...,K$: 
\begin{itemize}
  \item state smoothing pdf: {$p(\x_k|\y_{1:K}) = \gauss(\x_k ; \meanSmooth_k, \covSmooth_k)$} 
  
\end{itemize}
\end{enumerate}
\end{Algoritmo}\\
        \hline
\end{tabular}
\end{table}

\subsection{Problem statement}

Algorithms \ref{alg_kf} and \ref{alg_rts} are simple and efficient. However, they require the knowledge of the model parameters. In this paper, we assume $\{\H_k \}_{k=1}^K$, and $\{\R_k \}_{k=1}^K$ to be known, and we address the problem of obtaining the filtering and smoothing distributions when matrices $\A$ and $\P$ are unknown, and must be estimated jointly with the filtering/smoothing step. To do so, we introduce a double graphical modeling of the state equations, where matrices $\A$ and $\P$ now represent the weights of graphs with a specific, and complementary, statistical interpretation. We then propose an efficient and convergent inference approach to estimate both graphs under sparse priors given an observed sequence $\{\y_k \}_{k=1}^K$, while also obtaining its filtering and smoothing distributions at every time steps. \cblue{The graph representation learning and the time series inference is performed per each time-series observation (i.e., it does not require a training phase in identically distributed time series). This contrasts with the deep learning based techniques, deployed for instance in \cite{Revach2022,SeoGCRN}, that include a supervision stage averaging over a dataset of multiple time series.}

\section{Proposed {model and inference method}}   
\label{sec_proposed}

{We now introduce our graphical modeling perspective of the LG-SSM. This novel view gives a graph-based interpretation of the transition matrix $\A$, and the precision matrix $\P$. Then, we present the optimization methodology in order to \cblue{estimate such matrices, and, as such, learning the joint graph representation of a given observed time series}.}

\subsection{Sparse dynamical graphical model}

{Our novel approach interprets the transition matrix $\A$ and the precision matrix $\P$ as (weighted) directed and undirected graphs, respectively. We now particularize this perspective for each matrix, deepening into the interpretability of such novel view, the complementarity of both graphs, and its benefits during the inference process.}



{


\subsubsection{State transition matrix}

Matrix $\A$ governs the hidden process in \eqref{eq_model_state} and can be seen as the matrix parameter of an order-one vector auto-regressive (VAR) unobserved process. For every $n \in \{1,\ldots,N_x\}$ and $\ell \in \{1,\ldots,N_x\}$, the entry $A(n,\ell)$ contains the information of how the $n$-th time series $\{x_k(n)\}_{k=1}^K$ is affected by the $\ell$-th time series $\{x_k(\ell)\}_{k=1}^K$ in consecutive time steps. More precisely, we can express the update of the $n$-th dimension of the latent state in the generative model as
\begin{equation}
x_k(n) = \sum_{\ell=1}^{N_x} A(n,\ell)x_{k-1}(\ell)+q_k(n).
\end{equation}

Thus, if $A(n,\ell)=0$, it implies that the $\ell$-th time series of the latent state does not provide any information to predict the $n$-th time series one time step ahead conditioned to observing the information in all time series for which $A({n,m})\neq 0$, $m \in \{1,\ldots,N_x\}\setminus \ell$. 
%

\cblue{
We express this one-step-ahead conditional independence by denoting  
\begin{equation}
x_k(n) \perp\!\!\!\perp x_{k-1}(\ell) | \{ x_{k-1}(i) \}_{i\in\mathcal{I}_n},
\end{equation}
with  {$ \mathcal{I}_{n} = \{ m |A(n,m)\neq 0, m \in \{1,\ldots,N_x\} \}$}, i.e., due to the Markovian structure, conditionally on the $k$-th components of the time series indexed by $\mathcal{I}_{n}$, all other past components of all time series are independent to $x_k(n)$. Even more, it is possible to establish that $p(x_k(n)| \x_{1:k-1}) = p(x_k(n)|\{x_{k-1}(i) \}_{i\in\mathcal{I}_n})$, again due to the Markovian structure. 
}

Our interpretation is clearly connected to Granger causality \citep{granger1969investigating}, and more in particular to conditional Granger causality \citep{Luengo19}.\footnote{\cblue{We recall here a vanilla version of the Granger-causality hypothesis test. Let us consider two observed univariate time-series $\z_i = [z_{1,i},z_{2,i},...,z_{K,i}]$ and $\z_j = [z_{1,j},z_{2,j},...,z_{K,j}]$, and the two following statistical models: (A) $z_{k,i} = a_1z_{k-1,i}+ \varepsilon_k$; and (B) $z_{k,i} = a_1 z_{k-1,i}+ b_1 z_{k-1,j}+ \gamma_k$. We say that $\z_j$  Granger-causes $\z_i$ if, when fitting the two auto-regressive (AR) models  (with order $p=1$ in our example), model B is statistically significantly better than model A, i.e., the variance of $\gamma_k$ is significantly smaller than the variance of $\varepsilon_k$.}} 
{In particular, if $A(\cblue{n,\ell})\neq 0$ for some $(n,\ell) \in \{1,\ldots,N_x\}^2$, it means that $x_{k-1}(\cblue{\ell})$ causes (in conditional Granger sense) $x_{k}(\cblue{n})$, for every $k \in \{1,\ldots,K\}$. The conditional Granger causality directional relationships within the entries of the multivariate latent state time series $\x$ can hence be represented as a graphical model made of a directed graph with $N_x$ vertices, whose weights are those of the transpose matrix $\A^\top$. Moreover, self-loops occur at each vertex associated to a non-zero diagonal entry in $\A$}. {The perspective of matrix $\A$ interpreted as a (weighted) directed graph bears some links with the work of graphical Granger causality by \cite{shojaie2010discovering}, although we here model the interactions in the latent space instead of directly on the observations. In our case the graphical modeling is simpler, since the work by \cite{shojaie2010discovering} considers all possible interactions in the observation space across time (i.e., the interpreted graph size is $K\cdot N_y$). The price to pay of our modeling is the difficulty in inferring $\A$, since it governs the latent process, hence it is never observed. We propose an advanced methodology to address the inferential task in Section \ref{sec_proposed}.} 




\noindent\textbf{Illustrative example.} In Figure \ref{fig:adj_A}, we display an illustrative example of the graphical model associated to the following state equations for $N_x = 5$, for every $k \in 1,\ldots,K$, with $\q_k \in \mathbb{R}^{N_x}$ the latent state noise
\begin{equation}
\begin{cases}
x_k(1) = \,0.9 \,x_{k-1}(1) + 0.7 x_{k-1}(2) +  q_k(1),\\
x_k(2) = - 0.3 \,x_{k-1}(3) +  q_k(2),\\
x_k(3) = \,0.8 \,x_{k-1}(5) +  q_k(3),\\
x_k(4) = - 0.1 \,x_{k-1}(2) +  q_k(4),\\
x_k(5) = \,0.5 \,x_{k-1}(3) +  q_k(5).
\end{cases}
\label{eq:stateExample}
\end{equation}
We display the matrix $\A$, the associated binary matrix $\rm{supp}(\A)$ and the resulting directed graph under this interpretation.

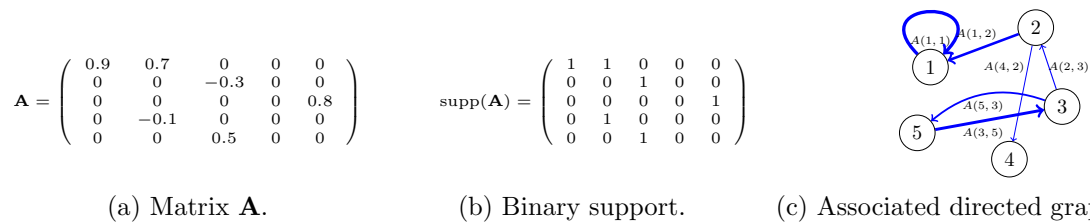
\begin{figure}
\centering
\begin{subfigure}{0.34\textwidth}
\tiny
 \begin{equation}
\label{eq:adj}
\A = 
\left(
\begin{array}{ccccc} 
0.9 & 0.7 & 0 & 0 & 0 \\
0 & 0 & -0.3 & 0 & 0 \\
0 & 0 & 0 & 0 & 0.8 \\
0 & -0.1 & 0 & 0 & 0 \\
0 & 0 & 0.5 & 0 & 0 
\end{array}
\right)\nonumber
 \end{equation}
     \caption{Matrix $\A$.}
    \label{fig:first}
\end{subfigure}
\hfill
\begin{subfigure}{0.23\textwidth}
\tiny
 \begin{equation}
\label{eq:adj}
\rm{supp}(\A) = 
\left(
\begin{array}{ccccc}
1 & 1 & 0 & 0 & 0 \\
0 & 0 & 1 & 0 & 0 \\
0 & 0 & 0 & 0 & 1 \\
0 & 1 & 0 & 0 & 0 \\
0 & 0 & 1 & 0 & 0 
\end{array}
\right)\nonumber
 \end{equation}
     \caption{Binary support.}
    \label{fig:first}
\end{subfigure}
\hfill
\begin{subfigure}{0.33\textwidth}
\hspace{1cm}
\scalebox{0.7}{
     \input{tikz_files/graph_A.tikz} 
		}
    \caption{Associated directed graph.}
    \label{fig:second}
\end{subfigure}        
\caption{Graphical model associated to \eqref{eq:stateExample}. Matrix $\A$ (a), its binary support (b) and associated directed graph (c). The edges are defined as non-zero entries of $\A^\top$. Non-zero diagonal entries result in self-loops (here, in vertex $1$). The thickness of arrows is proportional to the absolute entries of $\A^\top$.}
\label{fig:adj_A}
\end{figure}


 
\subsubsection{State noise precision matrix}

Matrix $\Q$ denotes the noise covariance in the state Eq. \eqref{eq_model_state}. Since the noise is assumed to be Gaussian, this matrix, and more precisely, the associated precision matrix $\P = \Q^{-1}$, also has a direct interpretation in terms of graphical modeling, using the notion of Gaussian graphical model (GGM) \citep[Section 13.4]{buhlmann2011statistics}\citep{Uhler2017GaussianGM}. Since we consider $\Q$ constant during the whole time series, let us denote the multivariate state noise r.v. at any time step as $\q \sim \mathcal{N}(0,\Q)$. The GGM consists in a graphical modeling of the independence (or not) between the scalar random variables $q(1),\ldots,q(N_x)$. It is easy to prove that,
\begin{equation}
q(n) \perp\!\!\!\perp q(\ell) | \{ q(j), j \in {1,\ldots,N_x} \backslash \{n,\ell\}\}  \Longleftrightarrow P(n,\ell) = P(\ell,n) =  0,   \label{eq:independance}
\end{equation}
i.e., the entries $n$ and $\ell$ of $\q$ are independent given all other entries if and only if the entry $P(n,\ell)$ is zero (and obviously also $P(\ell,n)$ since the precision matrix is symmetric). Note that it is possible to condition in the l.h.s. of \eqref{eq:independance} only to the entries $q(j)$ for which $P(n,j) \neq 0$ and the equivalence would still hold. \cblue{The GGM interprets the precision matrix (i.e., the inverse of the covariance matrix, $\P = \Sigmab^{-1}$) as a weighted undirected graph.} 
In particular, $(n,\ell) \notin \mathbb{E}$ if and only if $P(n,\ell) = P(\ell,n) = 0$. 

This GGM construction is at the core of the famous GLASSO (Graphical Lasso) formulation \citep{Friedman07}\citep[Section 9.7]{Lauritzen}, whose goal is to build the maximum a posteriori estimator of $\P$ given realizations of the random vector $\q$ under a sparsity assumption on matrix $\P$. The sparsity is here interpreted as a way to eliminate spurious edges in the graph associated to $\P$.  

\noindent\textbf{Illustrative example.} In Figure \ref{fig:adj_P}, we display an illustrative example on the GGM associated to a given precision matrix $\P$ for $N_x = 5$. We show the associated binary support matrix $\rm{supp}(\P)$ and the resulting undirected graph under this interpretation. Although self-loops (i.e., non-zero diagonal elements in $\P$) occur, we removed them from the graphical representation for ease of readability. 

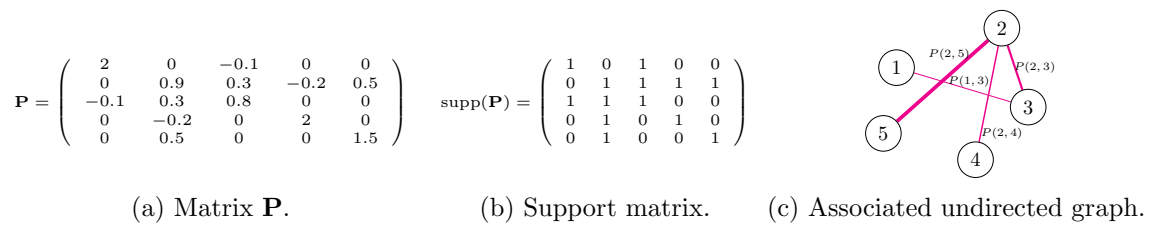
\begin{figure}
\centering
\begin{subfigure}{0.37\textwidth}
\tiny
 \begin{equation}
\label{eq:adj}
\P = 
\left(
\begin{array}{ccccc} 
2 & 0 & -0.1 & 0 & 0 \\
0 & 0.9 & 0.3 & -0.2 & 0.5 \\
-0.1 & 0.3 & 0.8 & 0 & 0 \\
0 & -0.2 & 0 & 2 & 0 \\
0 & 0.5 & 0 & 0 & 1.5 
\end{array}
\right)\nonumber
 \end{equation}
     \caption{Matrix $\P$.}
    \label{fig:first}
\end{subfigure}
\hfill
\begin{subfigure}{0.28\textwidth}
\tiny
 \begin{equation}
\label{eq:adj}
\rm{supp}(\P) = \left(
\begin{array}{ccccc}
1 & 0 & 1 & 0 & 0 \\
0 & 1 & 1 & 1 & 1 \\
1 & 1 & 1 & 0 & 0 \\
0 & 1 & 0 & 1 & 0 \\
0 & 1 & 0 & 0 & 1
\end{array}
\right)\nonumber
 \end{equation}
     \caption{Support matrix.}
\end{subfigure}
\hfill
\begin{subfigure}{0.33\textwidth}
\hspace{1cm}
\scalebox{0.7}{
     \input{tikz_files/graph_Q.tikz} 
		}
    \caption{Associated undirected graph.}
\end{subfigure}        
\caption{Matrix $\P$ (a), its binary support (b), and the associated undirected graph (c) with edge thickness proportional to the absolute entries of $\P$. Self-loops are removed, for readability purpose.}
\label{fig:adj_P}
\end{figure}

\subsubsection{Proposed unifying view}

We now summarize the graphical perspective on both $\A$ and $\Q$ and describe a unifying approach, where sparsity plays a key role. Matrix $\A$ is interpreted as the weight matrix of a directed graph with $N_x$ vertices. Sparsity (i.e., absence of edge in the graph) in $\A$ is interpreted as pair-wise {partial/}conditional independence, given a subset of the remaining time series, for a one-step ahead prediction of the hidden state. Matrix $\P = \Q^{-1}$ is interpreted as the weight matrix of an undirected graph, related to a GGM describing the noise in the latent space. Sparsity in $\P$ is interpreted as pair-wise {partial/}conditional independence of two dimensions of the additive state noise, given a subset of the remaining dimensions. \cblue{Both graphs have $N_x$ nodes (i.e., $N_x$ vertices), and a maximum of $N_x^2$ edges for $\A$ (resp. $N_x(N_x-1)$ for $\P$), possibly including self-loops, associated to weights defined as the entries of $\A$ or $\P$}. 

Our perspective in the state process of the LG-SSM in \eqref{eq_model_state} is that $\A$ encodes the way the information flows in consecutive time-steps between the nodes (state dimensions) of the network (vector state). Thus, its properties shape how the energy/information is transferred and dissipated (under the noise). In contrast, $\P = \Q^{-1}$ encodes how information that is not in the system at time $k-1$ enters in the system at time $k$. In that respect, the interpreted graph with weight matrix $\P$ encodes the dependency of the new information across the nodes of the network. 

We adopt the above perspective to estimate both $\A$ and $\Q$ by promoting properties in both graphs. Specifically, we introduce sparsity priors on the matrices, as the sparsity property is key to reach interpretability and compactness of the whole model. In particular, it allows to understand the inner structure of the latent space. Moreover, it can be helpful to speed up computations as the sparsity level is increased, e.g., when running the Kalman filter and RTS smoother. Our proposed method \acro~(Dynamic Graphical Lasso) hence aims at providing the maximum a posteriori (MAP) estimator of $\A$ and $\P$ (i.e., the weight matrices related to the graphical modeling of the latent state correlation and causality) under Lasso sparsity regularization on both matrices, given the observed sequence $\y_{1:K}$. A visual representation of \acro~graphical model is given in Figure~\ref{fig:graph_both}. The figure summarizes the relationships among the state entries of an LG-SSM using matrices $(\A,\P)$ from Figures~\ref{fig:adj_A} and \ref{fig:adj_P}.} 

\paragraph*{Related works:} Our approach \acro~generalizes important existing sparse graphical inference ones. For instance, our model with $\A = 0$ (degenerate case) has no memory, and all the energy/information of the system is lost at each time step, thus the state dimensions only incorporate exogenous energy/information through the additive noises. {This degenerate case is the same model than GLASSO \citep{Friedman07} in the case when $\R_k \equiv 0$, and same than the robust GLASSO model \citep[Sec.5.2]{Benfenati18} when $\R_k \equiv \sigma_{\R}^2 \bf{Id}_{N_y}$. In contrast, if the state noise covariance matrix $\Q$ is known, \acro~coincides with our recent GraphEM framework \citep{elvira2022graphem}. Probably the closer related work is~\citep{Ioannidis2019}, which also introduces a joint graph modeling within an LG-SSM, capturing order-one causal relationships and instantaneous influence (i.e., order zero), through two sparse graphs. Their proposed inference method is an alternating optimization technique, that infers the two graphs under Lasso prior, jointly with the estimation of hidden state. In contrast with \acro, in~\citep{Ioannidis2019}, (i) the state model follows a structural vector autoregressive model (SVAR) where instantaneous causality and noise are distinguished, {while \acro~assumes an order-one VAR in the hidden state}; and} 
(ii) the cost function does not result from a Bayesian modeling, and as such it is not related to a maximum a posteriori loss for the graph variables, (iii) the state estimation is point wise defined as the solution of an handcrafted optimization problem, while \acro~preserves a full Bayesian interpretation and hence allows the complete characterization of the filtering/smoothing state distributions. In particular,~\citep{Ioannidis2019} model does not recover GLASSO as a particular case.

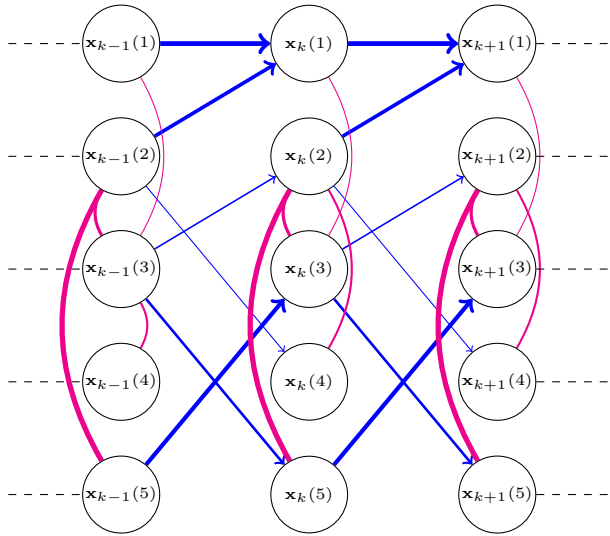
\begin{figure}[h]
\centering
\input{tikz_files/graph_both.tikz}
\caption{Summary representation of the \acro~graphical model, for the example graphs presented in Figs.~\ref{fig:adj_A} and~\ref{fig:adj_P}. Blue (oriented) edges represent Granger causality between state entries among consecutive time steps, encoded in matrix $\A$ (Fig.~\ref{fig:adj_A}). Magenta edges represent static (i.e., instantaneous) relationships between the state entries, at every time step, due to correlated state noise described by matrix $\P$ (Fig.~\ref{fig:adj_P}).}
\label{fig:graph_both}
\end{figure}

\subsection{Optimization problem}

The considered MAP inference problem in DGLASSO reads as an optimization problem that we formulate hereafter. More specifically, let us denote the posterior of the unknown parameter, $p(\A,\P|\y_{1:K})$, where the hidden states have been marginalized. It is direct to show, using Bayes rule and composition with the (strictly increasing) logarithmic function, that the maximum of $p(\A,\P|\y_{1:K}) \propto p(\A,\P)p(\y_{1:K}|\A,\P)$, with $p(\A,\P)$ some prior on the parameters $\A$ and $\P$, coincides with the minimum of the following loss function:
\begin{align}
(\forall \A \in \mathbb{R}^{N_x \times N_x})(\forall \P \in \mathcal{S}_{N_x}) \quad 
\loss(\A,\P) & \triangleq \loss_{1:K}(\A,\P) + \loss_0(\A,\P).
    \label{eq:phik}
    \end{align}
		with $\loss_{1:K}(\A,\P) \triangleq - \log p(\y_{1:K}| \A,\P)$ {and $\loss_{0}(\A,\P) = -\log p(\A,\P)$}.  According to \citep[Chap. 12]{Sarkka}, 
\begin{equation}
(\forall \A \in \mathbb{R}^{N_x \times N_x})(\forall \P \in \mathcal{S}_{N_x}) \quad
\loss_{1:K}(\A,\P) =  \sum_{k=1}^K \tfrac{1}{2} \log \det( 2 \pi \S_k) + \frac{1}{2} \z_k^\top \S_k^{-1} \z_k,
\label{eq:lossimplicit}
\end{equation}
where $\z_k = \y_k - \H_k \A\mub_{k-1}$  and $\S_k$ is the covariance matrix of the predictive distribution $p(\y_{k}|\y_{1:k-1}, \A,\P)= \mathcal{N}\left(\y_{k};\H\A\mub_{k-1},\S_{k}\right)$, both being obtained by the KF in Algorithm~\ref{alg_kf} run for given $(\A,\P)$ (see \cite[Section 4.3]{Sarkka}). 

As already mentioned, the introduction of priors that induce sparsity is advantageous due to several reasons. First, it generally reduces over-fitting, particularly when $K$ is low compared to the number of parameters to be estimated. Also, it enhances the interpretability. The zero elements in $\A$ and $\P$ have a clear interpretation of conditional independence between the time series. Ideally, we would use the $\ell_0$ (pseudo) norm of $\A$ and $\P$, i.e., penalizing the number of non-zero entries of the matrix. However, this penalty is known to have undesirable properties such as being non-convex, non continuous, and associated to a improper law $p(\A)$. We thus propose instead the regularization term
\begin{equation}
(\forall \A \in \mathbb{R}^{N_x \times N_x})(\forall \P\in \mathcal{S}_{N_x}) \quad \loss_0(\A,\P) = \lambda_A \|\A\|_1 + \lambda_P \|\P\|_1.\label{eq:priorgen}
\end{equation}
The $\ell_1$ norm in \eqref{eq:priorgen}, defined as the sum of absolute values of the matrix entries, is a proper convex function, that leads to the so-called Lasso regularization \citep{BachSparse}. Note that this penalty, used in numerous works of signal processing and machine learning~\citep{Tibshirani,Chaux_2007}, including graph signal processing \citep{Friedman07,Benfenati18}, is associated with a joint Laplace prior distribution on $\A$ and $\P$. Such joint distribution factorizes (i.e., the prior assumes independence on both parameters), the means are zero, and the scale parameters are proportional to, respectively, $\lambda_A$ and $\lambda_P$. The larger the regularization parameter $\lambda_A$ (or $\lambda_P)$, the higher sparsity of $\A$ (or $\P)$, with the degenerate case of a null $\A$ (and $\P$) when the regularization parameter grows.

\subsection{General minimization procedure}

The expressions \eqref{eq:lossimplicit}-\eqref{eq:priorgen} provide an effective way to evaluate~\eqref{eq:phik}. However, due to the recursive form in \eqref{eq:lossimplicit}, it is challenging to derive direct quantities (e.g., gradient) for $\loss_{1:K}$. Moreover, despite its simple expression, the regularization term \eqref{eq:priorgen} is non differentiable. For both reasons, the minimization of~\eqref{eq:phik} is a challenging question. 
 
We propose a block alternating majorization-minimization (MM) technique to infer the MAP estimates of $(\A,\P)$. Our method presents the advantage of sound convergence guarantees and the ability to incorporate sparsity priors on both $\A$ and $\P$. The general idea of MM is to replace a complicated optimization problem by a sequence of more tractable ones \citep{Sun2016,Hunter2004}. Surrogate approximations of the cost function are built iteratively, following a majorization principle. For any estimates $\widetilde{\A} \in \mathbb{R}^{N_x \times N_x} $ and $\widetilde{\P} \in \mathcal{S}_{N_x}^{++}$ (i.e., in the interior domain of definition of $\loss$) of $(\A,\P)$, a majorizing approximation is constructed for the likelihood term. {It is required to satisfy both}
\begin{equation}
(\forall \A \in \mathbb{R}^{N_x \times N_x})(\forall \P\in \mathcal{S}_{N_x}) \quad
\maj(\Ab,\P;\widetilde{\A},\widetilde{\P}) \geq \loss_{1:K}(\A,\P),
\label{eq:majorant}
\end{equation}
and also the so-called tangency condition
\begin{equation}
\maj(\widetilde{\A},\widetilde{\P};\widetilde{\A},\widetilde{\P}) =  \loss_{1:K}(\widetilde{\A},\widetilde{\P}).
\label{eq:tangency}
\end{equation}
The MM algorithm then alternates between \textbf{M}ajorization step to build function $\mathcal{Q} + \loss_0$ satisfying conditions \eqref{eq:majorant} and \eqref{eq:tangency}, and \textbf{M}inimization step to minimize this majorizing approximation. In our proposed approach, we adopt a block alternating implementation of the MM method, where only one variable (i.e., $\A$ or $\P$) is updated at each iteration, following a cyclic rule \citep{jacobson,Hong2016}. This alternating strategy considerably simplifies the majorizing function construction as well as its minimization. We furthermore add so-called proximal terms in both updates. Let us recall that, for a function $\phi: \mathcal{H} \mapsto (-\infty,+\infty] \in \Gamma_0(\mathcal{H})$, with $\mathcal{H}$ a Hilbert space with endowed norm $\| \cdot \|$, the proximity operator\footnote{See also \url{http://proximity-operator.net/}} of function $f$ at $\widetilde{\V} \in \mathcal{H}$ is defined as~\citep{Combettes2011}
\begin{equation}
    \operatorname{prox}_\phi(\widetilde{\V}) = \argmin{\V \in \mathcal{H}}{\left(\phi(\V) + \frac{1}{2}\| \V - \widetilde{\V}\|^2 \right)}.
		\label{eq:defprox}
\end{equation} 
In our context, the considered Hilbert space is either $\mathbb{R}^{N_x \times N_x}$ for the update of the transition matrix or $\mathcal{S}_{N_x}$ for the update of the precision matrix, and the endowed norm is in both cases the Frobenius norm $\| \cdot \|_F$. Introducing proximity terms thus amounts to adding to each majorizing function a quadratic distance to the previous iterate, weighted by a positive factor. This modification preserves the MM interpretation of the method, while ensuring improved stability and convergence guarantees. {As we show below, the iterates belong to the interior of domain of the loss function by construction. Namely for every $i \in \mathbb{N}$, $(\A^{(i)},\P^{(i)}) \in \mathbb{R}^{N_x \times N_x} \times \mathcal{S}_{N_x}^{++}$,  so the precision matrix remains invertible along the iterations and the algorithm is well defined. The resulting \acro~approach is summarized in Algorithm~\ref{algorithm}. \acro~aims at providing ultimately the MAP estimates for the matrix parameters $(\A,\P)$ of the considered LG-SSM, through the minimization of~\eqref{eq:phik}. The covariance state noise MAP estimate is straightforwardly obtained by inversion of the precision state noise matrix provided as an output of~\acro. The state filtering/smoothing pdf associated to each estimates are finally computed by running KF/RTS loops when setting $(\A,\P)$ equal to \acro~outputs. Next sections are dedicated to the (i) construction of the majorizing function, (ii) discussion about the resolution of each inner step, and (iii) convergence analysis.}

\begin{table}[!t]
\vspace{4mm}
    \centering
    \begin{tabular}{|p{0.95\columnwidth}|}
    \hline
\begin{Algoritmo}
\label{algorithm}
\acro~algorithm
\begin{enumerate}
  \item[] \textbf{Inputs.} Prior parameters $\meanPrior_0$ and  $\covPrior_0$; model parameters $\{\H_k \}_{k=1}^K$  and $\{\R_k \}_{k=1}^K$; set of observations $\{\y_k \}_{k=1}^K$; hyper-parameters $(\lambda_A,\lambda_P)>0$; stepsizes $(\theta_A,\theta_P)>0$; precisions $(\varepsilon,\xi)>0$. 
\item[] \textbf{Initialization.} Set $(\A^{(0)},\P^{(0)}) \in \mathbb{R}^{N_x \times N_x} \times \mathcal{S}_{N_x}^{++}$.
\item[] \textbf{Recursive step.} For $i=0,1,\ldots$:  
 \begin{enumerate}
\item \texttt{//Update transition matrix}
\item[(i)] Build surrogate function satisfying \eqref{eq:majorant} and \eqref{eq:tangency} at $\widetilde{\A} = \A^{(i)}$ and $\widetilde{\P} = \P^{(i)}$.
\item[(ii)] Run Algorithm \ref{algo:ProxA} with precision $\xi$ to solve
\begin{equation}
\begin{array}{ll}
\A^{(i+1)} & = \operatorname{prox}_{\A \to \theta_A  \maj(\A,\P^{(i)};\A^{(i)},\P^{(i)}) + \theta_A \lambda_A \|\A\|_1 }\left(\A^{(i)}\right),
 \\
 & = \argmin{\A \in \mathbb{R}^{N_x \times N_x}}{\maj(\A,\P^{(i)};\A^{(i)},\P^{(i)})  + \lambda_A \|\A\|_1 + \frac{1}{2 \theta_A} \|\A - \A^{(i)}\|^2_F}.
\label{eq:updateA}
\end{array}
\end{equation}
\item \texttt{//Update noise precision matrix}
\item[(i)] Build surrogate function satisfying \eqref{eq:majorant} and \eqref{eq:tangency} at $\widetilde{\A} = \A^{(i+1)}$ and $\widetilde{\P} = \P^{(i)}$.
\item[(ii)] Run Algorithm \ref{algo:ProxP} with precision $\xi$ to solve
\begin{equation}
\begin{array}{ll}
 \P^{(i+1)} & = \operatorname{prox}_{\P \to \theta_P \maj(\A^{(i+1)},\P;\A^{(i+1)},\P^{(i)}) + \theta_P \lambda_P \|\P\|_1  }\left(\P^{(i)}\right),
\\
&= \argmin{\P \in \mathcal{S}_{N_x}}{\maj(\A^{(i+1)},\P;\A^{(i+1)},\P^{(i)})  + \lambda_P \|\P\|_1 + \frac{1}{2 \theta_P} \|\P - \P^{(i)}\|^2_F}.
\end{array}
\label{eq:updateP}
\end{equation}
 \end{enumerate}
\item[] If $\|\A^{(i+1)} - \A^{(i)} \|_F \leq \varepsilon \| \A^{(i)} \|_F$ \textbf{and} 
$\|\P^{(i+1)} - \P^{(i)} \|_F \leq \varepsilon \| \P^{(i)} \|_F$, \textbf{stop the recursion} by returning $(\A^{(i+1)},\P^{(i+1)})$.
 %
\item[] \textbf{{Output.}} MAP estimates of the transition and state noise precision matrices.
\end{enumerate}
\end{Algoritmo}\\
        \hline
\end{tabular}
\end{table}
%


\subsection{Building the majorizing function}


In this section, we derive a theorem regarding the expression of the loss function $\loss_{1:K}$ and a valid majorant function for it. 

\begin{mytheorem}
The loss function can be expressed as\footnote{\cblue{Note that the l.h.s. in \eqref{eq:lossexplicit} does not depend on $\x_{0:K}$, so the r.h.s. is valid for any arbitrary value of $\x_{0:K}$ with non-zero probability under $p(\x_{0:K})$, i.e., for all $\x_{0:K} \in \Real^{(K+1)N_x}$. }}
\begin{multline}
(\forall \A \in \mathbb{R}^{N_x \times N_x})(\forall \P \in \mathcal{S}_{N_x}) \quad
\loss_{1:K}(\A,\P) =  \frac{{1}}{2}\sum_{k=1}^K \left(\left(\x_k - \A \x_{k-1} \right)^\top \P  \left(\x_k - \A \x_{k-1}  \right) \right) \\
- \frac{K}{2}\log \det ({2\pi} {\P}) + \log p(\x_{0:K} | \y_{1:K},\A,\P) + \log p(\x_0) - \sum_{k=1}^K p(\y_k | \x_k).
\label{eq:lossexplicit}
\end{multline}
Moreover, consider the outputs of Algorithms~\ref{alg_kf} and~\ref{alg_rts}  for a given $\widetilde{\A} \in \mathbb{R}^{N_x \times N_x}$ and $\widetilde{\P} \in \mathcal{S}_{N_x^{++}}$. Denote
\begin{equation}
\begin{cases}
\cblue{\widetilde{\Psib}} & = \frac{1}{K} \sum_{k=1}^K  \left(\covSmooth_k + \meanSmooth_{k} (\meanSmooth_{k})^\top\right),\\
\cblue{\widetilde{\Deltab}} & = \frac{1}{K} \sum_{k=1}^K  \left(\covSmooth_k \G_{k-1}^\top  + \meanSmooth_{k} (\meanSmooth_{k-1})^\top\right),\\
\cblue{\widetilde{\Phib}} & = \frac{1}{K} \sum_{k=1}^K  \left(\covSmooth_{k-1} + \meanSmooth_{k-1} (\meanSmooth_{k-1})^\top\right),
\end{cases}
\label{eq:psideltaphi}
\end{equation}
where, for every $k \in\{1,\ldots,K\}$, $\G_k = \Sigmab_k \widetilde{\A}^\top (\widetilde{\A} \Sigmab_k \widetilde{\A}^\top + \widetilde{\P}^{-1})^{-1}$ (see Algorithm~\ref{alg_rts}). Then, conditions \eqref{eq:majorant} and \eqref{eq:tangency} hold with
\begin{multline}
(\forall \A \in \mathbb{R}^{N_x \times N_x})(\forall \P \in \mathcal{S}_{N_x}) \quad 
\maj(\Ab,\P;\widetilde{\A},\widetilde{\P}) = \frac{K}{2} \operatorname{tr}\left(\P(  
\cblue{\widetilde{\Psib}} - \cblue{\widetilde{\Deltab}} \Ab^\top - \Ab \cblue{\widetilde{\Deltab}}^\top + \Ab \cblue{\widetilde{\Phib}} \Ab^\top) \right) \\ - \frac{K}{2}\log \det({2\pi} {\P}).
\label{eq:majoranteQ}
\end{multline}
As a consequence, for every $(\theta_A,\theta_P)>0$,
\begin{multline}
(\forall \A \in \mathbb{R}^{N_x \times N_x})(\forall \P \in \mathcal{S}_{N_x}) \quad 
\loss(\A,\P) \leq \maj(\Ab,\P;\widetilde{\A},\widetilde{\P}) + \loss_{1:K}(\widetilde{\A},\widetilde{\P}) - \maj(\widetilde{\A},\widetilde{\P};\widetilde{\A},\widetilde{\P})\\
 + \lambda_A \| \A\|_1 + \lambda_P \|\P\|_1 + \frac{1}{2 \theta_A} \|\A - \widetilde{\A}\|^2_F + \frac{1}{2 \theta_P} \|\P - \widetilde{\P}\|^2_F,
\label{eq:proxmaj}
\end{multline}
with equality holding for $\A = \widetilde{\A}$ and $\P = \widetilde{\P}$. 
\label{theo:maj}
\end{mytheorem}
\begin{proof}
See Appendix~\ref{sec:appendixA}.
\end{proof}

Theorem~\ref{theo:maj} allows to build, for any tangent point $(\widetilde{\A},\widetilde{\P})$, a majorizing approximation \eqref{eq:proxmaj} for $\mathcal{L}$. \cblue{Function \eqref{eq:proxmaj} depends on three matrices $(\cblue{\widetilde{\Psib}},\cblue{\widetilde{\Deltab}},\cblue{\widetilde{\Phib}})$ through \eqref{eq:psideltaphi}, themselves depending on the tangent point $(\widetilde{\A},\widetilde{\P})$, as highlighted by the tilde symbol.} DGLASSO method leverages this property designing a block alternating MM scheme. At each iteration $i \in \mathbb{N}$, two steps are conducted, {namely the update of (a) the transition matrix, (b) the noise precision matrix.} Each steps follows an MM structure, that is it first builds a majorizing approximation for $\mathcal{L}$ at the current estimates, using Theorem~\ref{theo:maj}, and then minimizes it with respect to the active variable ($\A$ in step (a), or $\P$ in step (b)). Processing the variables in two distinct steps allows to build upon the desirable convex structure of \eqref{eq:proxmaj} with respect to one of the variable, the other being fixed. The good performance of MM approaches combined with block alternating steps have been illustrated in \citep{Hong2016,Chouzenoux2016jogo,HienICML2020}. In particular, convergence guarantees are at reach, as we will show in Section \ref{sec:convergence}. 



\subsection{Resolution of the inner problems}
\label{sec:subproblems}

We now discuss the structure and resolution of the inner problems \eqref{eq:updateA} and \eqref{eq:updateP} arising in Algorithm~\ref{algorithm}. 

\paragraph{Minimization with respect to $\A$:}
Let $\widetilde{\A} \in \mathbb{R}^{N_x \times N_x}$ and $\widetilde{\P} \in \mathcal{S}_{N_x}^{++}$. By definition of the proximity operator \eqref{eq:defprox} and the majorant expression in \eqref{eq:majoranteQ}, Eq. \eqref{eq:updateA} \cblue{requires to}
\cblue{
\begin{equation}
\text{minimize}_{\A \in \mathbb{R}^{N_x \times N_x}}{\mathcal{C}_1(\A)}, \label{eq:updateAexplicit}
\end{equation}
where, for every $\A \in \mathbb{R}^{N_x \times N_x}$,
\begin{equation*}
\mathcal{C}_1(\A) \triangleq \frac{\theta_A K}{2} \text{tr}\left(\widetilde{\P}(  
\cblue{\widetilde{\Psib}} - \cblue{\widetilde{\Deltab}} \Ab^\top - \Ab \cblue{\widetilde{\Deltab}}^\top + \Ab \cblue{\widetilde{\Phib}} \Ab^\top) \right)  + \theta_A \lambda_A \|\A\|_1 + \frac{1}{2 } \| \A - \widetilde{\A}\|^2_F.
\end{equation*}
}


Remarkably, the problem above is a special instance of a multivariate Lasso regression problem \citep{Tibshirani}, for which many efficient iterative solvers are available. The specificity here is that the problem is strongly convex thanks to the proximal term. We thus suggest the use of the Dykstra-like algorithm by \cite{Dystra2008}, whose iterations are recalled in the Appendix \ref{sec:appendixB}. This method presents the advantage of fast convergence rate, ease of implementation, and no parameter tuning.  

\paragraph{Minimization with respect to $\P$:}
Let $\widetilde{\A} \in \mathbb{R}^{N_x \times N_x}$ and $\widetilde{\P} \in \mathcal{S}_{N_x}^{++}$. {The update of Eq. \eqref{eq:updateP} solves a minimization problem with generic form}
\cblue{
\begin{equation}
\text{minimize}_{\P \in \mathcal{S}_{N_x}} \mathcal{C}_2(\P), \label{eq:updatePexplicit}
\end{equation}
where, for every $\P \in \mathcal{S}_{N_x}$, we denote
\begin{equation*}
\mathcal{C}_2(\P) \triangleq
\frac{\theta_P K}{2} \text{tr}\left(\P \cblue{\widetilde{\Pib}} \right) - \frac{\theta_P K}{2}\log \det{{\P}} + \theta_P \lambda_P \| \P \|_1  + \frac{1}{2 } \| \P - \widetilde{\P}\|^2_F,
\end{equation*}
}
\begin{equation}
\cblue{\widetilde{\Pib}} \triangleq 
\cblue{\widetilde{\Psib}} - \cblue{\widetilde{\Deltab}} \widetilde{\Ab}^\top - \widetilde{\Ab}  \cblue{\widetilde{\Deltab}}^\top + \widetilde{\Ab}  \cblue{\widetilde{\Phib}} \widetilde{\Ab} ^\top.
\label{eq:Pib}
\end{equation}
{Here we have used the definition of the proximity operator \eqref{eq:defprox} and the majorant expression in \eqref{eq:majoranteQ} (ignoring the constant multiplicative term in the logarithm).} 
 Remarkably, \eqref{eq:updatePexplicit} reads as a regularized form of the famous GLASSO problem \citep{Friedman07}, and gets actually equivalent to it when $\theta_P \to \infty$. Matrix $\cblue{\widetilde{\Pib}}$ in \eqref{eq:Pib} plays the same role as the empirical covariance matrix in GLASSO, and $\frac{2 \lambda_P}{K}$ acts as the weight on the $\ell_1$ term. The proximal term works as a Tikhonov-like regularizer, ensuring the strong convexity of the problem, and thus the uniqueness of its minimizer. {Moreover, by the definition of the log-determinant in Eq. \eqref{eq:logdet}}, the solution of \eqref{eq:updatePexplicit} belongs to $\P \in \mathcal{S}_{N_x}^{++}$, i.e., the precision matrix is symmetric and invertible, and thus a valid covariance matrix can be deduced from it by inversion. Standard GLASSO solvers can be easily modified to solve \eqref{eq:updatePexplicit}. We present in the Appendix \ref{sec:appendixB} the complete derivations, when applying the Dykstra-like algorithm from \cite{Dystra2008}, which presents the advantage of fast convergence, and ease of implementation.

\subsection{\cblue{Discussion}}

\paragraph{\cblue{Strengths and weaknesses of the MM method:}} \cblue{As highlighted in our introductory section, state-of-the-art approaches for SSM parameter estimation are grounded either on the EM (or MM) framework or on gradient (or Newton) updates (through sensitivity equations). In the present work, we adopt the former strategy. We here discuss the pros and cons for such choice for the minimization of \eqref{eq:phik}, some aspects being also illustrated by the numerical experiments presented in Section \ref{sec:complex}.}

\cblue{A potential limitation of the proposed minimization algorithm is its computational complexity. From Theorem~\ref{theo:maj}, the construction of the majorizing function at each tangent point $(\widetilde{\A},\widetilde{\P})$ requires to run the KF Algorithm~\ref{alg_kf} and the RTS Algorithm~\ref{alg_rts}. This is indeed necessary for the computation, using \eqref{eq:psideltaphi}, of the three matrices, $(\cblue{\widetilde{\Psib}},\cblue{\widetilde{\Deltab}},\cblue{\widetilde{\Phib}})$. This feature was already present in previous EM-based approaches for LG-SSM parameter estimations (see for instance \cite[Chap. 12]{Sarkka} and \cite{elvira2022graphem}). Due to its block alternating form, the proposed DGLASSO algorithm requires to build twice per iteration a majorizing approximation, which means that KF/RTS are ran twice per iteration. Moreover, inner solvers are required, for the minimization of the majorant functions, increasing again the method complexity. In contrast, gradient-based methods relying on sensitivity equations \citep{Gupta1974} only require the KF recursions (and not the RTS ones) for building their updates, and do not make use of inner iterations, which might seem appealing from the computational viewpoint.}
 
\cblue{Gradient-based approaches however face some difficulties. The $\ell_1$ norm is non differentiable, preventing from the use of standard first (or second) order minimization techniques as those proposed in \cite{Olsson}. This could be fixed either by smoothing the $\ell_1$ terms at the price of adding extra hyperparameters or by implementing suitable updates (e.g., subgradient or proximal operations). An additional issue comes from the logarithmic determinant term on $\P$, involved in the likelihood term \eqref{eq:lossexplicit}. This makes the definition domain of the loss function restricted to the set of symmetric positive definite matrices $\P$. Moreover, onto this set, the gradient loss is not Lipschitz, and its norm explodes for matrices $\P$ close to non invertible. DGLASSO accounts for the singularity of the log-det term thanks to proximal-based updates in the inner solver of Algorithm~\ref{algo:ProxP} detailed in Appendix \ref{sec:appendixB}. In contrast, gradient-based solvers require a tedious manual stepsize tuning (typically, backtracking linesearch), with no proved guarantee of the existence of a valid stepsize. As a consequence, we observed in our experiments in Sec.~\ref{sec:complex} that a very large number of iterations was required to reach stability for these methods, which largely undermines their apparent advantage in terms of iteration complexity. }


\cblue{Finally, an important feature of the proposed method is that it inherits the sound convergence properties from the MM paradigm, as detailed in our Section \ref{sec:convergence}. DGLASSO iterates are proved to reach a cluster point of the MAP loss function. No such results can be derived for the gradient-based competitors, due to the non-convexity of the loss, and the singularity of its gradient, as we already mentioned.  
}


\paragraph{\cblue{Regularization parameters tuning:}} \cblue{DGLASSO formulation requires the setting of the weights $(\lambda_A,\lambda_P)$ balancing the sparsity prior and the data fidelity. The automatic tuning of DGLASSO regularization hyperparameters in the context of real data is a challenging problem. Among promising avenues, let us mention several recent approaches based on supervised learning, either implementing deep learning architectures~\citep{Shrivastava2020,Revach2022,Buchnik2023,Shrivastava2022}, or bi-level approaches~\citep{Pouliquen2023,Franceshi2017,Bertrand2022} combined with SURE statistical estimators~\citep{Deledalle2014,Luo2014}. Note however that the aforementioned works consider either static graphical models such as GLASSO, or Kalman filtering without graph formulation. Their extension to our DGLASSO model remains an open question that we leave as a future research line. In our experimental section, synthetic data will be considered, allowing an empirical finetuning through qualitative metrics computed on the ground truth model.} 

\section{Convergence analysis}
\label{sec:convergence}

We now present our convergence proof for the proposed DGLASSO algorithm presented in Algorithm~\ref{algorithm}. Our analysis is made assuming that the inner steps \eqref{eq:updateA} and \eqref{eq:updateP} are solved in an exact manner. The extension of the result to the case of an inexact resolution of the subproblems is discussed at the end of the section. 

\subsection{Descent property}

\begin{mylemma}
Assuming exact resolution of \eqref{eq:updateA} and \eqref{eq:updateP}, the sequence $\{\A^{(i)},\P^{(i)}\}_{i \in \mathbb{N}}$ produced by \acro~algorithm satisfies 
\begin{equation}
(\forall i \in \mathbb{N}) \quad \loss(\A^{(i+1)},\P^{(i+1)}) \leq \loss(\A^{(i)},\P^{(i)}). \label{eq:desc0}
\end{equation}
\label{lemma:descent}
\end{mylemma}
\begin{proof}
Let $i \in \mathbb{N}$. First, let us \cblue{use the} Theorem~\ref{theo:maj} at $(\widetilde{\A},\widetilde{\P}) = (\A^{(i+1)},\P^{(i)})$ (Inequality (a)), and the definition of $\P^{(i+1)}$ in \eqref{eq:updateP} (Inequality (b)) as
\begin{align}
\loss(\A^{(i+1)},\P^{(i+1)}) & \overset{(a)}{\leq} \maj(\A^{(i+1)},\P^{(i+1)};\A^{(i+1)},\P^{(i)}) + \loss_{1:K}(\A^{(i+1)},\P^{(i)}) - \maj(\A^{(i+1)},\P^{(i)};\A^{(i+1)},\P^{(i)}) \nonumber \\
& + \lambda_A \| \A^{(i+1)}\|_1 + \lambda_P \|\P^{(i+1)}\|_1 + \frac{1}{2 \theta_A} \|\A^{(i+1)} - \A^{(i+1)}\|^2_F + \frac{1}{2 \theta_P} \|\P^{(i+1)} - \P^{(i)}\|^2_F\\
& \overset{(b)}{\leq} \maj(\A^{(i+1)},\P^{(i)};\A^{(i+1)},\P^{(i)}) + \loss_{1:K}(\A^{(i+1)},\P^{(i)}) - \maj(\A^{(i+1)},\P^{(i)};\A^{(i+1)},\P^{(i)}) \nonumber \\
 & + \lambda_A \| \A^{(i+1)}\|_1 + \lambda_P \|\P^{(i)}\|_1 + \frac{1}{2 \theta_A} \|\A^{(i+1)} - \A^{(i+1)}\|^2_F + \frac{1}{2 \theta_P} \|\P^{(i)} - \P^{(i)}\|^2_F. \label{eq:ineqdesc1}
\end{align}
Inequality~\eqref{eq:ineqdesc1} simplifies into
\begin{equation}
\loss(\A^{(i+1)},\P^{(i+1)}) \leq \loss_{1:K}(\A^{(i+1)},\P^{(i)}) + \lambda_A \| \A^{(i+1)}\|_1 + \lambda_P \|\P^{(i)}\|_1 = \loss(\A^{(i+1)},\P^{(i)}).
\end{equation}
Applying Theorem~\ref{theo:maj} now at $(\widetilde{\A},\widetilde{\P}) = (\A^{(i)},\P^{(i)})$ (Inequality (a)), and the definition of $\A^{(i+1)}$ in \eqref{eq:updateA} (Inequality (b)), leads to
\begin{align}
\loss(\A^{(i+1)},\P^{(i)}) & \overset{(a)}{\leq} \maj(\A^{(i+1)},\P^{(i)};\A^{(i)},\P^{(i)}) + \loss_{1:K}(\A^{(i)},\P^{(i)}) - \maj(\A^{(i)},\P^{(i)};\A^{(i)},\P^{(i)}) \nonumber \\
& + \lambda_A \| \A^{(i+1)}\|_1 + \lambda_P \|\P^{(i)}\|_1 + \frac{1}{2 \theta_A} \|\A^{(i+1)} - \A^{(i)}\|^2_F + \frac{1}{2 \theta_P} \|\P^{(i)} - \P^{(i)}\|^2_F\\
& \overset{(b)}{\leq} \maj(\A^{(i)},\P^{(i)};\A^{(i)},\P^{(i)}) + \loss_{1:K}(\A^{(i)},\P^{(i)}) - \maj(\A^{(i)},\P^{(i)};\A^{(i)},\P^{(i)}) \nonumber \\
 & + \lambda_A \| \A^{(i)}\|_1 + \lambda_P \|\P^{(i)}\|_1 + \frac{1}{2 \theta_A} \|\A^{(i)} - \A^{(i)}\|^2_F + \frac{1}{2 \theta_P} \|\P^{(i)} - \P^{(i)}\|^2_F,
\end{align}
which simplifies into
\begin{equation}
\loss(\A^{(i+1)},\P^{(i)}) \leq \loss_{1:K}(\A^{(i)},\P^{(i)}) + \lambda_A \| \A^{(i)}\|_1 + \lambda_P \|\P^{(i)}\|_1 = \loss(\A^{(i)},\P^{(i)}),
\end{equation}
and concludes the proof.
\end{proof}
If the cost function $\loss$ is lower bounded (e.g., if it is coercive), Lemma \ref{lemma:descent} implies the convergence of sequence $ \left\{\loss(\A^{(i)},\P^{(i)})\right\}_{i \in \mathbb{N}}$ to a finite value and, as such, the existence of cluster points in $\{\A^{(i)},\P^{(i)}\}_{i \in \mathbb{N}}$. This is however a rather weak convergence result and we propose hereafter a thorough analysis relying on recent tools of nonlinear analysis \citep{Attouch2010,BoltePALM} combined with the works \citep{phan2023inertial,HienICML2020} on the convergence of block alternating MM schemes. 

\subsection{Convergence guarantees}

\begin{mytheorem}
Consider the sequence $\{\A^{(i)},\P^{(i)}\}_{i \in \mathbb{N}}$ generated by DGLASSO, assuming exact resolution of both inner steps \eqref{eq:updateA} and \eqref{eq:updateP}. If the sequence $\{\A^{(i)},\P^{(i)}\}_{i \in \mathbb{N}}$ is bounded, then $\{\A^{(i)},\P^{(i)}\}_{i \in \mathbb{N}}$ converges to a critical point of $\loss$.
\label{theo:convergence}
\end{mytheorem}

\begin{proof}
The convergence analysis relies in proving that the exact form of DGLASSO algorithm is a special instance of TITAN algorithm from \citep{phan2023inertial}, and as such, inherits \citep[Theorem 6]{phan2023inertial} and \citep[Theorem 8]{phan2023inertial}, under our assumptions.

$\bullet$ Let us introduce the following notations. 
\begin{multline}
(\forall \A \in \mathbb{R}^{N_x \times N_x})(\forall \P \in \mathcal{S}_{N_x})  \quad
f(\A,\P) =  \frac{{1}}{2}\sum_{k=1}^K \left(\left(\x_k - \A \x_{k-1} \right)^\top \P  \left(\x_k - \A \x_{k-1}  \right) \right) 
\\ + \log p(\x_{0:K} | \y_{1:K},\A,\P) - \log p(\x_0) - \sum_{k=1}^K \log p(\y_k | \x_k),
\label{eq:lossminuslog}
\end{multline}
\begin{align}
(\forall \A \in \mathbb{R}^{N_x \times N_x}) & \quad g_1(\A) = \lambda_A \|\A\|_1,\\
(\forall \P \in \mathcal{S}_{N_x}) & \quad g_2(\P) = - \frac{K}{2} \log \det(2 \pi \P) + \lambda_P \| \P\|_1,
\end{align}
so that
\begin{equation}
(\forall \A \in \mathbb{R}^{N_x \times N_x})(\forall \P \in \mathcal{S}_{N_x}) \quad \loss(\A,\P) = f(\A,\P) + g_1(\A) + g_2(\P),
\end{equation}
with $f$ lower semi-continuous function, $g_1 \in \Gamma_0(\mathbb{R}^{N_x \times N_x})$ and $g_2 \in \Gamma_0(\mathcal{S}_{N_x})$. Moreover, let us denote
 \begin{equation}
(\forall \A \in \mathbb{R}^{N_x \times N_x})(\forall \P \in \mathcal{S}_{N_x})  \quad
\mathcal{J}(\A,\P) =   \frac{K}{2} \text{tr}\left({\P}(  
\cblue{\widetilde{\Psib}} - \cblue{\widetilde{\Deltab}} \Ab^\top - \Ab \cblue{\widetilde{\Deltab}}^\top + \Ab \cblue{\widetilde{\Phib}} \Ab^\top) \right),
\label{eq:quadmajor}
\end{equation}
and, for every $\widetilde{\A} \in \mathbb{R}^{N_x \times N_x}$ and $\widetilde{\P} \in \mathcal{S}_{N_x}^{++}$,
\begin{equation}
(\forall \A \in \mathbb{R}^{N_x \times N_x})  \quad u_1(\A ; \widetilde{\A}, \widetilde{\P}) = \mathcal{J}(\A, \widetilde{\P}) + \frac{1}{2 \theta_A} \| \A - \widetilde{\A}\|^2_F + f(\widetilde{\A},\widetilde{\P}) - \mathcal{J}(\widetilde{\A},\widetilde{\P}),
\label{eq:u1}
\end{equation}
\begin{equation}
(\forall \P \in \mathcal{S}_{N_x}) \quad u_2(\P ; \widetilde{\A}, \widetilde{\P}) = \mathcal{J}(\widetilde{\A},\P) + \frac{1}{2 \theta_P} \| \P - \widetilde{\P}\|^2_F + f(\widetilde{\A},\widetilde{\P}) - \mathcal{J}(\widetilde{\A},\widetilde{\P}). \label{eq:u2}
\end{equation}
By Theorem~\ref{theo:maj}, the following majorization properties hold for every $\widetilde{\A} \in \mathbb{R}^{N_x \times N_x}$ and $\widetilde{\P} \in \mathcal{S}_{N_x}^{++}$,
\begin{align}
\label{eq:maju1}
(\forall \A \in \mathbb{R}^{N_x \times N_x}) \quad u_1(\A ; \widetilde{\A}, \widetilde{\P}) & \geq f(\A,\widetilde{\P}),\\
(\forall  \P \in \mathcal{S}_{N_x}) \quad u_2(\P; \widetilde{\A}, \widetilde{\P}) & \geq f(\widetilde{\A},\P),
\label{eq:maju2}
\end{align}
and we have the tangency condition, for every $\widetilde{\A} \in \mathbb{R}^{N_x \times N_x}$ and $\widetilde{\P} \in \mathcal{S}_{N_x}^{++}$,
\begin{align}
u_1(\widetilde{\A} ; \widetilde{\A}, \widetilde{\P}) & = f(\widetilde{\A},\widetilde{\P}), \label{eq:tangentu1}\\
u_2(\widetilde{\P} ; \widetilde{\A}, \widetilde{\P}) & = f(\widetilde{\A},\widetilde{\P}).
\label{eq:tangentu2}
\end{align}
Then, straightforward computations allow to rewrite the iterates of Algorithm \ref{algorithm} as follows:
\begin{equation}
(\forall i \in \mathbb{N}) \quad
\begin{cases}
\A^{(i+1)} = \argmin{\A \in \mathbb{R}^{N_x \times N_x}}{u_1(\A ; \A^{(i)}, \P^{(i)}) + g_1(\A)},\\
\P^{(i+1)} = \argmin{\P \in \mathcal{S}_{N_x}}{u_2(\P ; \A^{(i+1)}, \P^{(i)}) + g_2(\P)},\\
\end{cases}
\label{eq:titan}
\end{equation}
which identifies with the iterative scheme TITAN from \cite{phan2023inertial}, in the case of two blocks and setting the extrapolation step to zero. The rest of the proof amounts to check the fulfillment of the assumptions required for \cite[Theorem 6]{phan2023inertial} and \cite[Theorem 8]{phan2023inertial}.

$\bullet$ Let us denote, for every $\widetilde{\A} \in \mathbb{R}^{N_x \times N_x}$ and $\widetilde{\P} \in \mathcal{S}_{N_x}^{++}$,
\begin{align}
(\forall \A \in \mathbb{R}^{N_x \times N_x}) \quad & \tilde{u}_1(\A ; \widetilde{\A}, \widetilde{\P})  = u_1(\A ; \widetilde{\A}, \widetilde{\P}) + g_1(\A),\\
(\forall  \P \in \mathcal{S}_{N_x}) \quad & \tilde{u}_2(\P; \widetilde{\A}, \widetilde{\P})
  = u_2(\P; \widetilde{\A}, \widetilde{\P}) + g_2(\P).
\end{align}
Functions $u_1$ and $u_2$ are quadratic and strongly convex with respective strong convexity constants $\theta_A^{-1}$ and $\theta_P^{-1}$. Since both $g_1$ and $g_2$ are convex, functions $\tilde{u}_1$ and $\tilde{u}_2$ are also strongly convex, with respective strong convexity constants $\theta_A^{-1}$ and $\theta_P^{-1}$. Let $i \in \mathbb{N}$. According to the optimality conditions of both equations in \eqref{eq:titan}, there exists $\mathbf{T}_1^{(i+1)} \in \partial \tilde{u}_1 (\A^{(i+1)}; \A^{(i)}, \P^{(i)}) \in \mathbb{R}^{N_x \times N_x}$ and $\mathbf{T}_2^{(i+1)} \in \partial \tilde{u}_2 (\P^{(i+1)}; \A^{(i+1)}, \P^{(i)}) \in \mathcal{S}_{N_x}$ such that
\begin{equation}
\begin{cases}
\text{tr}\left( \mathbf{T}_1^{(i+1)} (\A^{(i)} - \A^{(i+1)}) \right) \geq 0,\\
\text{tr}\left( \mathbf{T}_2^{(i+1)} (\P^{(i)} - \P^{(i+1)}) \right) \geq 0.
\end{cases}
\label{eq:fermatpositive}
\end{equation}
Moreover, by strong convexity of both $\tilde{u}_1$ and $\tilde{u}_2$,
\begin{equation}
\begin{cases}
\tilde{u}_1(\A^{(i)}; \A^{(i)}, \P^{(i)}) \geq \tilde{u}_1(\A^{(i+1)}; \A^{(i)}, \P^{(i)}) + 
\text{tr}\left( \mathbf{T}_1^{(i+1)} (\A^{(i)} - \A^{(i+1)}) \right) \\
\phantom{\tilde{u}_1(\A^{(i)}; \A^{(i)}, \P^{(i)}) \geq \tilde{u}_1(\A^{(i+1)}; \A^{(i)}, \P^{(i)}) + 
} + \frac{1}{2 \theta_A} \|\A^{(i+1)} - \A^{(i+1)}\|^2_F,\\
\tilde{u}_2(\P^{(i)}; \A^{(i+1)}, \P^{(i)}) \geq \tilde{u}_2(\P^{(i+1)}; \A^{(i+1)}, \P^{(i)}) + 
\text{tr}\left( \mathbf{T}_2^{(i+1)} (\P^{(i)} - \P^{(i+1)}) \right) \\
\phantom{\tilde{u}_1(\A^{(i)}; \A^{(i)}, \P^{(i)}) \geq \tilde{u}_1(\A^{(i+1)}; \A^{(i)}, \P^{(i)}) + 
}+ \frac{1}{2 \theta_P} \|\P^{(i+1)} - \P^{(i+1)}\|^2_F.\\
\end{cases}
\end{equation}
Hence, using \eqref{eq:fermatpositive},
\begin{equation}
\begin{cases}
\tilde{u}_1(\A^{(i)}; \A^{(i)}, \P^{(i)}) \geq \tilde{u}_1(\A^{(i+1)}; \A^{(i)}, \P^{(i)}) + \frac{1}{2 \theta_A} \|\A^{(i+1)} - \A^{(i+1)}\|^2_F,\\
\tilde{u}_2(\P^{(i)}; \A^{(i+1)}, \P^{(i)}) \geq \tilde{u}_2(\P^{(i+1)}; \A^{(i+1)}, \P^{(i)}) + \frac{1}{2 \theta_P} \|\P^{(i+1)} - \P^{(i+1)}\|^2_F,
\end{cases}
\end{equation}
and, using \eqref{eq:maju1}-\eqref{eq:maju2}-\eqref{eq:tangentu1}-\eqref{eq:tangentu2},
\begin{equation}
\begin{cases}
\loss(\A^{(i)}, \P^{(i)}) \geq \loss(\A^{(i+1)}, \P^{(i)}) + \frac{1}{2 \theta_A} \|\A^{(i+1)} - \A^{(i+1)}\|^2_F,\\
\loss(\A^{(i+1)}, \P^{(i)}) \geq \loss(\A^{(i+1)}, \P^{(i+1)}) + \frac{1}{2 \theta_P} \|\P^{(i+1)} - \P^{(i+1)}\|^2_F.
\end{cases}
\label{eq:ndsp}
\end{equation}
It means that the so-called NDSP (nearly sufficiently decreasing property) condition from \citep{phan2023inertial} holds with, for every $i \in \mathbb{N}$, $(\gamma_1^{(i)},\gamma_2^{(i)},\eta_1^{(i)},\eta_2^{(i)}) \equiv (0,0,\theta_A^{-1},\theta_P^{-1})$. Remark that our proof for \eqref{eq:ndsp} constitutes an alternative proof for Lemma \ref{lemma:descent}.

$\bullet$ According to \citep{Gupta1974}, function $f$ is continuously differentiable. As $g_1 \in \Gamma_0(\mathbb{R}^{N_x \times N_x})$ and $g_2 \in \Gamma_0(\mathcal{S}_{N_x})$, we have for every $\A \in \mathbb{R}^{N_x \times N_x}$ and $\P \in \mathcal{S}_{N_x}$,
\begin{equation}
\begin{cases}
\partial_A \left(f(\A,\P) + g_1(\A) \right) = \nabla_A f(\A,\P) + \partial g_1(\A), \\
\partial_P \left(f(\A,\P) + g_2(\P) \right) = \nabla_P f(\A,\P) + \partial g_2(\P), \\
\end{cases}
\end{equation}
so that \citep[Assumption 3(i)]{phan2023inertial} holds. Moreover, by construction of the majorizing function $\mathcal{J}$, it is also continuously differentiable and we have the coincidence of the gradient at the majorization point, namely for every $\widetilde{\A} \in \mathbb{R}^{N_x \times N_x}$ and $\widetilde{\P} \in \mathcal{S}_{N_x}^{++}$,
\begin{align}
\nabla_A f(\widetilde{\A},\widetilde{\P}) = \nabla  \tilde{u}_1(\widetilde{\A};\widetilde{\A},\widetilde{\P}),\\
\nabla_P f(\widetilde{\A},\widetilde{\P}) = \nabla  \tilde{u}_2 (\widetilde{\P};\widetilde{\A},\widetilde{\P}).
\end{align}
Thus, according to \citep[Lem.2]{phan2023inertial}, \citep[Assumption 2]{phan2023inertial} is verified. Moreover, as the restriction of $\mathcal{J}$ to both of its variables is quadratic, the partial gradients $\nabla  u_1$ and $\nabla u_2$ are linear and thus Lipschitz continuous on any bounded subset of $\mathbb{R}^{N_x \times N_x}$ and $\mathcal{S}_{N_x}$, respectively, which yields the fulfillment of \citep[Assumption 3]{phan2023inertial}.  

$\bullet$ Since NSDP and \cite[Assumption 2]{phan2023inertial} hold, the result in \cite[Theorem 6]{phan2023inertial} allows to show that, if the sequence $\{\A^{(i)},\P^{(i)}\}_{i \in \mathbb{N}}$ is bounded, then every limit point $(\A^*,\P^*)$ of it is a critical point of $\loss$. Moreover, \cite[Assumption 3]{phan2023inertial} is satisfied and \citep[Cond. 1]{phan2023inertial} and \citep[Cond. 4]{phan2023inertial} are trivially met in our case. We can also show that the loss function $\loss$ satisfies the Kurdyka-\L{}ojasiewicz property \citep{BoltePALM} (using a similar proof than \citep[Lemma 2]{LauJMIV}). Thus, \cite[Theorem 8]{phan2023inertial} holds which concludes our proof.

\end{proof}

\subsection{Discussion}
\paragraph{Kurdyka-\L{}ojasiewicz inequality in non-convex optimization:} The proof of Theorem~\ref{theo:convergence} is grounded on the recent works by \cite{phan2023inertial,HienICML2020}, generalizing the works by \cite{BoltePALM,Chouzenoux2016jogo}, for establishing the convergence of block alternating MM schemes under the Kurdyka-\L{}ojasiewicz (K\L{}) inequality assumption~\citep{Lojasiewicz}. The latter assumption is a powerful tool of nonlinear functional analysis that has been popularized in the seminal paper by \cite{Attouch2010}. In their paper, the authors show how to prove the convergence of the iterates generated by a large family of minimization schemes, under the sole requirement that the function to minimize satisfies the K\L{} inequality. The latter requirement is very mild, as it holds for a large class of functions, non necessarily convex, as soon as they can be embedded within an o-minimal structure. Semi-algebraic and analytical functions are examples of such functions. In our analysis, for the sake of conciseness, we skipped the proof showing that $\loss$ satisfies K\L{} as it is identical to the one of \citep[Lemma 2]{LauJMIV}.

\cblue{\paragraph{Inner updates:} Theorem~\ref{theo:convergence} assumes that both inner steps \eqref{eq:updateA} and \eqref{eq:updateP} are solved in an exact manner. Due to strong convexity of $\mathcal{C}_1$ and $\mathcal{C}_2$, each problem has a unique solution, which ensures the well posedness of the formulation. However, the resolution of \eqref{eq:updateA} and \eqref{eq:updateP}  does not take a closed form (except when $\lambda_A$ or $\lambda_P$ equals zero, in such case we retrieve the MLEM updates from \cite[Chap. 12]{Sarkka}). Iterative inner solvers are thus necessary, and we proposed some efficient ones in Section \ref{sec:subproblems}. The extension of our convergence study to the case of inexact resolution of \eqref{eq:updateA} and \eqref{eq:updateP} is not straightforward, up to our knowledge. One could exit Algorithm~\eqref{algo:ProxA} (resp. Algorithm~~\eqref{algo:ProxP}) as soon as $\mathcal{C}_1$ (resp. $\mathcal{C}_2$) decreases, which might be satisfied after only a few iterations. In such case, Lemma~\ref{lemma:descent} holds and a weak convergence result can be deduced, namely the convergence of the loss function sequence, and the existence of cluster points for $\{\A^{(i)},\P^{(i)}\}_{i \in \mathbb{N}}$. Convergence of iterates arising from an inexact form of DGLASSO is difficult to guarantee, due to the intricated form of function $\loss$ (i.e., non-convex, non-Lipschitz smooth, function). Inexact proximal schemes for K\L{} losses, with advanced stopping conditions, have been studied in various works, such as \cite{Attouch2013,Cherni2020,ChouzenouxVMFB,Zheng2023,Bonettini2018,Bonettini2021}, to name a few. We are not aware of any study covering the block alternating MM scheme considered here, and thus left the convergence study of the inexact implementation of DGLASSO as future work. In practice, we impose a rather demanding condition on the stopping rule for the inner solvers of \eqref{eq:updateA} and \eqref{eq:updateP} (typically, $\xi = 10^{-3}$ with a maximum of $20000$ iterations), and did not observe any numerical instabilities of the proposed algorithm.}

\section{Experiments}
\label{eq:sectionexperiment}
We perform a thorough evaluation study on various controlled scenarios where the ground truth matrices denoted $(\A^*,\P^*)$ (as well as $\Q^* = (\P^*)^{-1}$) are predefined, and the time series $\{\y_k,\x_k\}_{k=1}^K$ \cblue{are built} directly from the LG-SSM model \eqref{eq_model_state}-\eqref{eq_model_obs} using such matrices. The goal is then to provide estimates $(\widehat{\A},\widehat{\P},\widehat{\Q})$ of $(\A^*,\P^*,\Q^*)$, given the observation of $\{\y_k\}_{k=1}^K$. Except otherwise stated, all the compared methods have moreover access to a perfect knowledge of $(\R_k,\H_k,\mub_0,\Sigmab_0)$, \cblue{on} top of accessing the time series $\{\y_k\}_{k=1}^K$. The hidden state $\{\x_k\}_{k=1}^K$ is, by definition, not assumed to be known, and is only used to compute quality metrics on test sets. We first work on synthetic data in Section \ref{sec:synthetic}, where the structure, the sparsity level, the conditioning, of the sought matrices $(\A^*,\P^*)$; \cblue{are} controlled. This allows us to evaluate our method on multiple cases, to discuss its parameter tuning, and to compare it to benchmarks in terms of inference quality and complexity. We then address in Section \ref{sec:climate}, \cblue{a set of problems of graph inference arising in weather variability analysis}, using four \cblue{synthetic} datasets built upon the Neurips CauseMe data challenge \citep{Runge2020}. This second set of experiments aims at evaluating the ability of DGLASSO to model and estimate a large class of graph structures (here, 200 different graphs per dataset), in comparison with other state-of-the-art graph detection methods. 

All codes are run on a Desktop Dell Latitude computer, with 11th Gen Intel(R) Core(TM) i7-1185G7 at 3.00GHz, equipped with 32Go Ram, using Matlab R2021a software. \cblue{In all the experiments, we set the precision parameters in DGLASSO algorithm to $(\xi,\varepsilon) = (10^{-3},10^{-3})$, with a maximum number of $50$ iterations for the outer loop, and $20000$ iterations for the inner solvers. Similar stopping conditions are used for the benchmarks of competitors for ensuring fair comparisons. Moreover, the DGLASSO stepsize parameters are set to $(\theta_A,\theta_P) = (1,1)$, as this simple choice led to good performance in our experiments.} The code is publicly available, for reproducibility purpose.\footnote{\url{https://pages.saclay.inria.fr/emilie.chouzenoux/Logiciel.html}} 


\subsection{\cblue{Controlled} data} 
\label{sec:synthetic}

\subsubsection{Datasets}

We set $K = 10^3$, $\R_k = \sigma_\R^2 {\Id}_{N_y}$ for every $k \in \{1,\ldots,K\}$, $\mub_0 \in \mathbb{R}^{N_x}$ is a vector of ones, $\Sigmab_0 = \sigma_0^2 {\Id}_{N_x}$ with $(\sigma_\R,\sigma_0) = (10^{-1},10^{-4})$. Matrix $\H_k$ is set to identity matrix for every $k \in \{1,\ldots,K\}$, so that $N_x = N_y$. This setting models a one-to-one correspondence between states and observations. Identifiability issues are hence avoided. 

We set $N_x = N_y = 9$, and we rely on block-diagonal matrices $(\A^*,\P^*)$, made of $3$ blocks with dimensions $3 \times 3$. \cblue{Matrices $\A^*$ are randomly set as
\begin{equation}
\A^* = \U \text{Diag}(\max(\boldsymbol{\lambda},0.99)) \U^\top,  \label{eq:projA}
\end{equation}
where $(\U,\boldsymbol{\lambda})$ are obtained by the singular value decomposition 
\begin{equation}
\U \text{Diag}(\boldsymbol{\lambda}) \U^\top = \text{BDiag}\{(B_j)_{1 \leq j \leq 3} \},
\end{equation}
with
\begin{equation}
(\forall (j,n,\ell) \in \{1,2,3\}^3) \quad B_j(n,\ell) = \rho_j^{|\sigma_j(n) - \ell|}. \label{eq:ARA}
\end{equation}
Equation \eqref{eq:ARA} builds matrices of auto-regressive processes of order one. $\text{BDiag}\{(B_j)_{1 \leq j \leq 3} \}$ denotes the block diagonal matrix formed with the $(B_j)_{1 \leq j \leq 3}$ matrices, $\rho_j$ is a scalar uniformly sampled into $[0,1]$, and $\sigma_j$ is a random bijective permutation from $\{1,2,3\}$ to $\{1,2,3\}$. The latter permutation allows to break the symmetry property of block elements of $\A^*$. The capping operation in Equation~\eqref{eq:projA} implements a spectral projection~\citep{Benfenati18}, so that the spectral norm of $\A^*$ is lower or equal than $0.99$, and, as such, ensures the stability of the resulting SSM.
}

The diagonal blocks of $\P^* = \text{BDiag}\{(\Omega_j)_{1 \leq j \leq 3}\}$ are randomly set following the procedure from \cite{More}\cblue{, described as
\begin{align}
(\forall j \in \{1,2,3\}) \quad 
\Omega_j = \left(\Id_3 - 2 \frac{\mathbf{p}_j\mathbf{p}_j^\top}{\|\mathbf{p}_j\|}\right)
\text{Diag}\left\{(c^{(n-1)/2})_{1 \leq n \leq 3} \right\}
\left(\Id_3 - 2 \frac{\mathbf{p}_j\mathbf{p}_j^\top}{\|\mathbf{p}_j\|}\right),
\end{align}
where each $\mathbf{p}_j \in \mathbb{R}^3$ has uniformly sampled entries in $[-1,1]$. Parameter $c>0$ controls the conditioning number of matrices $(\Omega_j)_{1 \leq j \leq 3}$, and thus the one of $\P^*$. We set $\log_{10}(c) \in \{0.1, 0.2, 0.5, 1\}$, leading to datasets A, B, C, and D, respectively.
}

DGLASSO provides estimates $(\widehat{\A},\widehat{\P})$ as a direct output. The estimate $\widehat{\Q}$ is simply deduced as $\widehat{\Q} = (\widehat{\P})^{-1}$. We initialize DGLASSO with $\P^{(0)} = 10^{-1} {\Id}_{N_x}$, and $\A^{(0)}$ \cblue{equal to} a stable auto-regressive order one matrix with entries $A^{(0)}(n,m) = (10^{-1})^{|n-m|}$ \cblue{projected onto the set of matrices with spectral norm \cblue{equal to} 0.99}. The setting of regularization parameters $(\lambda_A,\lambda_P)$ is discussed in the dedicated Section~\ref{sec:hyperparam}. Performance of the method for varying initializations is discussed in Appendix~\ref{appendix:init}.

\subsubsection{Comparisons to other methods}

In our experiments, we also compare DGLASSO with other model inference techniques for LG-SSM. 

First, we consider the EM method from \cite{shumway1982approach} (denoted after by MLEM) to compute $(\widehat{\A},\widehat{\Q})$ as maximum likelihood estimates (i.e., no regularization is employed) of matrices  $(\A^*,\Q^*)$, the estimation $\widehat{\P}$ being defined here as the inverse of $\widehat{\Q}$. \cblue{Note that no inner loop is required in the MLEM procedure, as solutions to the M-step have a closed form expression~\citep{shumway1982approach}.} 


Second, we consider three model inference techniques that explicitly incorporate a sparse graphical prior knowledge on the sought matrices. Namely, we compare to GLASSO \citep{Friedman07}, that considers a simple static and noiseless version of the LG-SSM (i.e., $\widehat{\A}$ is empirically set to zero and observation noise is neglected). Matrix $\widehat{\P}$ is then obtained through a maximum a posteriori formulation under an $\ell_1$ prior. The convex GLASSO loss is minimized using the proximal splitting method described in \cite[Algorithm 1]{Benfenati18}. Matrix $\widehat{\Q}$ is deduced by inversion of the resulting $\widehat{\P}$. We also compare with the robust GLASSO (rGLASSO) approach introduced in \cite{Benfenati18}, that explicitly accounts for the expression of $\R_k$ in the maximum a posteriori loss expression. For the sake of fair comparison, we use hereagain an $\ell_1$ penalty to obtain $\widehat{\P}$ although more sophisticated priors could be encompassed by rGLASSO. For both aforementioned methods, we rely on the \textsf{Matlab} toolbox provided by the authors.\footnote{http://www-syscom.univ-mlv.fr/~benfenat/Software.html} Finally, we provide the results obtained with the GRAPHEM method we recently introduced in \cite{elvira2022graphem}. GRAPHEM provides a maximum a posteriori estimate $\widehat{\A}$ under an $\ell_1$ prior, while $\widehat{\Q}$ is empirically set to $\sigma_\Q^2 {\Id}_{N_y}$, with a finetuned $\sigma_\Q > 0$. The \textsf{Matlab} toolbox provided by the authors\footnote{https://pages.saclay.inria.fr/emilie.chouzenoux/LogicielEN.html} is used to produce the results for this method. 

\cblue{Third, we compare DGLASSO with a quasi-Newton approach also aiming at minimizing \eqref{eq:phik}. The log-likelihood gradient is computed using the sensitivity equations from~\cite{Nagakura2021}, while subgradient updates are used to handle the Lasso penalties. The chain rule from \cite{Matrix} is used to build gradients with respect to $\P = \Q^{-1}$. The minimization is performed using a Quasi-Newton routine from the \textsc{Optimization Toolbox} from \textsc{Matlab}, with default options and a maximum of 100 iterations.}

All the comparative methods are programmed in the same language, they are initialized using a similar strategy as our DGLASSO method, and similar stopping criteria and hyper-parameter tuning approach are employed, for fair comparisons.                                             

\subsubsection{Evaluation metrics}

We first evaluate the results of our method, as well as the comparative benchmarks, through quality assessment metrics. Namely, we use the relative mean square error (RMSE) between the ground truth matrices $(\A^*,\P^*,\Q^*)$ and the estimated $(\widehat{\A},\widehat{\P},\widehat{\Q})$ (when available). For instance,
\begin{equation}
\text{RMSE}(\A^*,\widehat{\A}) = \frac{\|\A^* - \widehat{\A}\|^2_{\text{F}}}{\|\A^*\|^2_{\text{F}}}.
\label{eq:RMSE}
\end{equation}
$\text{RMSE}(\P^*,\widehat{\P})$ and $\text{RMSE}(\Q^*,\widehat{\Q})$ are defined in a similar fashion. We also compute area-under-curve (AUC) and F1 score comparing the non-zero entries (that is, the graph edges positions) of the sparse matrices $(\A^*,\P^*)$ and their estimates $(\widehat{\A},\widehat{\P})$. \cblue{The absolute value of the entry of each matrix is used as a detection threshold for the AUC score. A threshold value of $10^{-10}$ is used for the detection hypothesis for the F1 score (i.e., a weight greater than $10^{-10}$, in absolute value, is considered as an edge).} We furthermore evaluate the ability of the estimated model parameters to actually describe and infer the time series (both observed and hidden states). To that end, we build test time series $(\x^{\text{test}},\y^{\text{test}})$, not seen by the algorithms, constructed by running the ground truth LG-SSM (i.e., with ground truth matrix parameters $(\A^*,\P^*,\Q^*)$). We then run KF and RTS algorithms~\ref{alg_kf} and~\ref{alg_rts}, respectively, using either the ground truth matrices $(\A^*,\P^*,\Q^*)$ or their estimations $(\widehat{\A},\widehat{\P},\widehat{\Q})$, to build, for every $k \in \{1,\ldots,K\}$, the predictive distribution means $(\mub_k^*,\nub_k^*,\mub_k^{\text{s}*})$ and $(\widehat{\mub}_k,\widehat{\nub}_k,\widehat{\mub}_k^{\text{s}})$, respectively.


This allows in particular to compute the cumulative normalized mean squared error (cNMSE) between the predictive distribution means using either the ground truth model matrices or the estimated ones. Namely, we calculate
\begin{equation}
\text{cNMSE}(\nub^*,\widehat{\nub}) = \frac{\sum_{k=1}^K \| \nub_k^* - \widehat{\nub} \|^2}{\sum_{k=1}^K \| \nub_k^* \|^2},
\label{eq:cNMSE}
\end{equation}
as well as $\text{cNMSE}(\mub^*,\widehat{\mub})$, and $\text{cNMSE}(\mub^{\text{s}*},\widehat{\mub}^{\text{s}})$. 

Finally, we evaluate the negative logarithm of the marginal likelihood $\loss_{1:K}(\widehat{\A},\widehat{\P})$ as defined in \eqref{eq:lossimplicit}, on the test time series.

\subsubsection{\cblue{Influence of regularization parameters setting}}
\label{sec:hyperparam}

\cblue{We here assess the influence on the results of} the setting of the \acro~hyper-parameters $(\lambda_A,\lambda_P)$, accounting for the sparsity priors on $\A$ and $\P$, respectively. To that aim, we ran \acro~for 100 values of hyperparameters $(\lambda_A,\lambda_P)$, regularly spaced on a log-scale grid between $1$ and $10^2$, and repeated the experience on {50} randomly generated time series, for dataset A. We display in Figure~\ref{fig:hyperparam} the averaged values, over the random runs, of several quantitative metrics, as a function of hyperparameters $(\lambda_A,\lambda_P)$. We also report in the caption the averaged RMSE scores obtained when running DGLASSO with $(\lambda_A,\lambda_P) = (0,0)$ (i.e., MLEM result). As it can be observed, both F1/RMSE for the estimation of the transition matrix $\A$ are mostly governed by the value of $\lambda_A$, while the quality scores for the state noise precision matrix $\P$ are mostly influenced by $\lambda_P$. Note that the RMSE scores on $\Q$, not shown here, follow similar evolution than RMSE on $\P$. Just inspecting F1/RMSE appears not informative enough to set parameters $(\lambda_A,\lambda_P)$, as each metric and each parameter seems to push towards a different goal. The maps of cNMSE$(\nub^*,\widehat{\nub})$ and of the marginal likelihood log-loss $\loss_{1:K}$ show very similar behavior. Note that the later is a practically interesting metric, because it does not require the knowledge of the ground truth matrices. On this example, it however appears not discriminative enough. Typically, it stays almost constant for a (too) wide value range of $(\lambda_A,\lambda_P)$, which confirms again the ill-posedness of the minimization problem. The maps for cNMSE$(\mub^*,\widehat{\mub})$ and cNMSE$(\mub^{\text{s}*},\widehat{\mub}^{\text{s}})$ are very similar. Interestingly, the minimization of both these quantities, related to the state mean distributions, seems to provide a meaningful value range for the regularization parameters, narrowing down around values that achieve an optimal compromise between (i) good estimation of the sought matrices, (ii) good estimation of the sparse matrices support, and (iii) good predictive behavior for time series inference by KF/RTS techniques. \cblue{Similar} conclusions were reached for the other three datasets. Note additionally that the DGLASSO results outperform for a wide range of $(\lambda_A,\lambda_P)$ those obtained with MLEM, confirming the advantage of introducing the proposed sparsity prior. This will be more deeply examined in an upcoming section. 

\cblue{In our forthcoming experiments, we opt for simple procedure to set the hyperparameters, depending on the task of interest. Namely, in this Section \ref{sec:synthetic}, we aim at an optimal compromise between both graph estimation, that reaches the best predictive power of the learned models, when evaluated on an unseen time series. This leads us to fix $(\lambda_A,\lambda_P)$ through \cblue{a rough grid search among $\{1,5,8, 10\}^2$} to minimize cNMSE$(\mub^*,\widehat{\mub})$ averaged on few (typically $5$) runs.}

\begin{figure}
\centering
\begin{tabular}{@{}c@{}c@{}}
\includegraphics[width = 6cm]{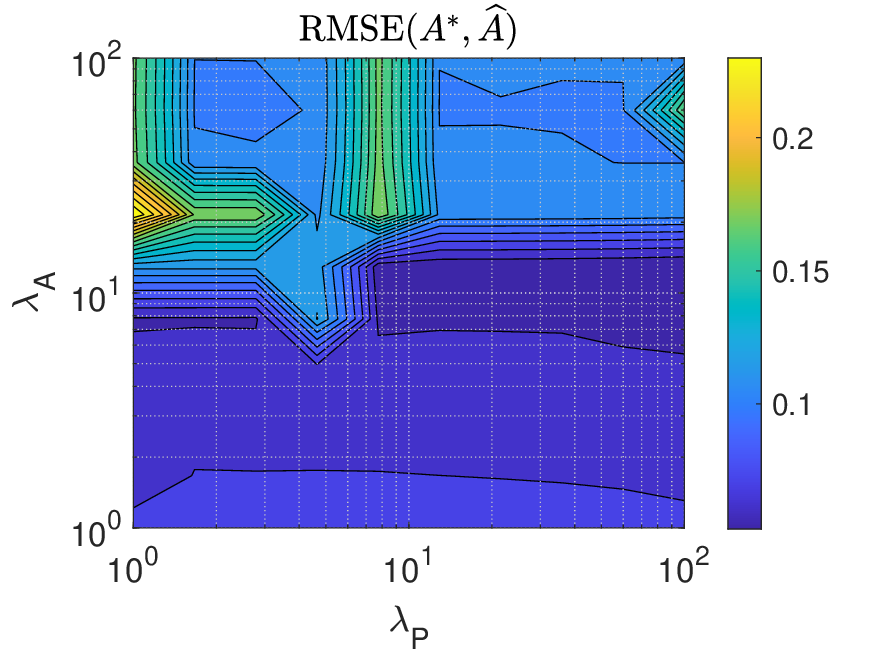} & \includegraphics[width = 6cm]{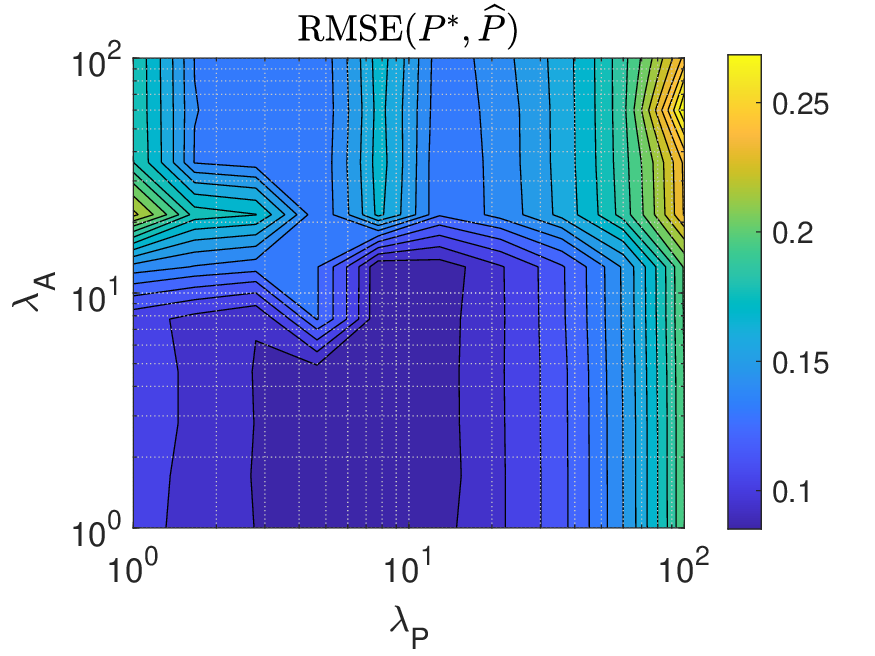}  \\
\includegraphics[width = 6cm]{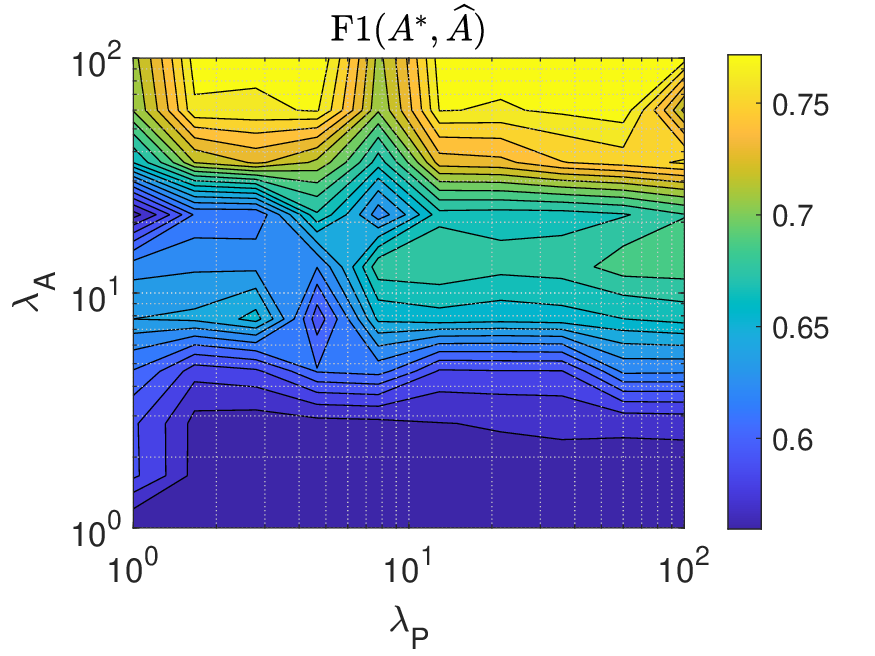} & \includegraphics[width = 6cm]{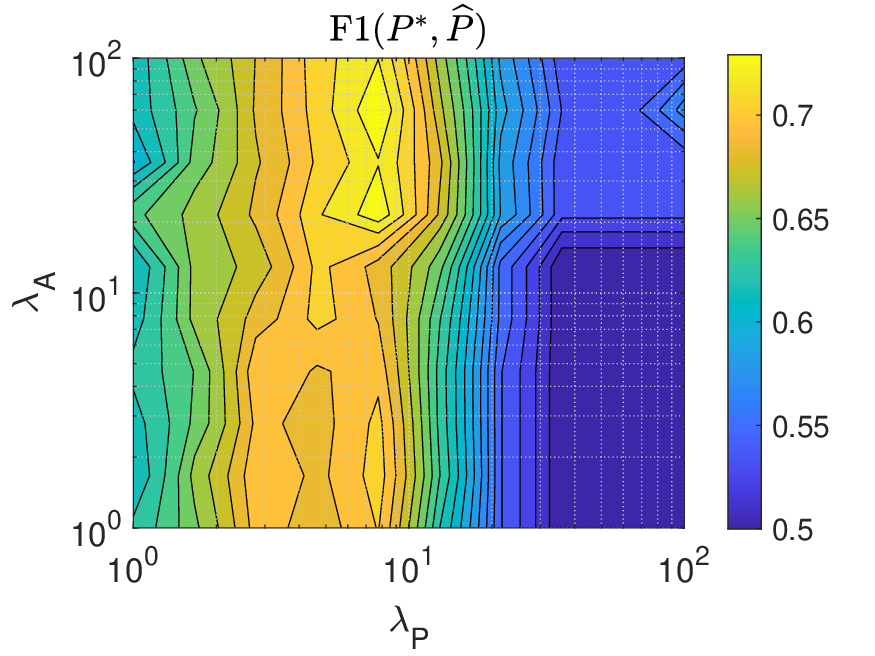}\\
\includegraphics[width = 6cm]{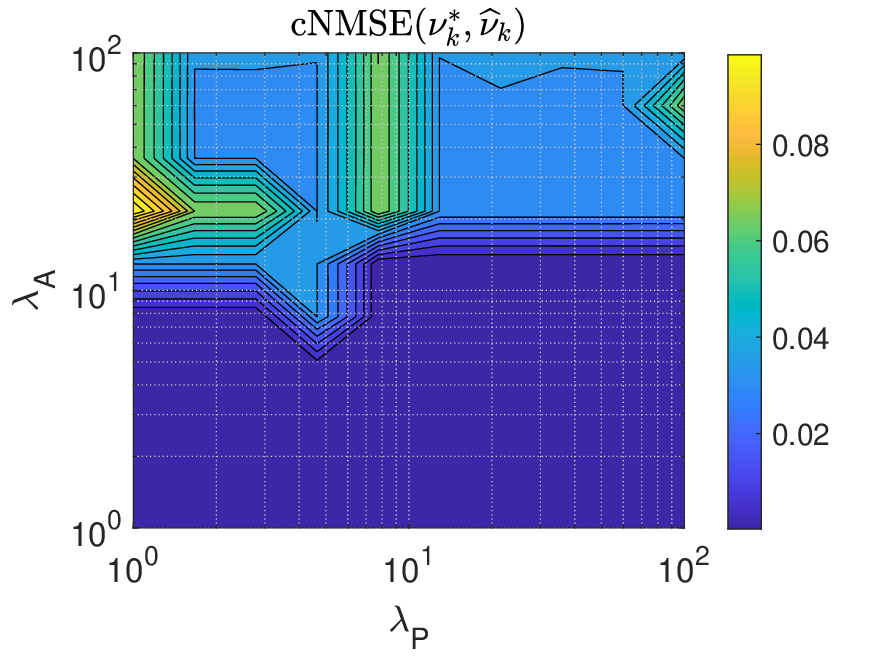} & \includegraphics[width = 6cm]{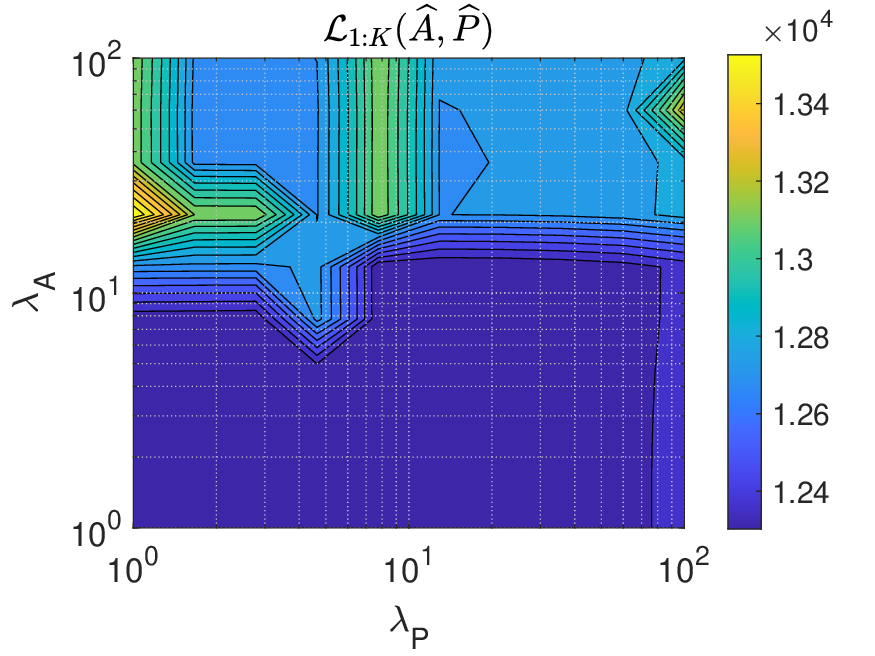}\\
\includegraphics[width = 6cm]{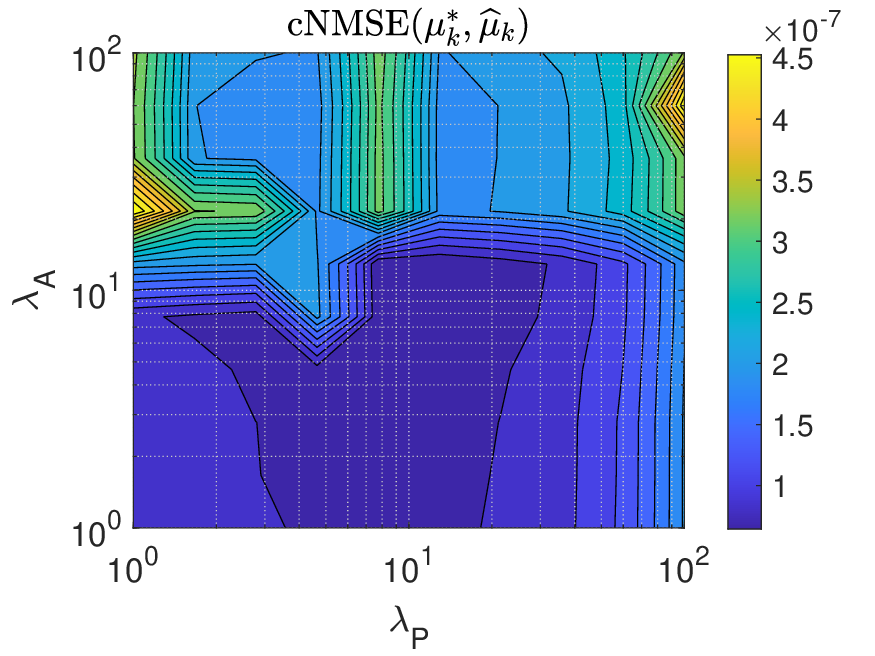} & \includegraphics[width = 6cm]{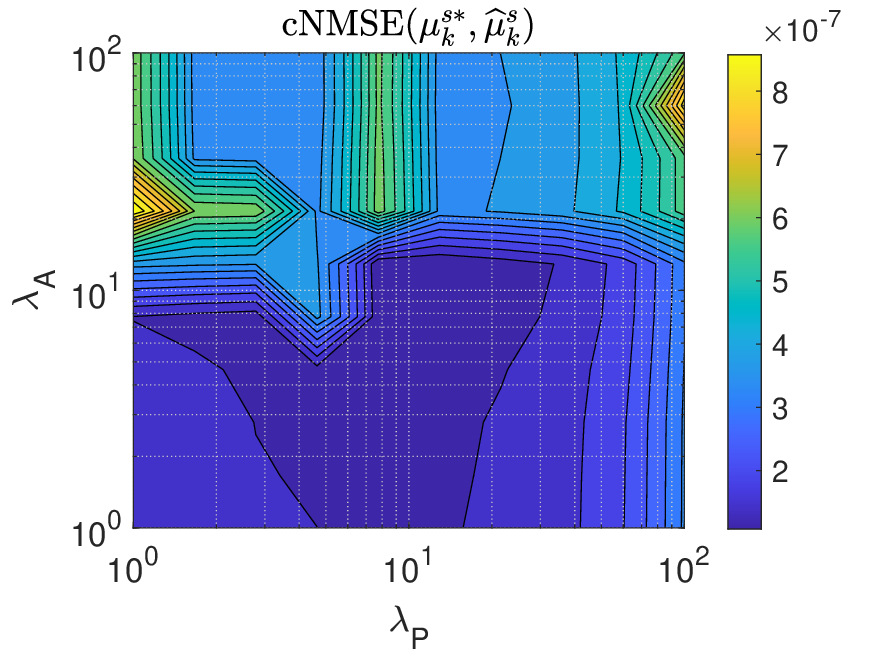}  \\
\end{tabular}
\caption{Evolution of RMSE, F1, cNMSE and loss scores on the estimation of $\A$ (left) and $\P$ (right) by DGLASSO, as a function of hyperparameters $(\lambda_A,\lambda_P)$, for dataset A (averaged on 10 runs). As a comparison, the averaged RMSE scores for $(\lambda_A,\lambda_P) = (0,0)$ (i.e., MLEM) on this example were $(0.077,0.106)$ for $(\A,\P)$, respectively.}
\label{fig:hyperparam}
\end{figure}
 
%

\subsubsection{Influence of conditioning number}

We report in Table \ref{tab:results} the results, averaged over 50 realizations of the time series generation. We do not provide the scores for the estimation of $\A$ (resp. $(\P,\Q)$) for GLASSO/rGLASSO (resp. GRAPHEM), as those methods are not designed for such task. The general trend in the results is a slight decrease of the quality of estimation for all methods, when $c_Q$ increases (i.e., from dataset A to D). This is expected, as an ill-conditioned matrix $\Q$ complicates the mathematical structure of the likelihood term. 

Regarding the estimation of $\A$, GRAPHEM method presents the best results for the three considered metrics. \acro~is second best, while MLEM follows. As already stated, GLASSO/rGLASSO do not estimate $\A$ (i.e., they assume this matrix to be zero, that is the process is i.i.d. without any Markov structure). Regarding the estimation of $(\P,\Q)$, DGLASSO is outperforming the benchmarks in terms of RMSE score. The second best, in terms of RMSE is MLEM. GLASSO and rGLASSO got very bad RMSE scores, probably because of the model mismatch induced by assuming $\A$ to be zero. The edge detection performance of \acro~are excellent, except in few cases where rGLASSO gets slightly better results. MLEM gets poorer results than \acro~in all metrics. In particular it exhibits a bad F1 score, as it does not include sparse prior, and thus does not succeed to eliminate any spurious edges. 

Regarding the distribution mean calculation, it is remarkable that DGLASSO outperforms in all examples the benchmarks, which shows that the proposed formulation allows to provide model parameters that are best suited for inference using KF/RTS at test phase. The marginal likelihood log-loss \cblue{computed on test set} is also minimized with the quantities provided by our method. This could appear as counter-intuitive, as the method is not specifically designed to minimize this loss \cblue{(while MLEM explicitly aims at minimizing this loss on the train set)}. The advantage of \acro~is that it accounts for prior knowledge, making the inferred model more robust, and less prone to overfitting, which is translated into better behavior on an independent test time series. 

We display in Figure~\ref{fig:runs} box plots assessing the variability to the RMSE and F1 scores, for both MLEM and DGLASSO methods, on 50 runs on dataset A and dataset D. The RMSE values are in most runs lower (i.e., better) for the proposed method. Both methods are quite stable with respect to the time series generation, as the plots show few outliers. For dataset D, corresponding to a more challenging case with an ill-conditioned matrix $\P^*$, the results are slightly more spread, especially for the metrics related to the recovery quality of this matrix. Remark that the F1 scores of MLEM are constant, \cblue{and are }equal to 0.5. As already \cblue{pointed} out, this is expected as MLEM is not designed to perform an edge detection task. 

We refer the reader to Appendix~\ref{appendix:sparse} for extra experiments on the synthetic datasets, assessing the performance of \acro~when the sparsity level of $\A^*$ increases.  
                                  
\begin{figure}
\centering
\begin{tabular}{@{}c@{}c@{}c@{}}
\includegraphics[width = 5.5cm]{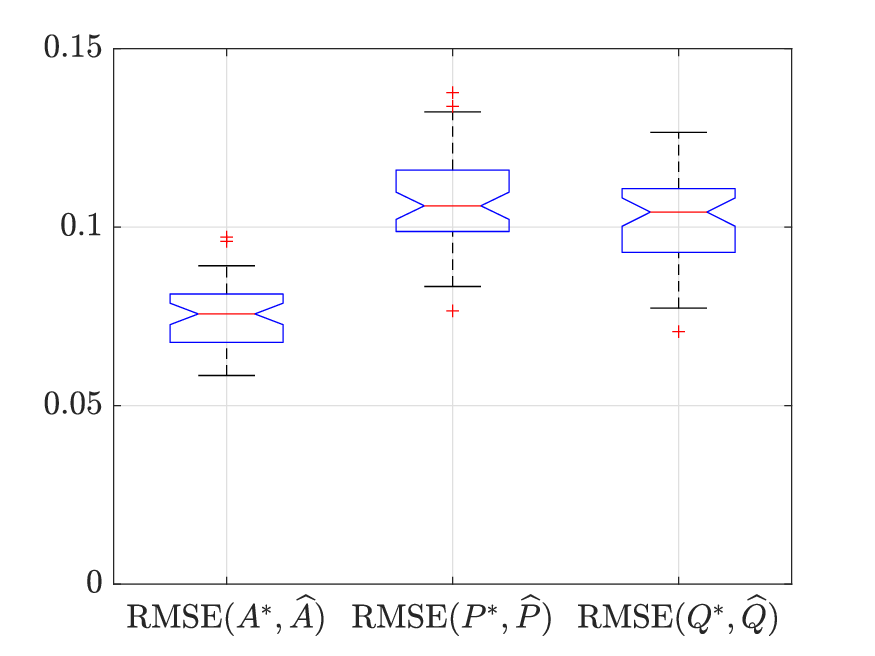} & \includegraphics[width = 5.5cm]{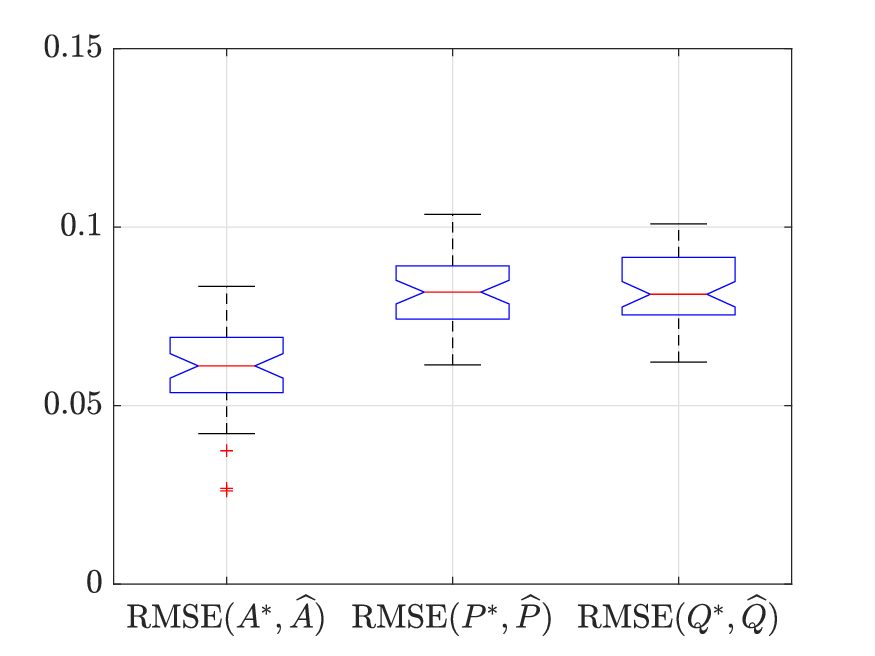}
& \includegraphics[width = 5.5cm]{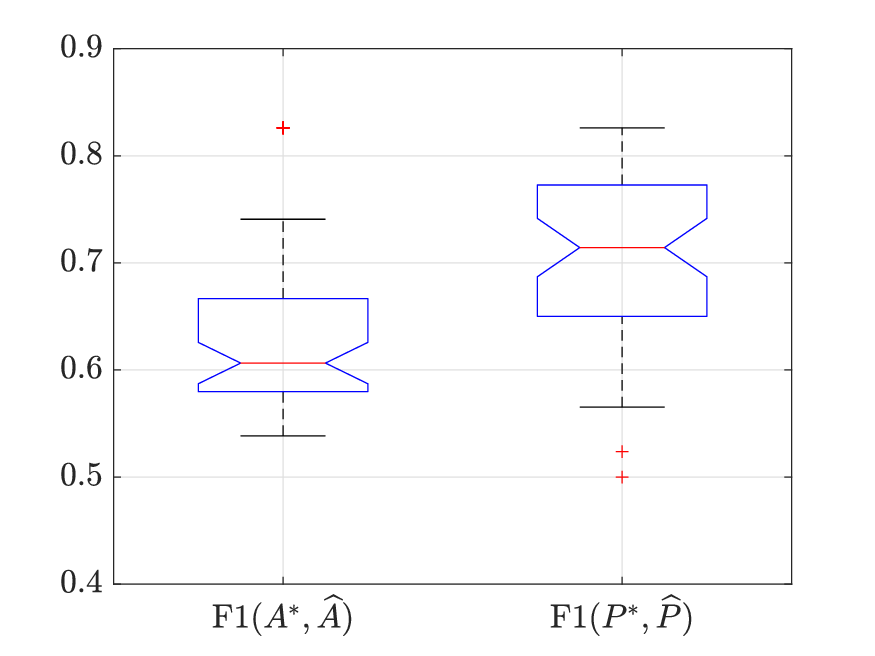}\\
 \includegraphics[width = 5.5cm]{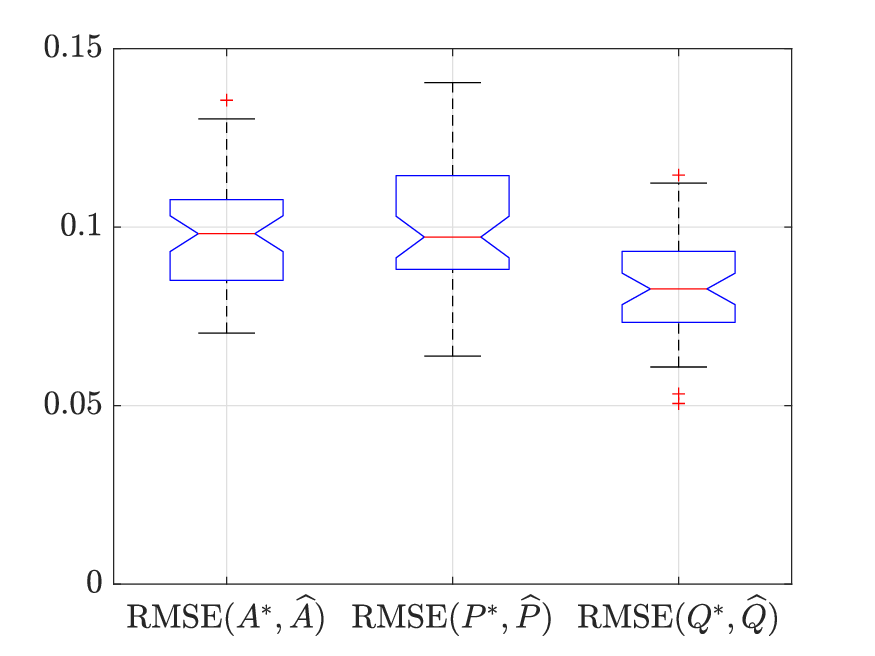} & \includegraphics[width = 5.5cm]{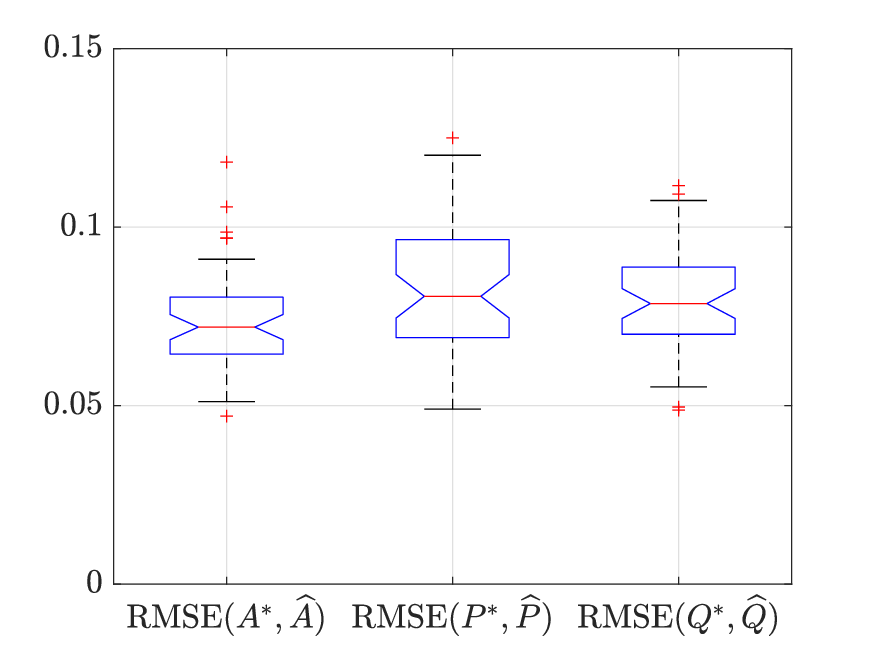} & \includegraphics[width = 5.5cm]{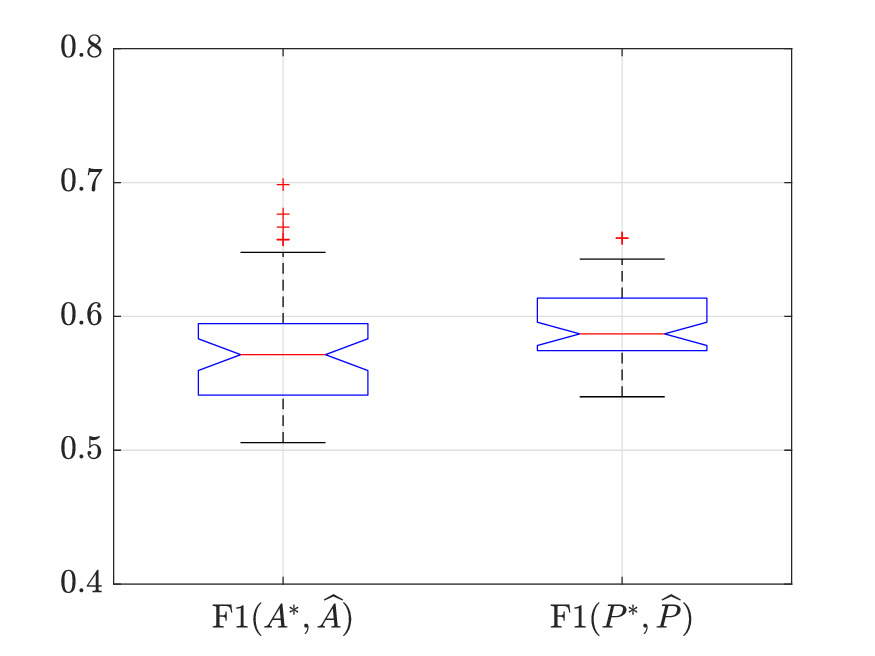}\\
MLEM (RMSE) & DGLASSO (RMSE) & DGLASSO (F1)
\end{tabular}
\caption{\cblue{Box plots for quantitative metrics when running MLEM, and \acro, on 50 randomly generated LG-SSM time series, for dataset A (top) and dataset D (bottom). F1 score is not reported for MLEM, as this method does not perform edge detection, resulting in a constant F1 score around $0.5$. \acro~outperforms MLEM with better (i.e., lower) RMSE scores for most runs and good F1 scores. Dataset D is more challenging in terms of inference, thus yielding to more spread results for both methods.}}
\label{fig:runs}
\end{figure}


\begin{table}[t]
\caption{Results for the four considered datasets A to D, with an increasing conditioning number of $\P^*$ \cblue{equal to} $\log_{10}(c) \in \{0.1, 0.2, 0.1, 1\}$, respectively. We evaluate the methods in terms of estimation quality for $(\A,\P,\Q)$, using either RMSE as defined in \eqref{eq:RMSE}, and edge detection scores (AUC, F1), as well as \cblue{in terms of inference quality on test set using cNMSE and marginal likelihood metrics defined in \eqref{eq:cNMSE}.}}
\hspace*{-2cm}
{\tiny
{
\begin{tabular}{|c|c|c|c|c||c|c|c||c||c|c||c|c|}
\cline{3-13}
\multicolumn{2}{c|}{ } & \multicolumn{3}{|c||}{Estimation of $\A$} & \multicolumn{3}{|c||}{Estimation of $\P$} & Estim. $\Q$ & \multicolumn{2}{|c||}{State distrib.} & \multicolumn{2}{|c|}{Predictive distrib.}\\
 \cline{2-13}
\multicolumn{1}{c|}{ } & Method & RMSE & AUC & F1 & RMSE & AUC & F1 & RMSE & cNMSE$(\mub^*,\widehat{\mub})$ & cNMSE$(\mub^{\text{s}*},\widehat{\mub}^\text{s})$ & cNMSE$(\nub^*,\widehat{\nub})$ & $\mathcal{L}_{1:K}(\widehat{\A},\widehat{\P})$\\
\hline

\parbox[t]{2mm}{\multirow{5}{*}{\rotatebox[origin=c]{90}{Dataset A}}} & DGLASSO & $\num{0.060525}$ & $\num{0.84255}$ & $\ubold \num{0.64057}$ & \ubold $\num{0.081927}$ & $\num{0.77801}$ & $\num{0.69812}$& $\ubold \num{0.082601}$ & $ \ubold \num{6.3935e-08}$ & $\ubold \num{1.0504e-07}$ & $\ubold \num{2.9837e-4}$ & $\ubold \num{12307.1687}$ \\
  & MLEM & $\num{0.076063}$ & $\num{0.81693}$ & $\num{0.5}$ & $\num{0.10528}$ & $\num{0.85739}$ & $\num{0.5}$& $\num{0.10152}$ & $\num{1.0947e-07}$ & $\num{1.8025e-07}$ & $\num{4.8425e-4}$ & $\num{12341.2045}$ \\
	& GLASSO & NA & NA & NA & $\num{0.81838}$ & $\num{0.80394}$ & $\num{0.49616}$ & $\num{1073.5104}$ & $\num{4.4853e-06}$ & $\num{7.1804e-06}$& $\num{1}$ & $\num{28459.2935}$ \\
	& rGLASSO & NA & NA & NA & $\num{0.76376}$ & $\ubold \num{0.92377}$ & $\num{0.59759}$ & $\num{31.6891}$ & $\num{2.8256e-06}$ & $\num{5.4924e-06}$ & $\num{1}$ & $\num{22957.6928}$ \\
	& GRAPHEM & $\ubold \num{0.044631}$ & $\ubold \num{0.89481}$ & $\ubold \num{0.84654}$ & NA & NA & NA & NA & $\num{4.3642e-06}$ & $\num{6.9443e-06}$ & $\num{2.98e-4}$ & $\num{29035.0296}$ \\
\hline
\hline
\parbox[t]{2mm}{\multirow{5}{*}{\rotatebox[origin=c]{90}{Dataset B}}} & DGLASSO & $\num{0.068331}$ & $\num{0.83317}$ & $\num{0.60322}$ & $\ubold \num{0.070299}$ & $\num{0.8934}$ & $\ubold \num{0.83536}$& $\ubold \num{0.07078}$ & $\ubold \num{7.4899e-08}$ & $\ubold \num{1.2361e-07}$ & $\ubold \num{3.2808e-4}$ & $\ubold \num{11806.7439}$\\
& MLEM & $\num{0.079544}$ & $\num{0.81495}$ & $\num{0.5}$ & $\num{0.10586}$ & $\num{0.898}$ & $\num{0.5}$ & $\num{0.10019}$ & $\num{1.2987e-07}$ & $\num{2.1332e-07}$ & $\num{4.6185e-4}$ & $\num{11833.4478}$\\
& GLASSO & NA & NA & NA & $\num{0.82748}$ & $\num{0.82565}$ & $\num{0.50531}$ & $\num{341.8726}$ & $\num{5.0689e-06}$ & $\num{8.0718e-06}$ & $\num{1}$ & $\num{27744.9642}$ \\
& rGLASSO & NA & NA & NA & $\num{0.73385}$ & $\ubold \num{0.93029}$ & $\num{0.60794}$ & $\num{33.8955}$ & $\num{3.2145e-06}$ & $\num{6.1866e-06}$ & $\num{1}$ & $\num{22530.0363}$ \\
& GRAPHEM & $\ubold \num{0.046795}$ & $\ubold \num{0.89339}$ & $\ubold \num{0.84779}$ & NA & NA & NA & NA & $\num{5.1584e-06}$ & $\num{8.0362e-06}$ & $\num{2.9122e-4}$ & $\num{29031.4116}$ \\ 
\hline
\hline
\parbox[t]{2mm}{\multirow{5}{*}{\rotatebox[origin=c]{90}{Dataset C}}} & DGLASSO & $ \num{0.070227}$ & $\num{0.82903}$ & $\num{0.58126}$ & $\ubold \num{0.090361}$ & $\num{0.95372}$ & $\ubold \num{0.82954}$& $\ubold \num{0.077905}$ & $\ubold \num{1.8962e-07}$ & $\ubold \num{2.9941e-07}$ & $\ubold \num{3.9556e-4}$ & $\ubold \num{10311.104}$\\
& MLEM & $\num{0.081204}$ & $\num{0.81025}$ & $\num{0.5}$ & $\num{0.097357}$ & $\ubold \num{0.97361}$ & $\num{0.5}$ & $\num{0.094366}$ & $\num{2.5827e-07}$ & $\num{4.1796e-07}$ & $\num{5.0526e-4}$ & $\num{10326.4099}$\\
& GLASSO & NA & NA & NA & $\num{0.90132}$ & $\num{0.80498}$ & $\num{0.48925}$ & $\num{3.926e+17}$ & $\num{0.012027}$ & $\num{0.012032}$ & $\num{1}$ & $\num{26634.8917}$\\
& rGLASSO & NA & NA & NA & $\num{0.80517}$ & $\num{0.92781}$ & $ \num{0.61421}$ & $\num{29.5302}$ & $\num{7.1948e-06}$ & $\num{1.3195e-05}$ & $\num{1}$ & $\num{21322.2465}$\\
& GRAPHEM & $\ubold \num{0.049276}$ & $\ubold \num{0.89185}$ & $\ubold \num{0.857}$ & NA & NA & NA & NA & $\num{1.0548e-05}$ & $\num{1.6405e-05}$ & $\num{3.9121e-4}$ & $\num{29023.369}$\\ 
\hline
\hline
\parbox[t]{2mm}{\multirow{5}{*}{\rotatebox[origin=c]{90}{Dataset D}}} & DGLASSO & $\num{0.073461}$ & $\num{0.83514}$ & $\num{0.57452}$ & $\ubold \num{0.083377}$ & $\ubold \num{1}$ & $\num{0.59828}$& $\ubold \num{0.080425}$ & $ \ubold \num{5.1272e-07}$ & $ \ubold \num{8.2434e-07}$ & $\ubold \num{3.3729e-4}$ & $\ubold \num{7911.9431}$ \\
& MLEM & $\num{0.09763}$ & $\num{0.80819}$ & $\num{0.5}$ & $\num{0.094563}$  & $\ubold \num{1}$ & $\num{0.5}$ & $\num{0.083501}$ & $\num{6.2956e-07}$ & $\num{1.0266e-06}$ & $\num{4.2188e-4}$ & $\num{7923.8499}$\\
& GLASSO & NA & NA & NA & $\num{0.96414}$ & $\num{0.94064}$ & $\num{0.55036}$ & $\num{187.8227}$ & $\num{2.3478e-05}$ & $\num{3.7012e-05}$ & $\num{1}$ & $\num{23684.1778}$\\
& rGLASSO & NA & NA & NA & $\num{0.88182}$ & $\num{0.95609}$ & $\ubold \num{0.64509}$ & $\num{28.7027}$ & $\num{1.8864e-05}$ & $\num{3.2391e-05}$ & $\num{1}$ & $\num{20100.491}$\\
& GRAPHEM & $\ubold \num{0.060615}$ & $\ubold  \num{0.8917}$ & $\ubold  \num{0.86438}$ & NA & NA & NA & NA & $\num{2.503e-05}$ & $\num{3.8388e-05}$ & $\num{3.7431e-4}$ & $\num{29016.3209}$\\
\hline
\end{tabular}
}
}
\label{tab:results}
\end{table}

\subsubsection{Complexity analysis}
\label{sec:complex}
We now examine the time complexity of the method, as well as of MLEM, when processing dataset A, for various values of the time series length $K$, namely $K \in \{100, 200, 500, 1000, 2000, 5000\}$ (we recall that our previous experiments were all done with $K=1000$). We display in Figure~\ref{fig:complexity}(left) the computing time in seconds for computing $(\widehat{\A},\widehat{\P})$, for DGLASSO and MLEM, averaged over 50 realizations. We also display, in Figure \ref{fig:complexity}(middle) for the same experiment, the RMSE between $\A^*$ and $\widehat{\A}$, and in Figure \ref{fig:complexity}(right) the metric $\text{cNMSE}(\mub^*,\widehat{\mub})$. One can notice that our method is slightly more demanding in terms of computational time, \cblue{requiring about twice the time compared to MLEM in order to reach convergence. But both DGLASSO and MLEM scale similarly.} Moreover, as already observed in our previous experiments, DGLASSO outperforms MLEM on both metrics shown here, in all tested values of $K$. As expected, the results are better for higher values of $K$, at the price of an increased computational time. Interestingly, the regularization still yields improved results for very large $K$.

\begin{figure}
\centering
\begin{tabular}{@{}c@{}c@{}c}
\includegraphics[width = 5cm]{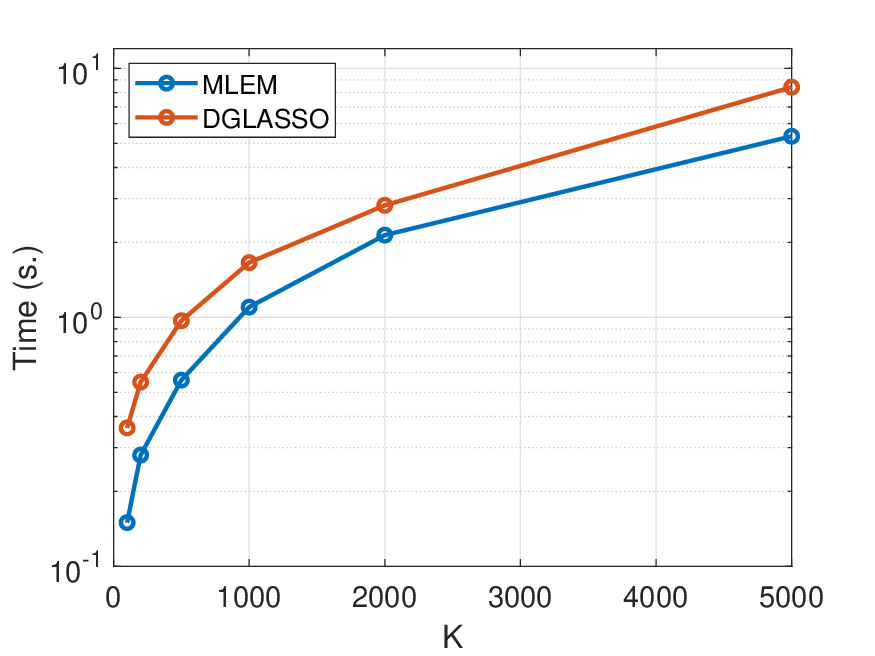} & \includegraphics[width = 5cm]{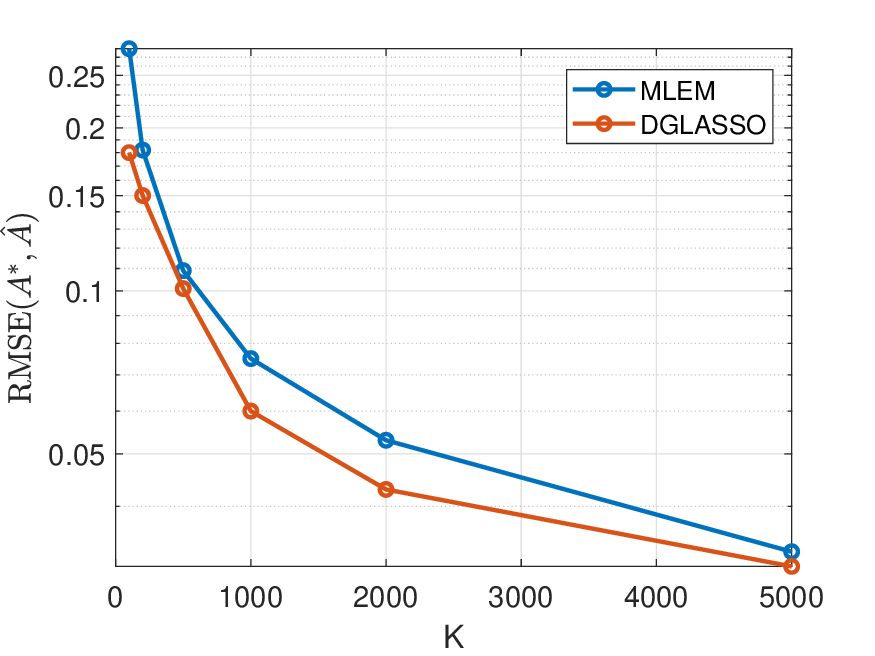} & \includegraphics[width = 5cm]{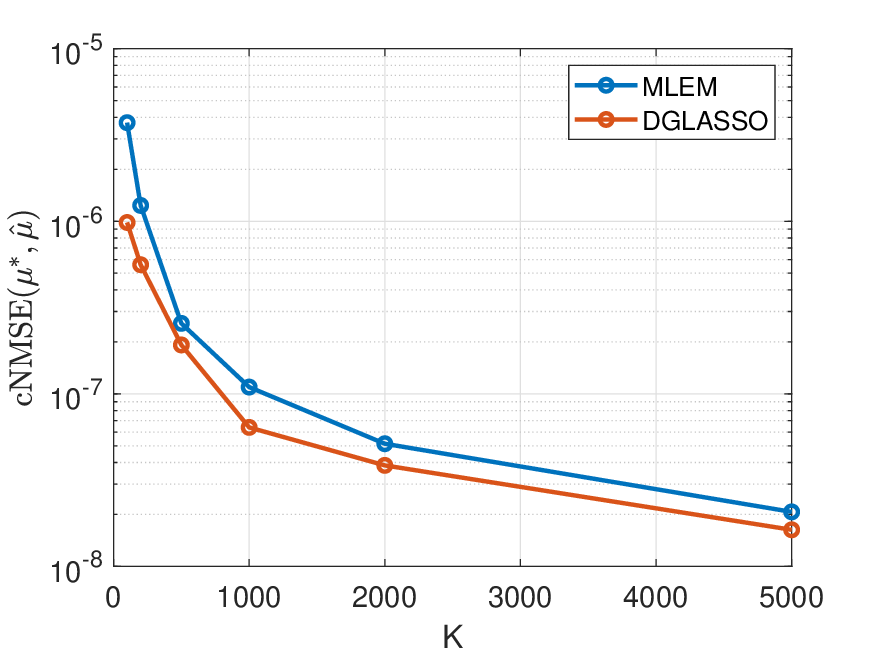}
\end{tabular}
\caption{Evolution of the complexity time (left), RMSE$(\A^*,\widehat{\A})$ (middle) and $\text{cNMSE}(\mub^*,\widehat{\mub})$ (right) metrics, as a function of the time series length $K$, for experiments on dataset A averaged over 50 runs. \cblue{Both methods scale similarly, with DGLASSO requiring about twice the time to run than MLEM.}}
\label{fig:complexity}
\end{figure}

\cblue{We end the complexity study by showing a comparison between DGLASSO and a Quasi-Newton (Q-N) implementation to solve the same minimization problem. Figure \ref{fig:complexityQN} displays, for both methods, the evolution of the loss function along time, for examples from Datasets A and D. The regularization hyperparameters are set the same for both algorithms to ensure a fair comparison. One can notice that both algorithms reach a similar value of the loss at convergence. DGLASSO requires very few iterations (typically, less than 10) and short time to stabilize, while Q-N displays a slower convergence profile despite its quasi-Newton (i.e., second-order) acceleration strategy. In terms of qualitative results, both algorithms reach a similar RMSE for the estimation of matrix $\A$ on Dataset A, while DGLASSO outperforms Q-N on Dataset D. In both scenarios, the estimation of $\P$ is considerably more accurate with the proposed DGLASSO. We explain these differences in terms of convergence speed and inference quality, by the gradient singularity of the loss with respect to variable $\P$, leading to numerical instabilities of the Q-N approach. This is especially the case for Dataset D, characterized by a poorly conditioned $\P^*$.}

\begin{figure}
\centering
\begin{tabular}{r@{}l@{}l@{}}
& \includegraphics[width = 7cm]{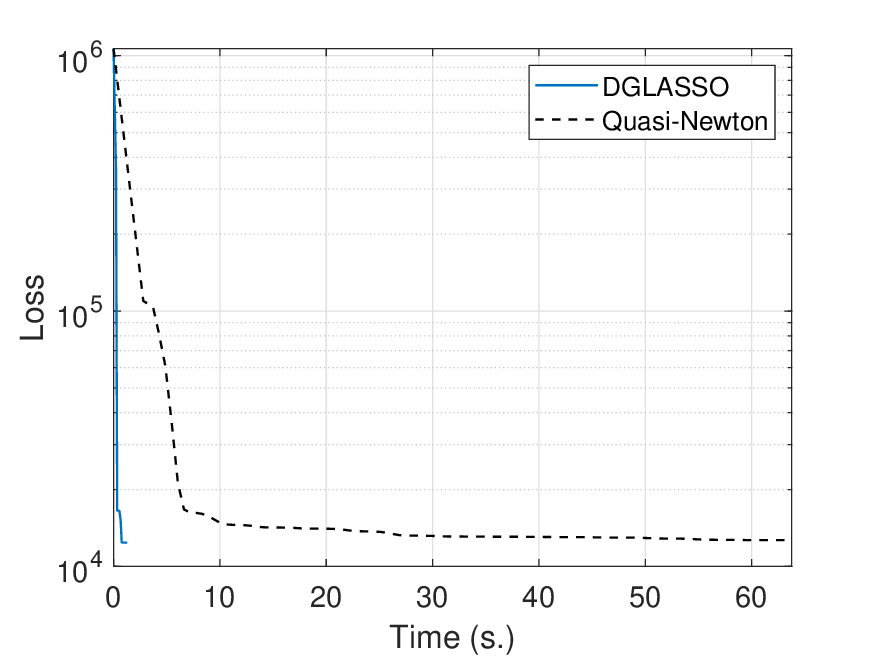} & \includegraphics[width = 7cm]{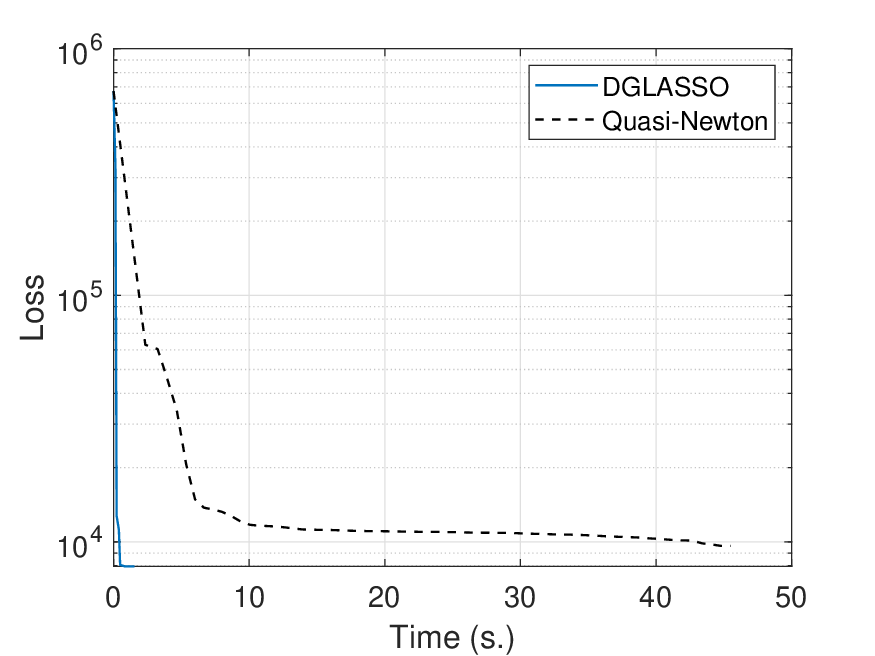}\\
\footnotesize DGLASSO:~&  \tiny $\qquad$ RMSE$(\A^*,\widehat{\A}) = 0.070$, RMSE$(\P^*,\widehat{\P}) = 0.081766$  &  \tiny $\qquad$ RMSE$(\A^*,\widehat{\A}) = 0.101$, RMSE$(\P^*,\widehat{\P}) = 0.082$  \\
 \footnotesize Q-N:~&  \tiny $\qquad$ RMSE$(\A^*,\widehat{\A}) = 0.077$, RMSE$(\P^*,\widehat{\P}) = 0.3015$  &  \tiny $\qquad$ RMSE$(\A^*,\widehat{\A}) = 0.359$, RMSE$(\P^*,\widehat{\P}) =  0.809$ \\
\end{tabular}
\normalsize
\caption{\cblue{Loss evolution versus computational time in seconds, and RMSE scores, for DGLASSO and Quasi-Newton algorithms, running on an example from Dataset A (left) and Dataset D (right). }}
\label{fig:complexityQN}
\end{figure} 

\subsection{Synthetic weather data}
\label{sec:climate}

\subsubsection{Experimental settings}

We now evaluate our method on \cblue{synthetic datasets arising from causal graph discovery studies} in the field of weather variability tracking. Specifically, we consider two sets of $200$ sparse matrices $\A^* \in \mathbb{R}^{N_x}$, with $N_x = 5$ or $10$ respectively, representing the ground truth causal graphs used to produce \texttt{WEATH} datasets in the Neurips 2019 data challenge~\citep{Runge2020}.\footnote{https://causeme.uv.es/static/datasets/TestWEATH/} For each $\A^*$, we create times series $(\x_k,\y_k)_{k=1}^K$ using \eqref{eq_model_state}-\eqref{eq_model_obs}, with $K=10^3$, $\H^* = \Id_{N_x}$ (i.e., $N_y = N_x$), and $(\sigma_{\R},\sigma_0) = (10^{-1},10^{-4})$. We set $\Q^*$ as a block diagonal matrix of $J$ blocks with dimensions ${(B_j)}_{1\leq j \leq J}$, with $\sum_{j=1}^J B_j = N_x$. Here, we consider two settings for the conditioning of $\Q^*$, namely one with a condition number close to one (i.e., $\Q^*$ is close to identity matrix) and another with  high condition number. The main characteristics of the datasets and their names are summarized in Table \ref{tab:datasetclimate}. 

We evaluate our results in terms on quality assessment metrics of the estimated $\widehat{\A}$ when compared to its ground truth $\A^*$, as this is the quantity of interest for these datasets. We compute RMSE$(\widehat{\A},\A^*)$, as well as the precision, recall, specificity, accuracy, and F1 score for detecting the non-zero entries of the transition matrix (that is, the graph edges positions). A threshold value of $10^{-10}$ on the absolute entries is used for the detection hypothesis. 

As for comparison, we also provide the results obtained with MLEM and GRAPHEM approaches, as in our previous experiments. The same stopping criterion than in our previous set of experiments \cblue{has been} used. \cblue{Since the datasets are synthetic, the ground truth is available. The hyperparameters $(\lambda_A,\lambda_P)$ for DGLASSO, and $\lambda_A$ for GRAPHEM, are finetuned on the rough grid $\{1,5,10\}$, to minimize RMSE$(\widehat{\A},\A^*)$, on one example randomly chosen, per each dataset. We then keep the parameters fixed for all the dataset, to limit overfitting behavior.} In addition to these EM-based methods, we provide comparisons with two Granger-causality approaches for graphical modeling, namely pairwise Granger Causality (PGC) and conditional Granger Causality (CGC). {Those approaches are based on sparse vector autoregressive (VAR) models. We allow the order of the VAR process to be up to $p=11$, CGC is run with a maximum distance of $5$ causal links, and in each experiment we display the best results (in F1 score) for the statistical tests performed with the significance level $\alpha\in\{0.01, 0.05, 0.1\}$ (see more details in \cite{Luengo19}).} As PGC and CGC do not provide a weighted graph estimation, no RMSE score is computed in those cases.

\subsubsection{Results}

We summarize our results in Table~\ref{tab:resultsClimate}. We also report the averaged inference time for all methods. Remarkably, the proposed method outperforms all benchmarks in all the considered metrics, related to the inference of the graph weights (RMSE metric) and the edge detection task (binary classification metrics). As expected, the result quality tends to slightly degrade when a more complicated structure is assumed for $\Q^*$ (see, for instance, WeathN5a versus WeathN5b), and when the dataset size increases (see, for instance, WeathN5a versus WeathN10a). MLEM has very poor edge detection metrics, since it does not exploit any sparsity prior on the sought matrix. GRAPHEM has better behavior, but, in these datasets, it stays way behind the quality of DGLASSO, by all metrics (which was not the case for the controlled synthetic data). This shows that our method really makes the change when dealing with complex graphs and correlated noise in the observations. In terms of computational time, the proposed method has similar complexity than the two other EM-based approaches. 
{The results of CGC and PGC are satisfactory in the first two examples although the binary metrics are far from the performance of \acro. PGC still gives meaningful results in the last two examples, but CGC has a low performance due to a high number of false negatives (i.e., many existing links are not discovered). We also note that while PGC and CGC have significantly lower running times for the case with $N_y=5$, the computational cost for $N_y=10$ is closer to \acro~while the performance gap is higher. Thus, in this examples we show that as the dimension grows, the computational cost of all methods is comparable while the proposed method has clearly a better performance. }
We also display some examples of inferred graphs in Figures \ref{fig:climateN5a} and \ref{fig:climateN10a}. We can observe that DGLASSO is able to capture in a very accurate manner the structure of the graphs, despite the wide variability of the dataset. MLEM and GRAPHEM capture the main edges, but their graphs are perturbed by several spurious edges, which shows how important it is to adopt a joint graph modeling, with sparsity priors on each. 

\begin{table}
\centering
\begin{tabular}{|c|c|c|c|}
\hline
Dataset & $N_x = N_y$ & $(B_j)_{1 \leq j \leq J}$ & $\log_{10}(\text{cond}(\Q^*))$ \\
\hline 
WeathN5a & $5$ & $(2,3)$ & $0.1$ \\
WeathN5b & $5$ & $(2,3)$ & $1$\\
\hline
WeathN10a & $10$ & $(5,5)$ & $0.1$ \\
WeathN10b & $10$ & $(5,5)$ & $1$ \\
\hline 
\end{tabular}
\caption{Details about the datasets. Each dataset is associated with $200$ examples for matrix $\A^*$.}     
\label{tab:datasetclimate}           
\end{table}

\begin{table}[t]
\caption{Results for climate datasets along with computing times. All the metrics are averaged over the $200$ examples of the dataset.}
\centering
{\footnotesize
{
\begin{tabular}{|c|c||c|c|c|c|c|c||c|}
\cline{2-9}
\multicolumn{1}{c|}{ } & method & RMSE & accur. & prec. & recall & spec. & F1 & Time (s.)\\
\hline
\multirow{5}{*}{WeathN5a} & DGLASSO & $\ubold \num{0.10754}$ & $\ubold \num{0.9374}$ & $\num{0.89422}$ & $\num{0.99814}$ & $\num{0.89402}$ & $\ubold \num{0.93745}$ & $\num{0.60779}$\\
  & MLEM & $\num{0.13974}$ & $\num{0.4128}$ & $\num{0.4128}$ & $\ubold \num{1}$ & $\num{0}$ & $\num{0.58404}$ & $\num{0.59586}$\\
  & GRAPHEM & $\num{0.1266}$ & $\num{0.7034}$ & $\num{0.59524}$ & $\ubold \num{1}$ & $\num{0.49583}$ & $\num{0.7415}$ & $\num{0.60581}$\\
  & PGC& - &  $\num{0.7724}$ & $\ubold \num{0.90207}$ & $\num{0.51493}$ & $\ubold \num{0.95262}$ & $\num{0.6522}$ & $\num{0.018722}$\\
  & CGC & - &  $\num{0.6724}$ & $\num{0.82789}$ & $\num{0.28457}$ & $\num{0.94485}$ & $\num{0.41457}$ & $\num{0.025629}$\\
\hline
\hline
\multirow{5}{*}{WeathN5b} & DGLASSO & $\ubold \num{0.16589}$ & $\ubold \num{ 0.7728}$ & $\num{0.66751}$ & $\num{0.99191}$ & $\num{0.61946}$ & $\ubold \num{0.78798}$ & $\num{0.62955}$\\
  & MLEM & $\num{0.19722}$ & $\num{0.4128}$ & $\num{0.4128}$ & $\ubold \num{1}$ & $\num{0}$ & $\num{0.58404}$ & $\num{0.37552}$\\
  & GRAPHEM & $\num{0.18582}$ & $\num{0.629}$ & $\num{0.53571}$ & $\ubold \num{1}$ & $\num{0.36833}$ & $\num{0.69446}$ & $\num{0.4696}$\\
  & PGC& - &  $\num{0.6752}$ & $\ubold \num{0.67735}$ & $\num{0.46933}$ & $\num{0.81897}$ & $\num{0.54432}$ & $\num{0.016851}$\\
  & CGC & - &  $\num{0.634}$ & $\num{0.65877}$ & $\num{0.26271}$ & $\ubold \num{0.89501}$ & $\num{0.36908}$ & $\num{0.023102}$\\
\hline
\hline
\multirow{5}{*}{WeathN10a} & DGLASSO & $\ubold \num{0.20195}$ & $\ubold \num{0.948}$ & $\ubold \num{0.89848}$ & $\num{0.92543}$ & $\num{0.95412}$ & $\ubold \num{0.89022}$ & $\num{1.3634}$\\
  & MLEM & $\num{0.26438}$ & $\num{0.21915}$ & $\num{0.21915}$ & $\ubold \num{1}$ & $\num{0}$ & $\num{0.359}$ & $\num{0.83396}$\\
  & GRAPHEM & $\num{0.2243}$ & $\num{0.51095}$ & $\num{0.31113}$ & $\num{0.99979}$ & $\num{0.37386}$ & $\num{0.47333}$ & $\num{1.4449}$\\
  & PGC& - &  $\num{0.8786}$ & $\num{0.9042}$ & $\num{0.50416}$ & $\ubold \num{0.98346}$ & $\num{0.6438}$ & $\num{0.23204}$\\
  & CGC & - &  $\num{0.77325}$ & $\num{0.53947}$ & $\num{0.21057}$ & $\num{0.93192}$ & $\num{0.27802}$ & $\num{0.3577}$\\
\hline
\hline
\multirow{5}{*}{WeathN10b} & DGLASSO & $\ubold {\num{0.19222}}$ & $\ubold \num{0.86585}$ & $\ubold \num{0.6334}$ & $\num{0.9944}$ & $\num{0.8293}$ & $\ubold \num{0.76915}$ & $\num{0.55689}$\\
  & MLEM & $\num{0.34208}$ & $\num{0.21915}$ & $\num{0.21915}$ & $\ubold \num{1}$ & $\num{0}$ & $\num{0.359}$ & $\num{0.98891}$\\
  & GRAPHEM & $\num{0.21936}$ & $\num{0.8545}$ & $\num{0.61983}$ & $\num{0.99404}$ & $\num{0.81552}$ & $\num{0.75718}$ & $\num{0.65545}$\\
  & PGC& - &  $\num{0.7987}$ & $\num{0.55829}$ & $\num{0.47325}$ & $\num{0.88988}$ & $\num{0.50604}$ & $\num{0.153757}$\\
  & CGC & - &  $\num{0.75035}$ & $\num{0.40668}$ & $\num{0.21762}$ & $\ubold \num{0.9002}$ & $\num{0.26462}$ & $\num{0.17846}$\\
\hline
\end{tabular}
}
}
\label{tab:resultsClimate}
\end{table}

\begin{figure}[t]
\centering
\begin{tabular}{@{}c@{}c@{}c@{}c@{}}%
\includegraphics[width = 4cm]{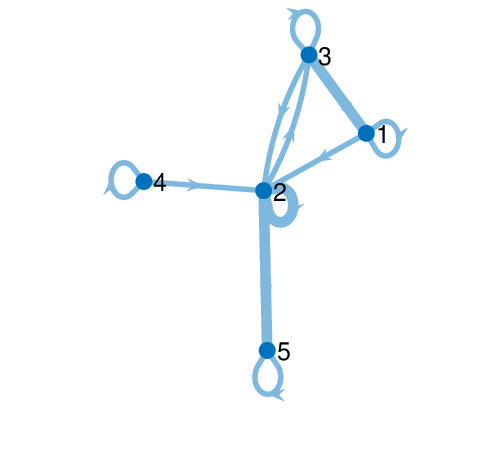} & \includegraphics[width = 4cm]{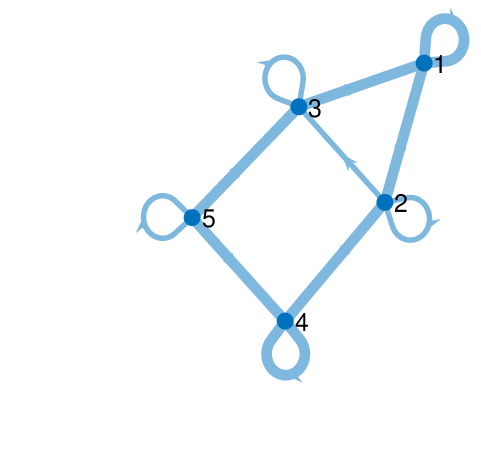} & \includegraphics[width = 4cm]{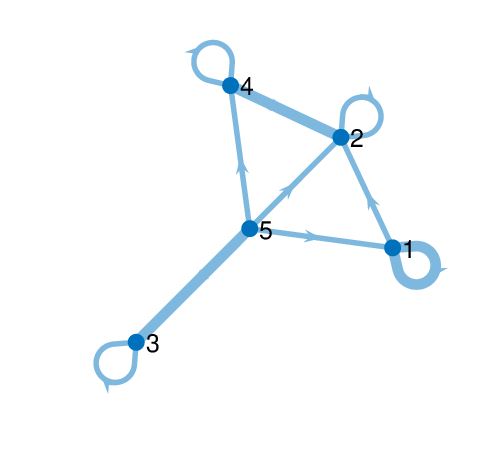} & \includegraphics[width = 4cm]{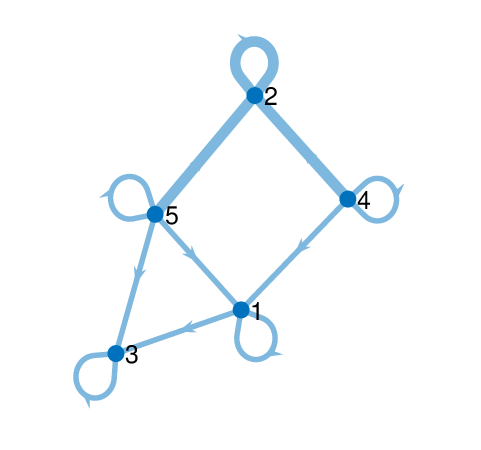}\\
\includegraphics[width = 4cm]{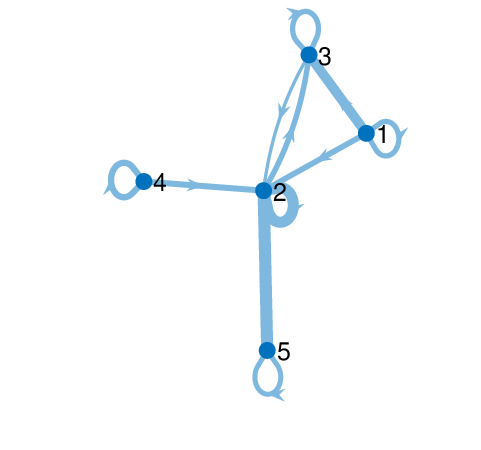} & \includegraphics[width = 4cm]{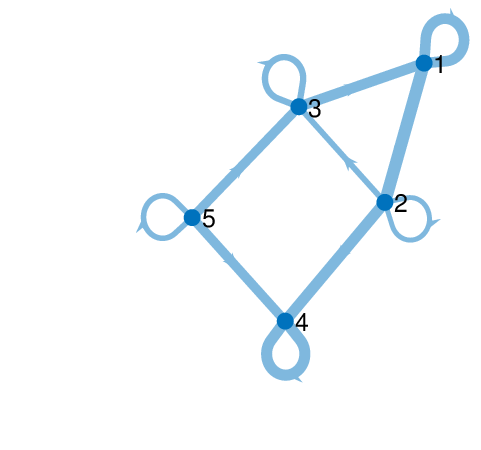} & \includegraphics[width = 4cm]{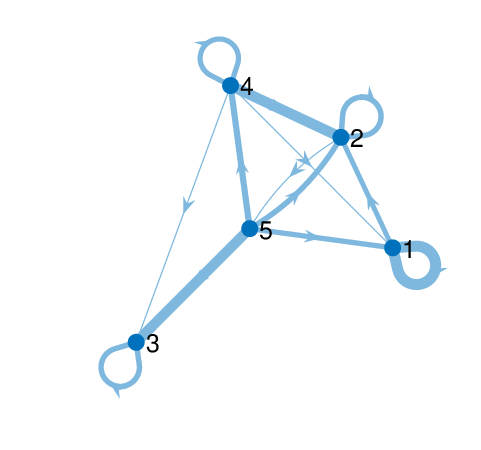} & \includegraphics[width =4cm]{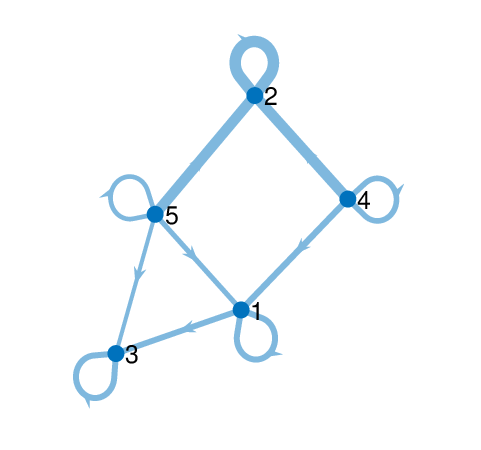}\\
\includegraphics[width = 4cm]{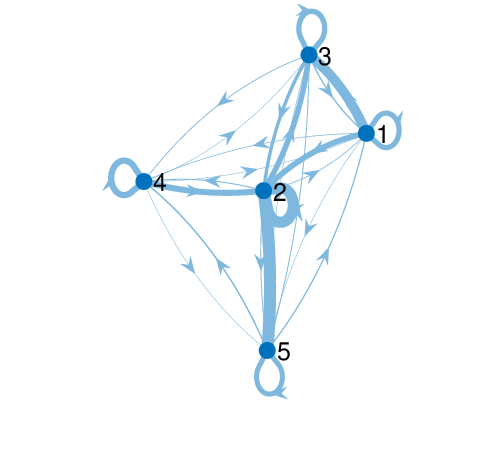} & \includegraphics[width = 4cm]{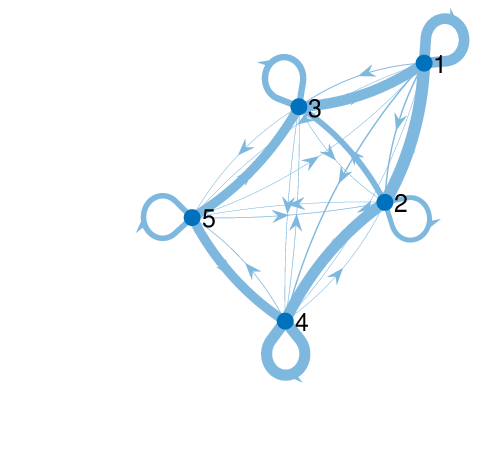} & \includegraphics[width = 4cm]{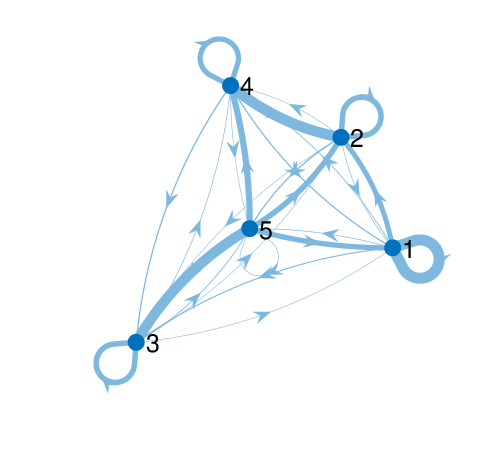} & \includegraphics[width = 4cm]{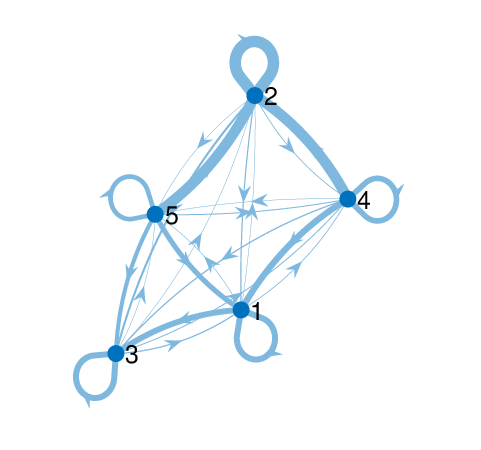}\\
\includegraphics[width = 4cm]{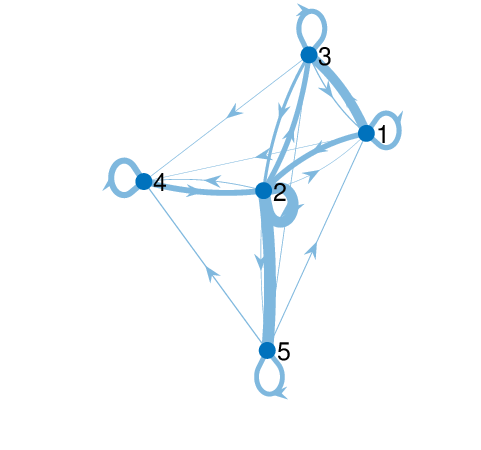} & \includegraphics[width = 4cm]{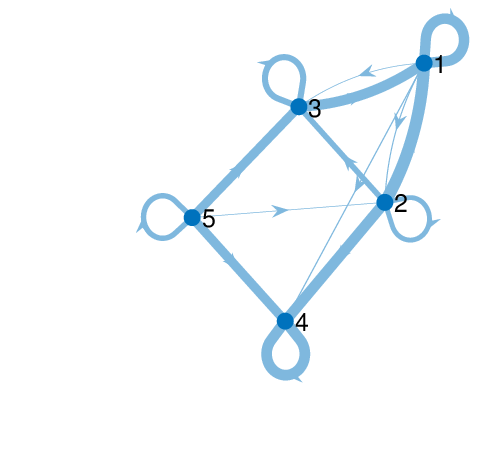} & \includegraphics[width = 4cm]{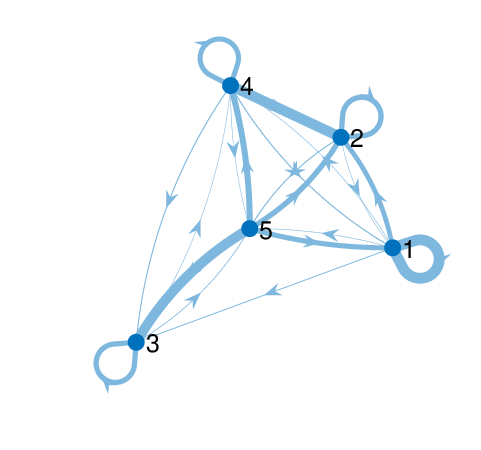} & \includegraphics[width = 4cm]{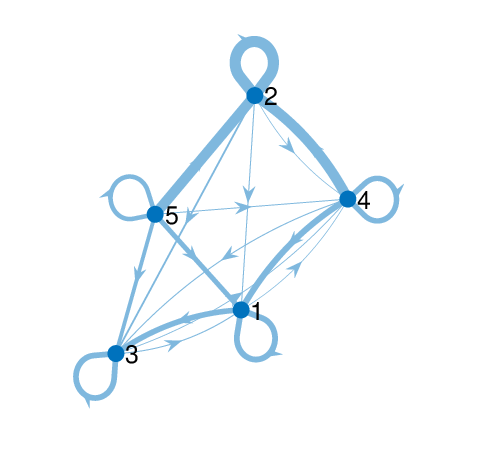}
\end{tabular}
\caption{Graph inference results on $4$ examples extracted from the dataset \emph{WeathN5a}. From top to bottom: Original graph representation of $\A^*$, and of its estimation $\widehat{\A}$, using DGLASSO, MLEM, GRAPHEM, respectively.}
\label{fig:climateN5a}
\end{figure}

\begin{figure}[t]
\centering
\begin{tabular}{@{}c@{}c@{}c@{}c@{}}
\includegraphics[width = 4cm]{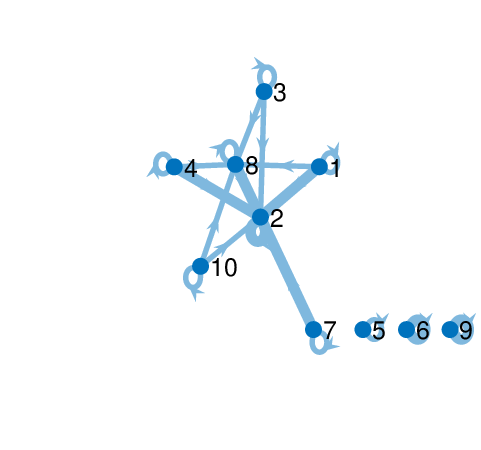} & \includegraphics[width = 4cm]{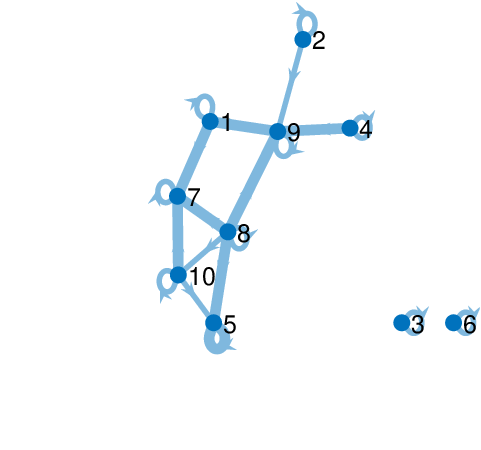} & \includegraphics[width = 4cm]{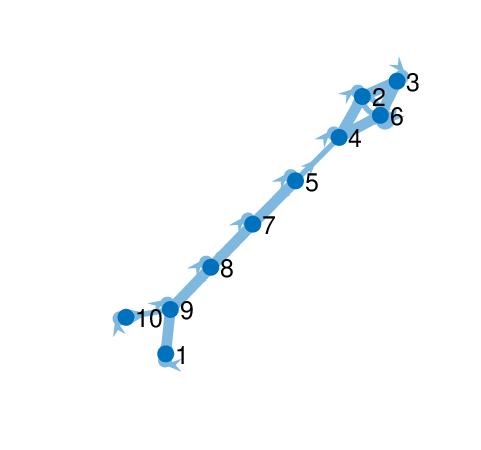} & \includegraphics[width = 4cm]{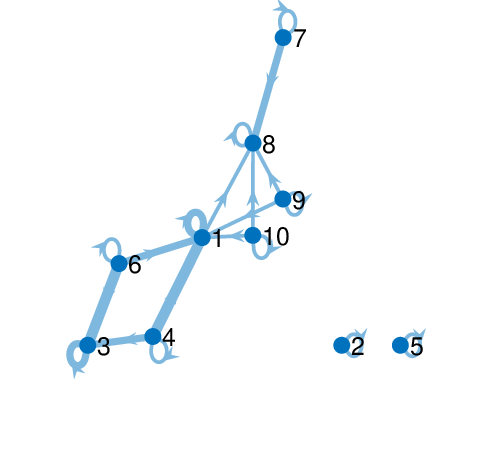}\\
\includegraphics[width = 4cm]{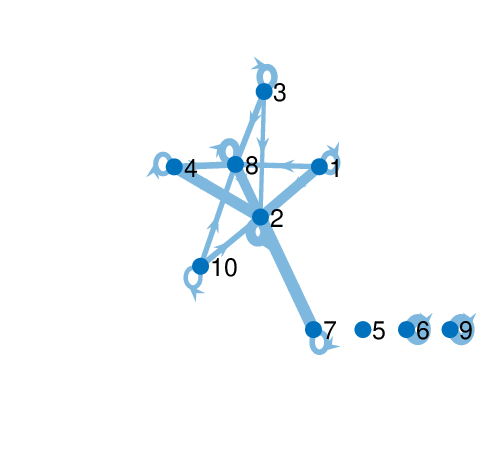} & \includegraphics[width = 4cm]{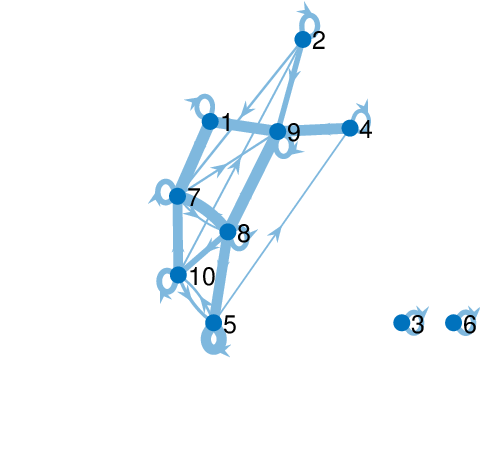} & \includegraphics[width = 4cm]{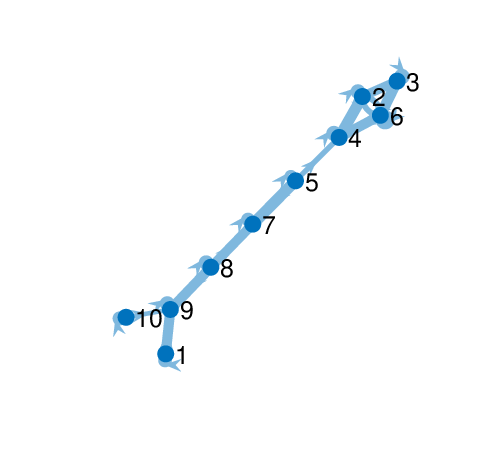} & \includegraphics[width = 4cm]{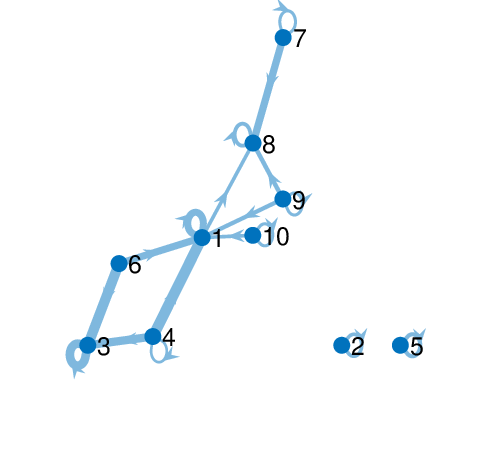}\\
\includegraphics[width = 4cm]{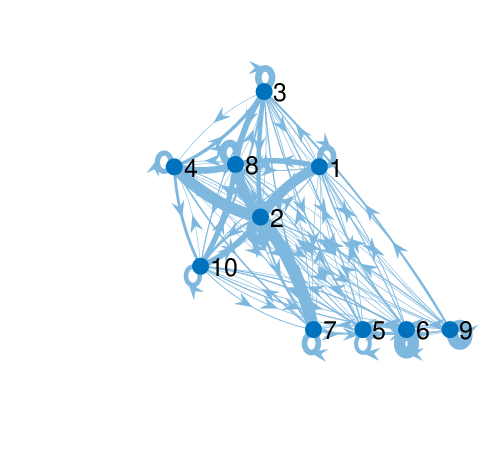} & \includegraphics[width = 4cm]{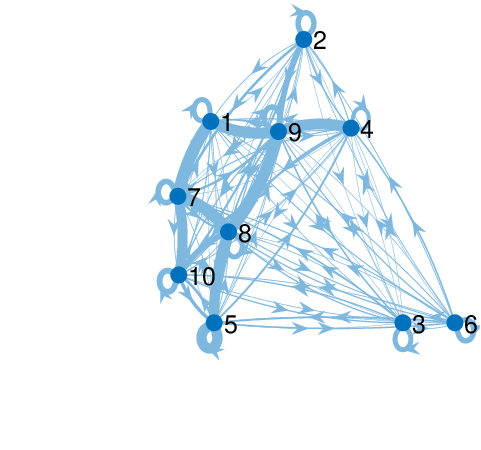} & \includegraphics[width = 4cm]{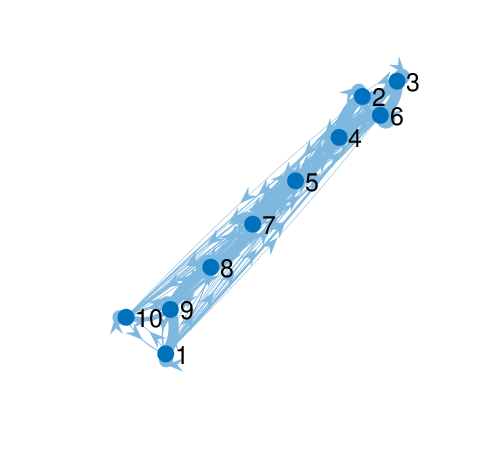} & \includegraphics[width = 4cm]{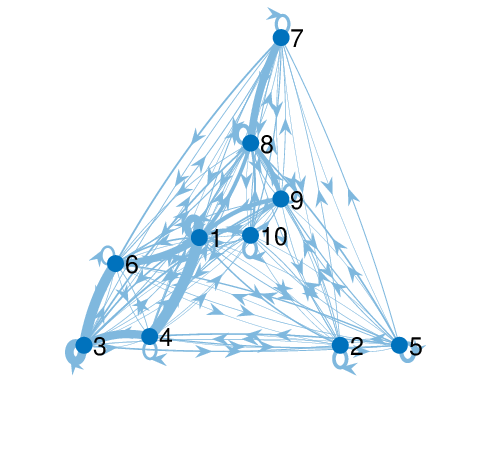}\\
\includegraphics[width = 4cm]{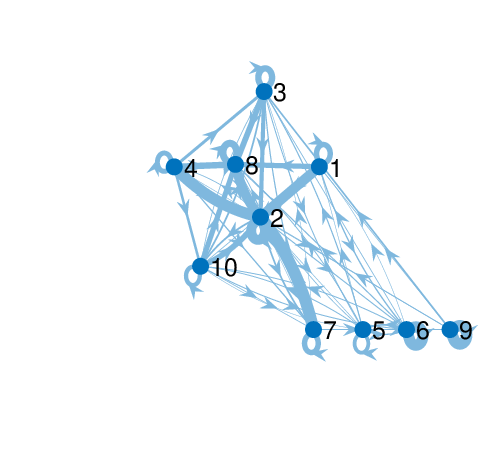} & \includegraphics[width = 4cm]{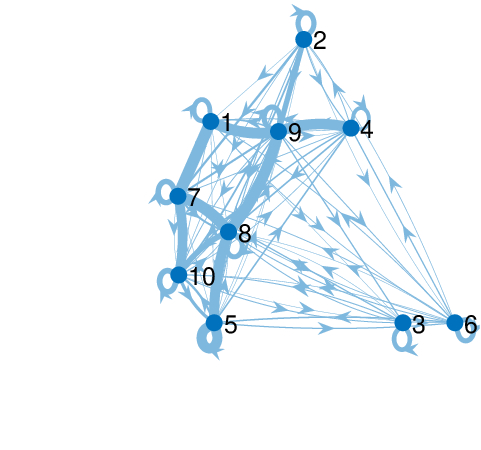} & \includegraphics[width = 4cm]{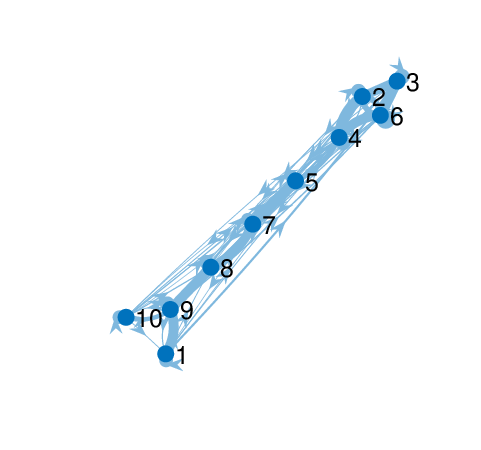} & \includegraphics[width = 4cm]{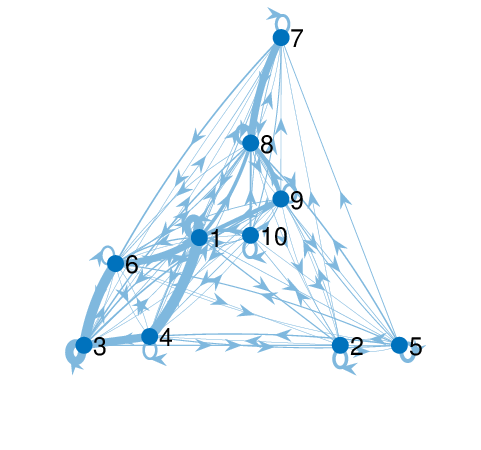}
\end{tabular}
\caption{Graph inference results on $4$ examples extracted from the dataset \emph{WeathN10a}. From top to bottom: Original graph representation of $\A^*$, and of its estimation $\widehat{\A}$, using DGLASSO, MLEM, GRAPHEM, respectively.}
\label{fig:climateN10a}
\end{figure}



\section{Conclusion}
{
This paper proposes a joint graphical modeling approach that incorporates a graphical perspective on the dynamics of the hidden state of a linear Gaussian state-space model. In particular, we propose a joint approach that considers a sparse undirected graph as a prior on the precision matrix of the hidden state noise and a sparse directed graph for the transition matrix that models the state dynamics. By bridging the gap between the static graphical Lasso model and the dynamic state-space model, we provide a novel comprehensive framework for interpretable inference with time-series data. The presented inference method, based on an efficient block alternating majorization-minimization algorithm, enables simultaneous estimation of both graphs and the construction of the filtering/smoothing distribution for the time series. The established convergence of our algorithm, departing from recent nonlinear analysis tools, enhances the reliability and practicality of our approach. Through extensive experimental validation on \cblue{synthetic data}, we have demonstrated the effectiveness and potential of our proposed model and inference algorithm. The results showcase its ability to uncover meaningful insights from time-series data, contributing not only to better forecasting performance but also to a better understanding of complex phenomena in various scientific and engineering domains. Future research can further explore \cblue{automatic hyperparameter tuning}, and extend the presented framework to tackle even more challenging state-space models, and more complex graphical structures.}

\acks{\'E.C. acknowledges support from the European Research Council Starting Grant MAJORIS ERC-2019-STG-850925. The work of V. E. is supported by the \emph{Agence Nationale de la Recherche} of France under PISCES (ANR-17-CE40-0031-01), the Leverhulme Research Fellowship (RF-2021-593), and by ARL/ARO under grant W911NF-22-1-0235.\\
The authors thank Gustau Camps-Valls and his team, for providing the dataset from~\cite{Runge2020}.}


\newpage

\appendix

\section{Proof of Theorem 1}
\label{sec:appendixA}

For any initial state $\x_{0:K}$, with non zero probability, the neg-log-likelihood $(\A,\P) \to \left(\loss_{1:K}(\A,\P) = - \log p(\y_{1:K}|\Ab,\P)\right)$ reads, according to Bayes' rule,
\begin{multline}
(\forall \A \in \mathbb{R}^{N_x \times N_x})(\forall \P \in \mathcal{S}_{N_x})\\
\loss_{1:K}(\A,\P) =  - \log p(\x_{0:K},\y_{1:K} | \Ab,\P) + \log p(\x_{0:K}|\y_{1:K},\Ab,\P). \label{eq:ML1}
\end{multline}
\cblue{Again, note that the l.h.s. does not depend on $\x_{0:K}$, so the r.h.s. is valid for any arbitrary value of $\x_{0:K}$ with non-zero probability under $p(\x_{0:K})$, i.e., for all $\x_{0:K} \in \Real^{(K+1)N_x}$.}
According to Eqs.~\eqref{eq_model_state}-\eqref{eq_model_obs}
\begin{equation}
\log p(\x_{0:K},\y_{1:K} | \Ab,\P)  = \log p(\x_0) + \sum_{k=1}^K \log p(\x_{k}|\x_{k-1}, \Ab,\P) + \sum_{k=1}^K \log p(\y_{k}|\x_k).
\label{eq_log_joint_lk}
\end{equation}
Moreover, using again Eq.~\eqref{eq_model_state} and the statistical model of the state noise,
\begin{align}
(\forall \A \in \mathbb{R}^{N_x \times N_x})&(\forall \P \in \mathcal{S}_{N_x})\\
\sum_{k=1}^K \log p(\x_{k}|\x_{k-1}, \Ab,\P) &= -\frac{{1}}{2}\sum_{k=1}^K \left(\left(\x_k - \A \x_{k-1} \right)^\top \P \left(\x_k - \A \x_{k-1}  \right) - 
\log \det({2\pi} {\P})\right),\\
& = -\frac{{1}}{2}\sum_{k=1}^K \left(\left(\x_k - \A \x_{k-1} \right)^\top \P  \left(\x_k - \A \x_{k-1}  \right) \right) + \frac{K}{2}\log \det({2\pi} {\P}),
\label{eq_log_joint_lk_2}
\end{align}
which concludes the first part of the proof.\\
Let us now consider some $\widetilde{\Ab} \in \mathbb{R}^{N_x \times N_x}$ and $\widetilde{\P} \in \mathcal{S}_{N_x}^{++}$. We start by recalling some known results arising from the EM methodology \citep{Dempster,Wu}, that we specify here for our context for the sake of clarity. First, we notice that
\begin{multline}
\label{eq:intconstant}
(\forall \A \in \mathbb{R}^{N_x \times N_x})(\forall \P \in \mathcal{S}_{N_x})\\
\loss_{1:K}(\A,\P) = - \int \log p(\y_{1:K}|\Ab,\P) p(\x_{0:K} | \y_{1:K},\widetilde{\Ab},\widetilde{\P}) \rm{d} \x_{0:K},
\end{multline}
since $\log p(\y_{1:K}|\Ab,\P)$ is constant with respect to the integration variable, and the distribution $p(\x_{0:K} | \y_{1:K},\widetilde{\Ab},\widetilde{\P})$ integrates to one. Then, according to \eqref{eq:ML1} and \eqref{eq:intconstant}, the expectation of the neg-log-likelihood multiplied by $p(\x_{0:K}|\y_{1:K},\widetilde{\Ab},\widetilde{\P})$ over all possible values of the
unknown state reads:
\begin{multline}
(\forall \A \in \mathbb{R}^{N_x \times N_x})(\forall \P \in \mathcal{S}_{N_x})\quad
 \loss_{1:K}(\A,\P) = \overbrace{- \int p(\x_{0:K} | \y_{1:K},\widetilde{\Ab},\widetilde{\P}) \log p(\x_{0:K},\y_{1:K} | \Ab,\P) \rm{d} \x_{0:K} }^{q(\A,\P;\widetilde{\A},\widetilde{\P})}\\
+ \overbrace{\int p(\x_{0:K} | \y_{1:K},\widetilde{\Ab},\widetilde{\P}) \log p(\x_{0:K}|\y_{1:K},\Ab,\P) \rm{d} \x_{0:K}}^{h(\A,\P;\widetilde{\A},\widetilde{\P})}.
\label{eq:Jensen1}
\end{multline}
In particular, for $(\A,\P) = (\widetilde{\Ab},\widetilde{\P})$, $\loss_{1:K}(\widetilde{\Ab},\widetilde{\P}) = q(\widetilde{\Ab},\widetilde{\P};\widetilde{\A},\widetilde{\P}) + h(\widetilde{\Ab},\widetilde{\P};\widetilde{\A},\widetilde{\P})$ so that
\begin{multline}
(\forall \A \in \mathbb{R}^{N_x \times N_x})(\forall \P \in \mathcal{S}_{N_x}) \quad
\loss_{1:K}(\A,\P) - \loss_{1:K}(\widetilde{\Ab},\widetilde{\P}) = q(\A,\P;\widetilde{\A},\widetilde{\P}) - q(\widetilde{\Ab},\widetilde{\P};\widetilde{\A},\widetilde{\P}) \\
+ h(\A,\P;\widetilde{\A},\widetilde{\P}) - h(\widetilde{\A},\widetilde{\P};\widetilde{\A},\widetilde{\P}).
\label{eq:majexpe}
\end{multline}
Using Gibbs's inequality, $h(\A,\P;\widetilde{\A},\widetilde{\P}) \leq h(\widetilde{\A},\widetilde{\P};\widetilde{\A},\widetilde{\P})$, with equality at $(\A,\P) = (\widetilde{\Ab},\widetilde{\P})$. Thus, using \eqref{eq:majexpe},
that is
\begin{equation}
(\forall \A \in \mathbb{R}^{N_x \times N_x})(\forall \P \in \mathcal{S}_{N_x}) \quad
\loss_{1:K}(\A,\P) \leq q(\A,\P;\widetilde{\A},\widetilde{\P}) + \loss_{1:K}(\widetilde{\Ab},\widetilde{\P}) - q(\widetilde{\Ab},\widetilde{\P};\widetilde{\A},\widetilde{\P}), \label{eq:majexact}
\end{equation}
where equality holds at $(\A,\P) = (\widetilde{\Ab},\widetilde{\P})$. Notice that, for any function reading
\begin{equation}
(\forall \A \in \mathbb{R}^{N_x \times N_x})(\forall \P \in \mathcal{S}_{N_x}) \quad
\maj(\A,\P;\widetilde{\A},\widetilde{\P}) = q(\A,\P;\widetilde{\A},\widetilde{\P}) + \ctAP, \label{eq:majconstant2}
\end{equation}
we obviously still have
\begin{equation}
(\forall \A \in \mathbb{R}^{N_x \times N_x})(\forall \P \in \mathcal{S}_{N_x}) \quad
\loss_{1:K}(\A,\P) \leq \maj(\A,\P;\widetilde{\A},\widetilde{\P}) + \loss_{1:K}(\widetilde{\Ab},\widetilde{\P}) - \maj(\widetilde{\Ab},\widetilde{\P};\widetilde{\A},\widetilde{\P}), \label{eq:majexact2}
\end{equation}
where equality hereagain holds at $(\A,\P) = (\widetilde{\Ab},\widetilde{\P})$.
Using \eqref{eq:majexact2}, \eqref{eq:phik}, \eqref{eq:priorgen} and noticing that, for every $(\theta_A,\theta_P)>0$,
\begin{equation}
(\forall \A \in \mathbb{R}^{N_x \times N_x})(\forall \P \in \mathcal{S}_{N_x}) \quad
\frac{1}{2 \theta_A}\| \A - \widetilde{\A}\|^2_F \geq 0, \quad \frac{1}{2 \theta_P} \| \P - \widetilde{\P}\|^2_F \geq 0, 
\end{equation}
with equality holding at $(\A,\P) = (\widetilde{\Ab},\widetilde{\P})$, we deduce the desired majorizing property
\begin{multline}
(\forall \A \in \mathbb{R}^{N_x \times N_x})(\forall \P \in \mathcal{S}_{N_x}) \quad
\loss(\A,\P) \leq \maj(\A,\P;\widetilde{\A},\widetilde{\P}) + \loss_{1:K}(\widetilde{\Ab},\widetilde{\P}) - \maj(\widetilde{\Ab},\widetilde{\P};\widetilde{\A},\widetilde{\P}) \\
+ \lambda_A \|\A\|_1 + \lambda_P \|\P\|_1 + \frac{1}{2 \theta_A}\| \A - \widetilde{\A}\|^2_F + 
\frac{1 }{2 \theta_P} \| \P - \widetilde{\P}\|^2_F,
\end{multline}
with equality holding at $(\A,\P) = (\widetilde{\Ab},\widetilde{\P})$. The remaining of the proof amounts to expliciting the expression for $(\A,\P) \to \maj(\A,\P;\widetilde{\A},\widetilde{\P})$ satisfying \eqref{eq:majconstant2} with function $q$ defined as in~\eqref{eq:Jensen1}:
\begin{multline}
(\forall \A \in \mathbb{R}^{N_x \times N_x})(\forall \P\in \mathcal{S}_{N_x}) \\ 
q(\Ab,\P;\widetilde{\A},\widetilde{\P}) = - \int p(\x_{0:K} | \y_{1:K},\widetilde{\A},\widetilde{\P}) \log p(\x_{0:K},\y_{1:K} | \Ab,\P) \rm{d} \x_{0:K}.
\label{eq:funQ0}
\end{multline}
Following \citep[Theorem 12.4]{Sarkka} (see also an alternative proof in \citep[Sec. III-B]{elvira2022graphem}), \eqref{eq:Jensen1}-\eqref{eq:majconstant2} hold for

{\small
{
\begin{equation}
(\forall \A \in \mathbb{R}^{N_x \times N_x})(\forall \P \in \mathcal{S}_{N_x}) \quad
\maj(\Ab,\P;\widetilde{\A},\widetilde{\P}) = \frac{K}{2} \text{tr}\left(\P(  
\cblue{\widetilde{\Psib}} - \cblue{\widetilde{\Deltab}} \Ab^\top - \Ab \cblue{\widetilde{\Deltab}}^\top + \Ab \cblue{\widetilde{\Phib}} \Ab^\top) \right) - \frac{K}{2}\log \det({2\pi} {\P}),
\label{eq:g}
\end{equation}
}
}
and
{
\begin{align}
\cblue{\widetilde{\Psib}} & = \frac{1}{K} \sum_{k=1}^K  \left(\covSmooth_k + \meanSmooth_{k} (\meanSmooth_{k})^\top\right), \label{eq:Sig}\\
\cblue{\widetilde{\Deltab}} & = \frac{1}{K} \sum_{k=1}^K  \left(\covSmooth_k \G_{k-1}^\top  + \meanSmooth_{k} (\meanSmooth_{k-1})^\top\right), \label{eq:C} \\
\cblue{\widetilde{\Phib}} & = \frac{1}{K} \sum_{k=1}^K  \left(\covSmooth_{k-1} + \meanSmooth_{k-1} (\meanSmooth_{k-1})^\top\right).
\label{eq:Phi}
\end{align}
}
Hereabove, $(\meanSmooth_k,\covSmooth_k)_{0 \leq k \leq K-1}$ denotes the mean and covariance of the smoothing distribution obtained when running Algs.~\ref{alg_kf}-\ref{alg_rts} using $(\widetilde{\A},\widetilde{\P})$. 
Moreover, the matrix
\begin{equation}
\G_{k} = \covPred_k \widetilde{\A}^\top \left(\widetilde{\A} \covPred_k \widetilde{\A}^\top + \widetilde{\P}^{-1} \right) \label{eq:Gk}
\end{equation}
is defined as in Algorithm~\ref{alg_rts}. This concludes the proof.

\section{Proximal algorithms to solve the inner steps}

\label{sec:appendixB}

We present Algorithms~\ref{algo:ProxA} and \ref{algo:ProxP}, that are proximal splitting algorithms to solve, respectively, the inner problems \eqref{eq:updateAexplicit} and \eqref{eq:updatePexplicit}. Specifically, both algorithms are special instances of the Dykstra-like splitting algorithm from \cite{Dystra2008} (see also \cite[Sec.5]{Combettes2011}), for the minimization of the sum of two convex but non-differentiable functions. Sequence $(\A_n)_{n \in \mathbb{N}}$ (resp. $(\P_n)_{n \in \mathbb{N}}$) is guaranteed to converge to the solution of problem \eqref{eq:updateAexplicit} (resp. \eqref{eq:updatePexplicit}). The proximity steps involved in Algorithms~\ref{algo:ProxA} and \ref{algo:ProxP} have closed form expressions that can be found for instance in \cite{bauschke2017convex}. We explicit them hereafter for the sake of completeness. 

\paragraph*{Proximity of $\ell_1$.} Let $\gamma>0$ and $\widetilde{\V} \in \mathbb{R}^{N_x \times N_x}$. Then, 
\begin{align}
& \operatorname{prox}_{\gamma \ell_1}(\widetilde{\V})\\
& = \left(\text{sign}(\widetilde{V}(n,\ell) \max(0,  \widetilde{V}(n,\ell) - \gamma) \right)_{1 \leq n,\ell \leq N_x},
\end{align}
which amounts to applying the soft thresholding operator with weight $\gamma$ on every entry of the matrix input $\widetilde{\V}$. 

\paragraph*{Proximity of quadratic term.}  Let $\gamma>0$ and $\widetilde{\W} \in \mathbb{R}^{N_x \times N_x}$. Then, by definition, 
\begin{align}
\widehat{\Z} & = \operatorname{prox}_{\W \to \gamma \rm{tr}\left(- \widetilde{\P} \cblue{\widetilde{\Deltab}} \W - \widetilde{\P} \W \cblue{\widetilde{\Deltab}}^\top + \widetilde{\P} \W\cblue{\widetilde{\Psib}} \W^\top \right)} \left(\widetilde{\W} \right)\\
& = \argmin{\W \in \mathbb{R}^{N_x \times N_x}}{\gamma \rm{tr}\left(- \widetilde{\P} \cblue{\widetilde{\Deltab}} \W^\top  - \widetilde{\P} \W \cblue{\widetilde{\Deltab}}^\top + \widetilde{\P} \W\cblue{\widetilde{\Psib}} \W^\top \right) + \frac{1}{2} \| \W - \widetilde{\W}\|_F^2}. \label{eq:proxtrace}
\end{align}
The optimality condition for \eqref{eq:proxtrace} gives
\begin{equation}
- \gamma \widetilde{\P} \cblue{\widetilde{\Deltab}} - \gamma \cblue{\widetilde{\Deltab}}^\top \widetilde{\P} + \gamma \widetilde{\P} \widehat{\Z} \cblue{\widetilde{\Psib}}
+ \gamma \widetilde{\P}^\top \widehat{\Z} \cblue{\widetilde{\Psib}}^\top + \widehat{\Z} - \widetilde{\W} = \mathbf{0}.
\end{equation}
Since $\widetilde{\P} \in \mathcal{S}_{N_x}^{++}$, and $\cblue{\widetilde{\Psib}} \in \mathcal{S}_{N_x}$ (by construction), we have equivalently,
\begin{equation}
- \gamma \cblue{\widetilde{\Deltab}} - \gamma  \widetilde{\P}^{-1} \cblue{\widetilde{\Deltab}}^\top \widetilde{\P} + 2 \gamma \widehat{\Z} \cblue{\widetilde{\Psib}}
 + \widetilde{\P}^{-1} \widehat{\Z} - \widetilde{\P}^{-1} \widetilde{\W} = \mathbf{0}.
\end{equation}
Thus,
\begin{equation}
\widehat{\Z} = \operatorname{lyapunov}\left(\widetilde{\P}^{-1}, 2 \gamma \cblue{\widetilde{\Psib}},\gamma (\cblue{\widetilde{\Deltab}} + \widetilde{\P}^{-1} \cblue{\widetilde{\Deltab}}^\top \widetilde{\P}) + \widetilde{\P}^{-1} \widetilde{\W}\right),
\end{equation}
where $\mathsf{A} = \operatorname{lyapunov}(\mathsf{X},\mathsf{Y},\mathsf{Z})$ provides the solution to the Lyapunov equation $\mathsf{XA + AY = Z}$. 
\paragraph*{Proximity of log-determinant term.} Let $\gamma>0$ and $\widetilde{\W} \in \mathcal{S}_{N_x}^{++}$. Then, by definition, 
\begin{align}
\widehat{\Z} & = \operatorname{prox}_{\W \to \gamma (- \log \det(\W) + \text{tr}(\W \cblue{\widetilde{\Pib}}) )} \left(\widetilde{\W} \right)\\
& = \argmin{\W \in \mathcal{S}_{N_x}}{ - \gamma \log \det(\W) + \gamma \text{tr}(\W \cblue{\widetilde{\Pib}}) + \frac{1}{2} \| \W - \widetilde{\W}\|_F^2}.
\end{align}
Using \cite[Chap. 24]{bauschke2017convex} (see also  \cite{Benfenati18}), for every $\alpha>0$, 
\begin{align}
\widehat{\Z} & =  \U \text{Diag}\left(\left(\frac{1}{2} (\omega(n) + \sqrt{\omega(n)^2 + 4  \gamma}) \right)_{1 \leq n \leq N_x} \right) \U^\top
\end{align}
where $\boldsymbol{\omega} = (\omega(n))_{1 \leq n \leq N_x} $ gathers the eigenvalues of $\widetilde{\W} - \gamma \cblue{\widetilde{\Pib}} \in \mathcal{S}_{N_x}$ and $\U  \in \mathbb{R}^{N_x \times N_x}$ is an orthogonal matrix such that 
\begin{equation}
\widetilde{\W} - \gamma \cblue{\widetilde{\Pib}} = \U \text{Diag}(\boldsymbol{\omega}) \U^\top.
\end{equation}

\begin{table}[h!]
\vspace{4mm}
    \centering
    \begin{tabular}{|p{0.95\columnwidth}|}
    \hline
\begin{Algoritmo}
\label{algo:ProxA}
Proximal splitting method to solve \eqref{eq:updateAexplicit}
\begin{enumerate}
  \item[] \textbf{Inputs.} $\widetilde{\A}, \widetilde{\P}, \cblue{\widetilde{\Psib}}, \cblue{\widetilde{\Deltab}}, \cblue{\widetilde{\Phib}}$. Precision $\xi>0$.
\item \textbf{Setting.} Set stepsize $\vartheta \in (0,2)$.
  \item \textbf{Initialization.}  Set $\V_0 = \widetilde{\A}$.
\item \textbf{Recursive step.} For $n=1,2,\ldots$:  
\begin{equation*}
\begin{array}{l}
\A_n = \operatorname{prox}_{\theta_A \lambda_A \ell_1}\left(\widetilde{\Ab} - \V_n  \right)\\
\W_n = \V_n + \vartheta \A_n\\
\Z_n = \operatorname{prox}_{\W \to \frac{\vartheta \theta_A K}{2} \rm{tr}\left(- \widetilde{\P} \cblue{\widetilde{\Deltab}} \W - \widetilde{\P} \W \cblue{\widetilde{\Deltab}}^\top + \widetilde{\P} \W\cblue{\widetilde{\Psib}} \W^\top \right)} \left( \vartheta^{-1} \W_n \right)\\
\V_{n+1} = \W_n - \vartheta \Z_n.
 \end{array}
\end{equation*}
\item[] If  $|\mathcal{C}_1(\A_n) - \mathcal{C}_1(\A_{n-1}) | \leq \xi$, \textbf{stop the recursion}.
 %
\item[] \textbf{{Output.}} Transition matrix $\widehat{\A} = \A_n$.
\end{enumerate}
\end{Algoritmo}\\
        \hline
\end{tabular}\\
%
\vspace{1cm}
    \centering
    \begin{tabular}{|p{0.95\columnwidth}|}
    \hline
\begin{Algoritmo}
\label{algo:ProxP}
Proximal splitting method to solve \eqref{eq:updatePexplicit}
\begin{enumerate}
  \item[] \textbf{Inputs.} $\widetilde{\A}, \widetilde{\P}, \cblue{\widetilde{\Psib}}, \cblue{\widetilde{\Deltab}}, \cblue{\widetilde{\Phib}}$. Precision $\xi>0$.
\item \textbf{Setting.} Set stepsize $\vartheta \in (0,2)$ and $\cblue{\widetilde{\Pib}}$ as in \eqref{eq:Pib}.
  \item \textbf{Initialization.}  Set $\V_0 = \widetilde{\P}$.
\item \textbf{Recursive step.} For $n=1,2,\ldots$:  
\begin{equation*}
\begin{array}{l}
\P_n = \operatorname{prox}_{\theta_P \lambda_P \ell_1}\left(\widetilde{\P} - \V_n  \right)\\
\W_n = \V_n + \vartheta \P_n\\
\Z_n = \operatorname{prox}_{\W \to \frac{\vartheta \theta_P K}{2} \left(- \log \det(\W) + \rm{tr}(\cblue{\widetilde{\Pib}} \W)\right)} \left( \vartheta^{-1} \W_n \right)\\
\V_{n+1} = \W_n - \vartheta \Z_n.
 \end{array}
\end{equation*}
\item[] If  $|\mathcal{C}_2(\A_n) - \mathcal{C}_2(\A_{n-1}) | \leq \xi$, \textbf{stop the recursion}. 
 %
\item[] \textbf{{Output.}} Precision matrix $\widehat{\P} = \P_n$.
\end{enumerate}
\end{Algoritmo}\\
        \hline
\end{tabular}
\end{table}

\section{Additional experiments}
We present here additional experimental results completing Section~\ref{sec:synthetic}.

\subsection{Robustness to initialization} 
\label{appendix:init}

\acro~algorithm amounts to minimizing a non-convex loss function. As such, its results might be sensitive to the initialization of the algorithm. To evaluate this aspect, we consider the computation of $(\widehat{\A},\widehat{\P})$ given the observation of a single time series generated by the ground truth LG-SSM, when using 50 different initializations of DGLASSO algorithm. To do so, we use the same initialization strategy as discussed above, now with $(a,p)$ randomly selected as $p \sim \mathcal{U}([0,1])$ and $a \sim \mathcal{U}([0,1])$. Figure~\ref{fig:init} displays the box plots for RMSE and F1 scores obtained for dataset A. One can notice that the box plots are very concentrated, showing a good robustness of the method to its initialization, with the wider spreading observed for the F1 score on $\A$. Similar behavior was observed for the other three datasets.

\begin{figure}
\centering
\begin{tabular}{@{}c@{}c@{}}
\includegraphics[width = 7cm]{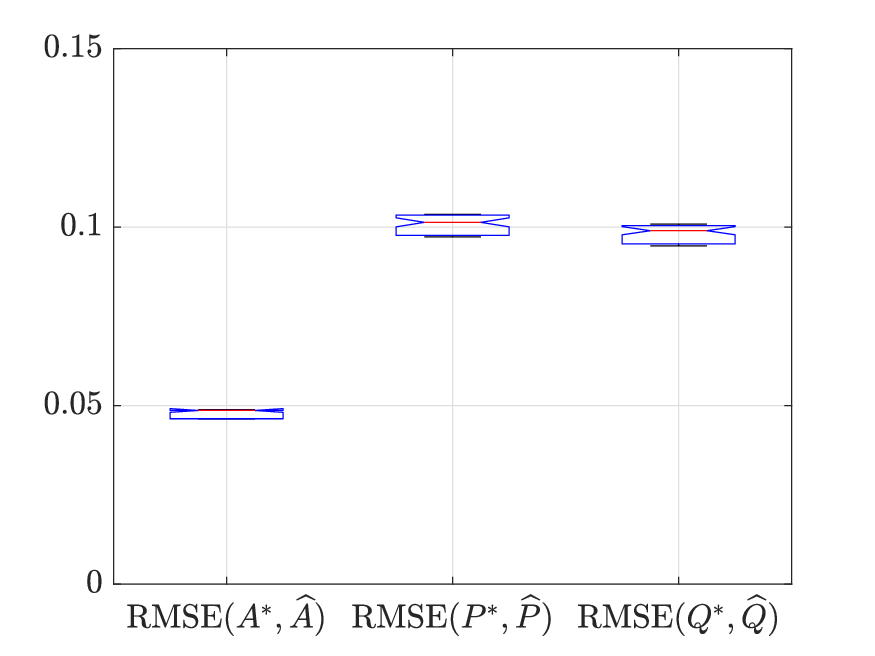} &
\includegraphics[width = 7cm]{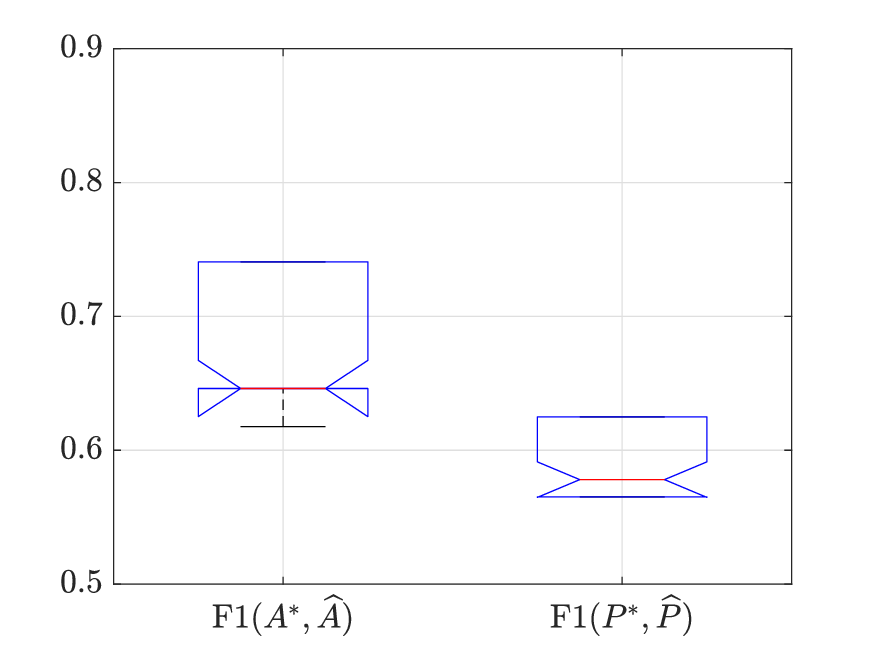} 
\end{tabular}
\caption{Box plots for the RMSE (left) and F1 (right) scores for retrieving $(\A,\P,\Q)$ matrices, when running DGLASSO on one single LG-SSM time series using dataset A, and 50 random initializations $(\A^{(0)},\P^{(0)})$. Noticeably, low variability is observed for all metrics.}
\label{fig:init}
\end{figure}

\subsection{Influence of sparsity level}
\label{appendix:sparse}
We evaluate here the performance of \acro, as well as the benchmarks, when varying the sparsity level of the ground truth matrices. To do so, we perform slight changes in the dataset, to vary the sparsity pattern of the matrices (i.e., the edge structure of the graphs). 
First, we modify the ground truth matrix $\A^*$ by keeping $s_A \in \{27,15,10,5\}$ entries of it, within the $27$ block diagonal ones, to be non-zero, the others being set to zero. We then rescaled the matrix to keep a spectral norm \cblue{equal to} $0.99$. Matrix $\Q^*$ is taken from dataset A. The results are reported in Table \ref{tab:resultssparseA}. 

As we can observe, the performance of MLEM, in terms of RMSE and F1 score, drop dramatically when the sparsity level on $\A$ increases. This is expected as this approach does not promote any sparsity prior on matrix $\A$. The best AUC are either obtained by MLEM, DGLASSO or rGLASSO (for $\P$), depending on the test cases. GLASSO/rGLASSO metrics slightly improve when $\A^*$ gets sparser, which is expected, as their assumption of a zero transition matrix gets more realistic. Hereagain, \acro~outperforms the benchmarks in most cases. 


\begin{table}[t]
\caption{Results for sparsity levels $s_A \in \{27,15,10,5\}$ of $\A^*$, using $(\P^*,\Q^*)$ from dataset A. The first set of results, with $s_A = 27$ identifies with the first row block of Tab.~\ref{tab:results}.}
\hspace*{-2cm}
{\tiny
{
\begin{tabular}{|c|c|c|c|c||c|c|c||c||c|c||c|c|}
\cline{3-13}
\multicolumn{2}{c|}{ } & \multicolumn{3}{|c||}{Estimation of $\A$} & \multicolumn{3}{|c||}{Estimation of $\P$} & Estim. $\Q$ & \multicolumn{2}{|c||}{State distrib.} & \multicolumn{2}{|c|}{Predictive distrib.}\\
 \cline{2-13}
\multicolumn{1}{c|}{ } & Method & RMSE & AUC & F1 & RMSE & AUC & F1 & RMSE & cNMSE$(\mub^*,\widehat{\mub})$ & cNMSE$(\mub^{\text{s}*},\widehat{\mub}^\text{s})$ & cNMSE$(\nub^*,\widehat{\nub})$ & $\mathcal{L}_{1:K}(\widehat{\A},\widehat{\P})$\\
\hline

\parbox[t]{2mm}{\multirow{5}{*}{\rotatebox[origin=c]{90}{$s_A = 27$}}} & DGLASSO & $\num{0.060525}$ & $\num{0.84255}$ & $\num{0.64057}$ & \ubold $\num{0.081927}$ & $\num{0.77801}$ & $\ubold \num{0.69812}$& $\ubold \num{0.082601}$ & $ \ubold \num{6.3935e-08}$ & $\ubold \num{1.0504e-07}$ & $\ubold \num{2.9837e-4}$ & $\ubold \num{12307.1687}$ \\
  & MLEM & $\num{0.076063}$ & $\num{0.81693}$ & $\num{0.5}$ & $\num{0.10528}$ & $\num{0.85739}$ & $\num{0.5}$& $\num{0.10152}$ & $\num{1.0947e-07}$ & $\num{1.8025e-07}$ & $\num{4.8425e-4}$ & $\num{12341.2045}$ \\
	& GLASSO & NA & NA & NA & $\num{0.81838}$ & $\num{0.80394}$ & $\num{0.49616}$ & $\num{1073.5104}$ & $\num{4.4853e-06}$ & $\num{7.1804e-06}$& $\num{1}$ & $\num{28459.2935}$ \\
	& rGLASSO & NA & NA & NA & $\num{0.76376}$ & $\ubold \num{0.92377}$ & $\num{0.59759}$ & $\num{31.6891}$ & $\num{2.8256e-06}$ & $\num{5.4924e-06}$ & $\num{1}$ & $\num{22957.6928}$ \\
		& GRAPHEM & $\ubold \num{0.044631}$ & $\ubold \num{0.89481}$ & $\ubold \num{0.84654}$ & NA & NA & NA & NA & $\num{4.3642e-06}$ & $\num{6.9443e-06}$ & $\num{2.98e-4}$ & $\num{29035.0296}$ \\
\hline
\hline
\parbox[t]{2mm}{\multirow{5}{*}{\rotatebox[origin=c]{90}{$s_A = 15$}}} & DGLASSO & $\ubold \num{0.10845}$ & $\num{0.91644}$ & $\num{0.78116}$ & $\ubold \num{0.077396}$ & $\num{0.88335}$ & $\ubold \num{0.85543}$& $\ubold \num{0.07547}$ & $\ubold \num{8.2675e-07}$ & $\ubold \num{1.1574e-06}$ & $\ubold \num{5.183e-3}$ & $\ubold \num{11801.2202}$\\
& MLEM & $\num{0.15894}$ & $\ubold  \num{0.93343}$ & $\num{0.3125}$ & $\num{0.10501}$ & $\ubold  \num{0.89948}$ & $\num{0.5}$ & $\num{0.10088}$ & $\num{1.2975e-06}$ & $\num{ 2.1095e-06}$ & $\num{6.0517e-3}$ & $\num{11806.4833}$\\
& GLASSO & NA & NA & NA & $\num{0.46309}$ & $\num{0.70222}$ & $\num{0.57579}$ & $\num{7.0875}$ & $\num{2.0638e-05}$ & $\num{3.7229e-05}$ & $\num{1}$ & $\num{14258.0095}$ \\
& rGLASSO & NA & NA & NA & $\num{0.42459}$ & $\num{0.72255}$ & $\num{0.59772}$ & $\num{4.5624}$ & $\num{1.3789e-05}$ & $\num{2.7684e-05}$ & $\num{1}$ & $\num{14432.7653}$ \\
& GRAPHEM & $ \num{0.1268}$ & $ \num{0.92766}$ & $\ubold \num{0.72061}$ & NA & NA & NA & NA & $\num{5.0766e-05}$ & $\num{6.4547e-05}$ & $\num{3.8015e-3}$ & $\num{29034.7154}$ \\ 
\hline
\hline
\parbox[t]{2mm}{\multirow{5}{*}{\rotatebox[origin=c]{90}{$s_A = 10$}}} & DGLASSO & $ \num{0.089227}$ & $\num{0.98144}$ & $\ubold \num{0.93849}$ & $\ubold \num{0.083783}$ & $\num{0.76786}$ & $\ubold \num{0.68562}$& $\ubold \num{0.084488}$ & $\ubold \num{9.3856e-07}$ & $\ubold \num{1.4145e-06}$ & $\ubold \num{9.7475e-3}$ & $ \num{12290.7891}$\\
& MLEM & $\num{0.16731}$ & $\ubold \num{0.99994}$ & $\num{0.21978}$ & $\num{ 0.10536}$ & $\ubold \num{0.89959}$ & $\num{0.5}$ & $\num{0.10103}$ & $\num{2.3831e-06}$ & $\num{3.9179e-06}$ & $\num{2.5917e-2}$ & $\ubold \num{11819.7389}$\\
& GLASSO & NA & NA & NA & $\num{0.36481}$ & $\num{0.66667}$ & $\num{0.5}$ & $\num{0.46889}$ & $\num{3.0258e-05}$ & $\num{5.958e-05}$ & $\num{1}$ & $\num{13173.4621}$\\
& rGLASSO & NA & NA & NA & $\num{0.33857}$ & $\num{0.74942}$ & $\ubold \num{0.69861}$ & $\num{0.63211}$ & $\num{1.9702e-05}$ & $\num{4.3746e-05}$ & $\num{1}$ & $\num{13497.0609}$\\
& GRAPHEM & $\ubold \num{ 0.11748}$ & $\ubold \num{0.99986}$ & $ \num{ 0.74478}$ & NA & NA & NA & NA & $\num{8.9556e-05}$ & $\num{1.1244e-4}$ & $\num{1.2255e-2}$ & $\num{29034.7268}$\\ 
\hline
\hline
\parbox[t]{2mm}{\multirow{5}{*}{\rotatebox[origin=c]{90}{$s_A = 5$}}} & DGLASSO & $\ubold  \num{0.099734}$ & $\num{0.82279}$ & $\ubold \num{0.77369}$ & $\ubold \num{0.084239}$ & $ \num{0.76421}$ & $\ubold \num{0.68226}$& $\ubold \num{0.085378}$ & $ \ubold \num{9.2174e-07}$ & $ \ubold \num{1.2055e-06}$ & $\ubold \num{9.9813e-3}$ & $\ubold \num{12268.6671}$ \\
& MLEM & $\num{0.2338}$ & $\num{0.89305}$ & $\num{0.11628}$ & $\num{0.10581}$  & $\ubold \num{0.85602}$ & $\num{0.5}$ & $\num{0.10182}$ & $\num{2.2318e-06}$ & $\num{3.5375e-06}$ & $\num{5.6104e-2}$ & $\num{12312.7487}$\\
& GLASSO & NA & NA & NA & $\num{0.24241}$ & $\num{0.66667}$ & $\num{0.5}$ & $\num{0.30584}$ & $\num{1.4063e-05}$ & $\num{3.0952e-05}$ & $\num{1}$ & $\num{12892.1439}$\\
& rGLASSO & NA & NA & NA & $\num{0.26181}$ & $\num{0.70354}$ & $\num{0.56627}$ & $\num{0.45122}$ & $\num{1.1953e-05}$ & $\num{2.8842e-05}$ & $\num{1}$ & $\num{12834.7367}$\\
& GRAPHEM & $\num{0.12723}$ & $\ubold  \num{0.93053}$ & $ \num{0.57328}$ & NA & NA & NA & NA & $\num{1.0603e-4}$ & $\num{1.2245e-4}$ & $\num{1.477e-2}$ & $\num{29034.4563}$\\ 
\hline
\end{tabular}
}
}
\label{tab:resultssparseA}
\end{table}


\vskip 0.2in

\bibliography{references}

\end{document}

%% file: arrowleft.tikzstyles
\tikzstyle{dot node}=[fill=white, draw=black, shape=circle]

\tikzstyle{arrow left}=[<-]
\tikzstyle{arrow right}=[->]
\tikzstyle{both arrow}=[<->]

%% file: tikz_files/graph_A.tikz
\begin{tikzpicture}
	\begin{pgfonlayer}{nodelayer}
		\node [style=dot node] (0) at (-6, 1.5) {$1$};
		\node [style=dot node] (1) at (-2, 3) {$2$};
		\node [style=dot node] (2) at (-1, 0) {$3$};
		\node [style=dot node] (3) at (-6.5, -1) {$5$};
		\node [style=dot node] (4) at (-3, -2) {$4$};
		\node [style=none] (5) at (-4.25, 2.75) {\tiny $A(1,2)$};
		\node [style=none] (6) at (-0.75, 1.5) {\tiny $A(2,3)$};
		\node [style=none] (7) at (-3.2, 1.5) {\tiny $A(4,2)$};
		\node [style=none] (8) at (-4, 0) {\tiny $A(5,3)$};
		\node [style=none] (9) at (-4, -1) {\tiny $A(3,5)$};
		\node [style=none] (10) at (-6, 2.5) {\tiny $A(1,1)$};
	\end{pgfonlayer}
	\begin{pgfonlayer}{edgelayer}
		\draw [<-, line width = 1.4pt, draw = blue] (0) to (1);
		\draw [<-, line width = 0.4pt, draw = blue] (4) to (1);
		\draw[<-, bend left, line width = 1pt, draw = blue] (3) to (2);
		\draw[->, line width = 1.6pt, draw = blue] (3) to (2);
		\draw [<-, line width = 1.8pt, in=135, out=45, loop, draw = blue] (0) to ();
		\draw [<-, line width = 0.6pt, draw = blue] (1) to (2);
	\end{pgfonlayer}
\end{tikzpicture}

%% file: tikz_files/graph_Q.tikz
\begin{tikzpicture}
	\begin{pgfonlayer}{nodelayer}
		\node [style=dot node] (0) at (-6, 1.5) {$1$};
		\node [style=dot node] (1) at (-2, 3) {$2$};
		\node [style=dot node] (2) at (-1, 0) {$3$};
		\node [style=dot node] (3) at (-6.5, -1) {$5$};
		\node [style=dot node] (4) at (-3, -2) {$4$};
		\node [style=none] (5) at (-4, 2) {\tiny $P(2,5)$};
		\node [style=none] (6) at (-0.75, 1.5) {\tiny $P(2,3)$};
		\node [style=none] (7) at (-3.25, 1) {\tiny $P(1,3)$};
		\node [style=none] (9) at (-2, -1) {\tiny $P(2,4)$};
	\end{pgfonlayer}
	\begin{pgfonlayer}{edgelayer}
		\draw[line width = 2pt, draw = magenta] (3) to (1);
		\draw[line width = 0.4pt, draw = magenta] (0) to (2);
		\draw[line width = 1.2pt, draw = magenta] (1) to (2);
		\draw[line width = 0.8pt, draw = magenta] (4) to (1);
	\end{pgfonlayer}
\end{tikzpicture}

%% file: tikz_files/graph_both.tikz
\begin{tikzpicture}
	\begin{pgfonlayer}{nodelayer}
		\node [circle, inner sep=0pt,   text width=1cm,   align=center,   draw=black,   fill=white] (0) at (-5, 4) {\tiny $\x_{k-1}(1)$};
		\node [circle, inner sep=0pt,   text width=1cm,   align=center,   draw=black,   fill=white] (1) at (-5, 1) {\tiny $\x_{k-1}(2)$};
		\node [circle, inner sep=0pt,   text width=1cm,   align=center,   draw=black,   fill=white] (2) at (-5, -2) {\tiny $\x_{k-1}(3)$};
		\node [circle, inner sep=0pt,   text width=1cm,   align=center,   draw=black,   fill=white]  (3) at (-5, -5) {\tiny $\x_{k-1}(4)$};
		\node [circle, inner sep=0pt,   text width=1cm,   align=center,   draw=black,   fill=white]  (4) at (-5, -8) {\tiny $\x_{k-1}(5)$};
		\node [circle, inner sep=0pt,   text width=1cm,   align=center,   draw=black,   fill=white]  (5) at (0, 4) {\tiny $\x_k(1)$};
		\node [circle, inner sep=0pt,   text width=1cm,   align=center,   draw=black,   fill=white] (6) at (0, 1) {\tiny $\x_k(2)$};
		\node [circle, inner sep=0pt,   text width=1cm,   align=center,   draw=black,   fill=white] (7) at (0, -2) {\tiny $\x_k(3)$};
		\node [circle, inner sep=0pt,   text width=1cm,   align=center,   draw=black,   fill=white] (8) at (0, -5) {\tiny $\x_k(4)$};
		\node [circle, inner sep=0pt,   text width=1cm,   align=center,   draw=black,   fill=white] (9) at (0, -8) {\tiny $\x_k(5)$};
		\node [circle, inner sep=0pt,   text width=1cm,   align=center,   draw=black,   fill=white] (10) at (5, 4) {\tiny $\x_{k+1}(1)$};
		\node [circle, inner sep=0pt,   text width=1cm,   align=center,   draw=black,   fill=white] (11) at (5, 1) {\tiny $\x_{k+1}(2)$};
		\node [circle, inner sep=0pt,   text width=1cm,   align=center,   draw=black,   fill=white] (12) at (5, -2) {\tiny $\x_{k+1}(3)$};
		\node [circle, inner sep=0pt,   text width=1cm,   align=center,   draw=black,   fill=white] (13) at (5, -5) {\tiny $\x_{k+1}(4)$};
		\node [circle, inner sep=0pt,   text width=1cm,   align=center,   draw=black,   fill=white] (14) at (5, -8) {\tiny $\x_{k+1}(5)$};
		\node [style=none] (15) at (8, 4) {};
		\node [style=none] (16) at (8, 1) {};
		\node [style=none] (17) at (8, -2) {};
		\node [style=none] (18) at (8, -5) {};
		\node [style=none] (19) at (8, -8) {};
		\node [style=none] (20) at (-8, 4) {};
		\node [style=none] (21) at (-8, 1) {};
		\node [style=none] (22) at (-8, -2) {};
		\node [style=none] (23) at (-8, -5) {};
		\node [style=none] (24) at (-8, -8) {};
	\end{pgfonlayer}
	\begin{pgfonlayer}{edgelayer}
		\draw [->, draw = blue, line width = 1.8pt] (0) to (5);
		\draw [->, draw = blue, line width = 1.4pt] (1) to (5);
		\draw [->, draw = blue, line width = 0.4pt] (1) to (8);
		\draw [->, draw = blue, line width = 1.6pt] (4) to (7);
		\draw [->, draw = blue, line width = 1pt] (2) to (9);
		\draw [->, draw = blue, line width = 0.6pt] (2) to (6);
		\draw [->, draw = blue, line width = 1.4pt] (6) to (10);
		\draw [->, draw = blue, line width = 1.8pt] (5) to (10);
		\draw [->, draw = blue, line width = 0.4pt] (6) to (13);
		\draw [->, draw = blue, line width = 0.6pt] (7) to (11);
		\draw [->, draw = blue, line width = 1pt] (7) to (14);
		\draw [->, draw = blue, line width = 1.6pt] (9) to (12);
		\draw [bend left, draw = magenta,line width = 0.4pt] (0) to (2);
		\draw [bend right, draw = magenta, line width = 1.2pt] (1) to (2);
		\draw [bend right, draw = magenta, line width = 2pt] (1) to (4);
		\draw [bend left, draw = magenta, line width = 0.8pt] (2) to (3);
		\draw [bend left, draw = magenta,line width = 0.4pt] (5) to (7);
		\draw [bend right, draw = magenta, line width = 1.2pt] (6) to (7);
		\draw [bend right, draw = magenta, line width = 2pt] (6) to (9);
		\draw [bend left, draw = magenta, line width = 0.8pt] (6) to (8);
		\draw [bend left, draw = magenta, line width =  0.4pt] (10) to (12);
		\draw [bend left, draw = magenta, line width = 0.8pt] (11) to (13);
		\draw [bend right, draw = magenta, line width = 1.2pt] (11) to (12);
		\draw [bend right, draw = magenta, line width = 2pt] (11) to (14);
		\draw [dashed] (20.center) to (0);
		\draw [dashed] (21.center) to (1);
		\draw [dashed] (22.center) to (2);
		\draw [dashed] (23.center) to (3);
		\draw [dashed] (24.center) to (4);
		\draw [dashed] (10) to (15.center);
		\draw [dashed] (11) to (16.center);
		\draw [dashed] (12) to (17.center);
		\draw [dashed] (13) to (18.center);
		\draw [dashed] (14) to (19.center);
	\end{pgfonlayer}
\end{tikzpicture}